\documentclass[10pt,a4paper]{extarticle}
\usepackage[latin1]{inputenc}
\usepackage[left=1.5cm,right=1.5cm,top=2cm,bottom=2cm]{geometry}
\usepackage{url}
\usepackage{algpseudocode}
\usepackage{algorithmicx}
\usepackage{algorithm}
\usepackage{times,amsmath,amssymb,amsfonts,epsfig,graphicx}
\usepackage{amsthm}  
\usepackage{enumerate}
\usepackage{enumitem}
\usepackage{stackrel}
\usepackage{multirow}
\usepackage{subfig}
\usepackage{psfrag}
\usepackage{color}
\usepackage{comment}
\usepackage{bbm}
\usepackage[normalem]{ulem}
\usepackage{xcolor}
\usepackage{framed}
\usepackage{tikz}
\usepackage{booktabs}
\usepackage{wrapfig}

\newcommand{\savefootnote}[2]{\footnote{\label{#1}#2}}
\newcommand{\repeatfootnote}[1]{\textsuperscript{\ref{#1}}}

\newcommand{\hemant}[1]{\textcolor{blue}{#1}}
\newcommand*\hemantt{\color{blue}}
\renewcommand{\hemant}[1]{\textcolor{black}{#1}}
\renewcommand*\hemantt{\color{black}}
\newtheorem{theorem}{Theorem}

\newtheorem{assumption}{Assumption}

\newtheorem{corollary}{Corollary}

\newtheorem{definition}{Definition}

\newtheorem{lemma}{Lemma}

\newtheorem{proposition}{Proposition}
\newtheorem{remark}{Remark}

\numberwithin{equation}{section}

\newcommand{\calB}{\ensuremath{\mathcal{B}}}
\newcommand{\calC}{\ensuremath{\mathcal{C}}}

\newcommand{\calH}{\ensuremath{\mathcal{H}}}

\newcommand{\calI}{\ensuremath{\mathcal{I}}}
\newcommand{\calP}{\ensuremath{\mathcal{P}}}

\newcommand{\calR}{\ensuremath{\mathcal{R}}}
\newcommand{\calS}{\ensuremath{\mathcal{S}}}
\newcommand{\calM}{\ensuremath{\mathcal{M}}}
\newcommand{\calN}{\ensuremath{\mathcal{N}}}
\newcommand{\calT}{\ensuremath{\mathcal{T}}}

\newcommand{\calV}{\ensuremath{\mathcal{V}}}

\newcommand{\calE}{\ensuremath{\mathcal{E}}}
\newcommand{\calVp}{\ensuremath{\calV^{\prime}}}
\newcommand{\calVpp}{\ensuremath{\calV^{\prime\prime}}}

\newcommand{\norm}[1]{\parallel{#1}\parallel}
\newcommand{\abs}[1]{|{#1}|}
\newcommand{\ceil}[1]{\lceil{#1}\rceil}
\newcommand{\floor}[1]{\lfloor{#1}\rfloor}
\newcommand{\set}[1]{\left\{{#1}\right\}}
\newcommand{\dotprod}[2]{\langle#1,#2\rangle}
\newcommand{\est}[1]{\widehat{#1}}
\newcommand{\expec}{\ensuremath{\mathbb{E}}}
\newcommand{\grad}{\ensuremath{\nabla}}
\newcommand{\hess}{\ensuremath{\nabla^2}}
\newcommand{\matR}{\ensuremath{\mathbb{R}}}
\newcommand{\matZ}{\ensuremath{\mathbb{Z}}}

\newcommand{\argmin}[1]{\underset{#1}{\operatorname{argmin}}}

\newcommand{\prob}{\ensuremath{\mathbb{P}}}
\newcommand{\Linfnorm}{L_{\infty}}
\newcommand{\vp}{\ensuremath{v^{\prime}}}
\newcommand{\vpp}{\ensuremath{v^{\prime\prime}}}
\newcommand{\tp}{\ensuremath{t^{\prime}}}
\newcommand{\lp}{\ensuremath{l^{\prime}}}
\newcommand{\qp}{\ensuremath{q^{\prime}}}

\newcommand{\vecx}{\mathbf x}
\newcommand{\vecn}{\mathbf n}

\newcommand{\vecv}{\mathbf v}
\newcommand{\vecvp}{\mathbf v^{\prime}}
\newcommand{\vecvpp}{\mathbf v^{\prime\prime}}
\newcommand{\vecw}{\mathbf w}

\newcommand{\vece}{\mathbf e}

\newcommand{\vecy}{\mathbf y}

\newcommand{\vecz}{\mathbf z}

\newcommand{\matA}{\ensuremath{\mathbf{A}}}
\newcommand{\matB}{\ensuremath{\mathbf{B}}}

\newcommand{\matV}{\ensuremath{\mathbf{V}}}
\newcommand{\matH}{\ensuremath{\mathbf{H}}}
\newcommand{\matX}{\ensuremath{\mathbf{X}}}

\newcommand{\matVpp}{\ensuremath{\mathbf{V^{\prime\prime}}}}

\newcommand{\phitil}{\ensuremath{\tilde{\phi}}}

\newcommand{\numcen}{\ensuremath{m_{x}}}
\newcommand{\numdirec}{\ensuremath{m_{v}}} 
\newcommand{\numdirecp}{\ensuremath{m_{v^{\prime}}}} 
\newcommand{\numdirecpp}{\ensuremath{m_{v^{\prime\prime}}}} 
\newcommand{\numcenpair}{\ensuremath{m^{\prime}_{x}}}
\newcommand{\sparsparam}{\ensuremath{K}} 
\newcommand{\totsparsity}{\ensuremath{k}} 
\newcommand{\univsparsity}{\ensuremath{k_1}} 
\newcommand{\bivsparsity}{\ensuremath{k_2}} 
\newcommand{\univsupp}{\ensuremath{\calS_1}} 
\newcommand{\bivsupp}{\ensuremath{\calS_2}} 
\newcommand{\bivsuppvar}{\ensuremath{\calS_2^{\text{var}}}} 
\newcommand{\totsupp}{\ensuremath{\calS}} 
\newcommand{\totsuppest}{\ensuremath{\widehat{\calS}}} 
\newcommand{\binspart}{\ensuremath{\calP}}
\newcommand{\thirdtayrem}{\ensuremath{R}}
\newcommand{\rootset}{\ensuremath{\calR}}
\newcommand{\ripconst}{\ensuremath{\kappa}}
\newcommand{\spacespl}{\ensuremath{{\calH}}}

\newcommand{\exnoisemag}{\ensuremath{\varepsilon}}
\newcommand{\exnoisemagp}{\ensuremath{\varepsilon^{\prime}}}
\newcommand{\exnoise}{\ensuremath{z}} 
\newcommand{\exnoisevec}{\ensuremath{\mathbf{\exnoise}}}   
\newcommand{\exnoisep}{\ensuremath{z^{\prime}}}  
\newcommand{\dimn}{\ensuremath{d}} 
\newcommand{\degree}{\ensuremath{\rho}} 
\newcommand{\maxdegree}{\ensuremath{\rho_{m}}}
\newcommand{\numdegree}{\ensuremath{\alpha}} 
\newcommand{\smconst}{\ensuremath{B}}
\newcommand{\idenconst}{\ensuremath{D}}
\newcommand{\critintmeas}{\ensuremath{\lambda}}
\newcommand{\gradstep}{\ensuremath{\mu}}
\newcommand{\gradstepp}{\ensuremath{\mu^{\prime}}}
\newcommand{\hessstep}{\ensuremath{\mu_1}}
\newcommand{\taynoisvec}{\ensuremath{\vecn}}
\newcommand{\taynoissca}{\ensuremath{n}}
\newcommand{\baseset}{\ensuremath{\chi}}
\newcommand{\twohashfam}{\ensuremath{\calH_{2}^{d}}}
\newcommand{\thashfam}{\ensuremath{\calH_{t}^{d}}}
\newcommand{\khashfam}{\ensuremath{\calH_{k}^{d}}}
\newcommand{\hashfn}{\ensuremath{h}}
\newcommand{\linoptmat}{\ensuremath{\calM}}
\newcommand{\linoptmatB}{\ensuremath{\calB}}
\newcommand{\canvec}{\ensuremath{\mathbf{e}}}
\newcommand{\derivsamperr}{\ensuremath{\tau}}
\newcommand{\hesssamperr}{\ensuremath{\tau^{\prime}}}
\newcommand{\derivsamperrpp}{\ensuremath{\tau^{\prime\prime}}}
\newcommand{\pardevestnois}{\ensuremath{\eta}}
\newcommand{\hessestnoisebd}{\ensuremath{\eta}}
\newcommand{\pardevstep}{\ensuremath{\beta}}
\newcommand{\lpair}{\ensuremath{(l,l^{\prime})}} 
\newcommand{\lpairi}{\ensuremath{(l^{\prime},l)}} 
\newcommand{\qpair}{\ensuremath{(q,q^{\prime})}} 
\newcommand{\qpairi}{\ensuremath{(q^{\prime},q)}}

\newcommand{\xlpair}{\ensuremath{(x_{l},x_{l^{\prime}})}} 
\newcommand{\xqpair}{\ensuremath{(x_{q},x_{q^{\prime}})}} 
\newcommand{\xqpairi}{\ensuremath{(x_{q^{\prime}},x_{q})}} 
 
\newcommand{\hessestnoisa}{\ensuremath{\mathbf{\eta_{q,1}}}}
\newcommand{\hessestnoisb}{\ensuremath{\mathbf{\eta_{q,2}}}}

\newcommand{\gradtayrem}{\ensuremath{\begin{pmatrix}
  \vecv^{T} \hess \partial_1 g(\zeta_1) \vecv  \\ \vdots \\ \vecv^{T} \hess \partial_{\totsparsity} g(\zeta_{\totsparsity}) \vecv
 \end{pmatrix}}}

\newcommand{\gradtayremf}{\ensuremath{\begin{pmatrix}
  {\vecvp}^{T} \hess \partial_1 f(\zeta_1) \vecvp  \\ \vdots \\ {\vecvp}^{T} \hess \partial_{\dimn} f(\zeta_{\dimn}) \vecvp
 \end{pmatrix}}}

\newcommand{\hessrowips}{\ensuremath{\begin{pmatrix}
  \dotprod{\grad \partial_1 f(\vecx)}{\vecvp}  \\ \vdots \\ \dotprod{\grad \partial_{\dimn} f(\vecx)}{\vecvp}
 \end{pmatrix}}}

\newcommand{\hessrowmeas}{\ensuremath{\frac{1}{\hessstep}\begin{pmatrix}
  (\est{\grad} f(\vecx + \hessstep\vecvp_1) - \est{\grad} f(\vecx))_q  \\ \vdots \\ (\est{\grad} f(\vecx + \hessstep\vecvp_{\numdirecp}) - \est{\grad} f(\vecx))_q}
 \end{pmatrix}}

\newcommand{\gradtayremfq}{\ensuremath{\begin{pmatrix}
  {\vecvp_1}^{T} \hess \partial_q f(\zeta_{1}) \vecvp_1  \\ \vdots \\ {\vecvp_{\numdirecp}}^{T} \hess \partial_q f(\zeta_{\numdirecp}) \vecvp_{\numdirecp}
 \end{pmatrix}}}

\newcommand{\gradestnoisq}{\ensuremath{\frac{1}{\hessstep}\begin{pmatrix}
  w_q(\vecx + \hessstep\vecvp_1) - w_q(\vecx)   \\ \vdots \\ w_q(\vecx + \hessstep\vecvp_{\numdirecp}) - w_q(\vecx)
 \end{pmatrix}}}

\title{Algorithms for Learning Sparse Additive Models with Interactions in High Dimensions\thanks{A preliminary version of this paper appeared in the proceedings of the $19^{th}$ International Conference
on Artificial Intelligence and Statistics (AISTATS) 2016 \cite{Tyagi_aistats16}. The present draft is an expanded version containing additional results.}}
\author{Hemant Tyagi\thanks{School of Mathematics, University of Edinburgh,
Edinburgh, United Kingdom; The Alan Turing Institute, London, United Kingdom}\\ \texttt{htyagi@turing.ac.uk} \and 
Anastasios Kyrillidis\thanks{Department of Electrical and Computer Engineering, The University of Texas at Austin} \\ \texttt{anastasios@utexas.edu} \and
Bernd G\"{a}rtner\thanks{Department of Computer Science, Institute of Theoretical Computer Science, ETH Z\"urich, CH-8092 Z\"urich} \\ \texttt{gaertner@inf.ethz.ch} \and
Andreas Krause\thanks{Department of Computer Science, ETH Z\"urich, CH-8092 Z\"urich} \\ \texttt{krausea@ethz.ch}} 

\begin{document}
\maketitle

\begin{abstract}
A function $f: \matR^d \rightarrow \matR$ is a Sparse Additive Model (SPAM), if it is of the
form $f(\vecx) = \sum_{l \in \calS}\phi_{l}(x_l)$ where $\calS \subset [\dimn]$, $\abs{\calS} \ll \dimn$.
Assuming $\phi$'s, $\calS$ to be unknown, there exists extensive work for estimating $f$ from its samples. 
In this work, we consider a generalized version of SPAMs, that also allows for the presence of a sparse number of \emph{second order interaction terms}.
For some $\univsupp \subset [\dimn], \bivsupp \subset {[d] \choose 2}$, with
$\abs{\univsupp} \ll \dimn, \abs{\bivsupp} \ll d^2$, the function $f$ is now assumed to be of the form: 
$\sum_{p \in \univsupp}\phi_{p} (x_p) + \sum_{\lpair \in \bivsupp}\phi_{\lpair} \xlpair$. 
Assuming we have the freedom to query $f$ anywhere in its domain, we derive efficient algorithms 
that provably recover $\univsupp,\bivsupp$ with \emph{finite sample bounds}. 
Our analysis covers the noiseless setting where exact 
samples of $f$ are obtained, and also extends to the noisy setting where the queries are corrupted with noise. 
For the noisy setting in particular, we consider two noise models namely: i.i.d Gaussian noise and arbitrary but bounded noise. 
Our main methods for identification of $\bivsupp$ essentially rely on estimation of sparse Hessian matrices, for which we provide two novel 
compressed sensing based schemes. 
Once $\univsupp, \bivsupp$ are known, we show how the individual components $\phi_p$, $\phi_{\lpair}$ can be estimated via additional queries of $f$, 
with uniform error bounds. Lastly, we provide simulation results on synthetic data that validate our theoretical findings.
\end{abstract}

\section{Introduction} \label{sec:intro}
Many scientific problems involve estimating an unknown function $f$, defined over a
compact subset of $\matR^{\dimn}$, with $\dimn$ large. Such problems arise for instance, 
in modeling complex physical processes \cite{Muller2008,Maathuis09,Wainwright09a}. Information 
about $f$ is typically available in the form of point values $(x_i,f(x_i))_{i=1}^{n}$, which 
are then used for learning $f$. It is well known that the problem suffers from the curse of dimensionality, if only smoothness 
assumptions are placed on $f$. For example, if $f$ is $C^s$ smooth ($s$ times continuously differentiable), then for uniformly approximating 
$f$ within error $\delta \in (0,1)$, one needs $n = \Omega(\delta^{-\dimn/s})$ samples \cite{Traub1988}. 

A popular line of work in recent times, considers the setting where $f$ possesses an intrinsic 
low dimensional structure, \textit{i.e.}, depends on only a small subset of $\dimn$ variables. 
There exist algorithms for estimating such $f$ -- tailored to the underlying structural assumption --  
along with attractive theoretical guarantees, that do not suffer from the curse of dimensionality 
(cf., \cite{Devore2011,Cohen2010,Tyagi2012_nips,Fornasier2012}). One such assumption leads to the 
class of sparse additive models (SPAMs) wherein $f = \sum_{l \in \totsupp} \phi_l$ for some unknown 
$S \subset \set{1,\dots,\dimn}$ with $\abs{\totsupp} = \totsparsity \ll \dimn$. There exist several 
algorithms for learning these models (cf. \cite{Ravi2009,Meier2009,Huang2010,Raskutti2012,Tyagi14_nips}).   
Here we focus on a generalized SPAM model, where $f$ can also contain a small number of \emph{second order 
interaction terms}, \textit{i.e.}, 
\begin{equation} \label{eq:intro_gspam_form}
f(x_1,x_2,\dots,x_d) = \sum_{p \in \univsupp}\phi_{p} (x_p) + \sum_{\lpair \in \bivsupp}\phi_{\lpair} \xlpair; \quad 
\univsupp \subset [\dimn], \bivsupp \subset {[d] \choose 2},  
\end{equation}
with $\abs{\univsupp} \ll \dimn,\abs{\bivsupp} \ll \dimn^2$. Here, $\phi_{\lpair} \xlpair \not\equiv g_l(x_l) + h_{\lp}(x_{\lp})$ for 
some univariates $g_l, h_{\lp}$ meaning that $\frac{\partial^2}{\partial_l \partial_{\lp}} \phi_{\lpair} \not\equiv 0$. 
As opposed to SPAMs, the problem is significantly harder now -- allowing interactions leads to an additional $\dimn(\dimn-1)/2$ 
unknowns out of which of only a few terms (\textit{i.e.}, those in $\bivsupp$) are relevant. In the sequel, we will denote $\totsupp$ to be the 
support of $f$ consisting of variables that are part of $\univsupp$ or $\bivsupp$, and $k$ to be the size of $\totsupp$. 
Moreover, we will denote $\maxdegree$ to be the maximum number of occurrences of a variable in $\bivsupp$ -- this parameter captures the 
underlying \emph{complexity} of the interactions.

There exist relatively few results for learning models of the form \eqref{eq:intro_gspam_form}, 
with the existing work being mostly in the \emph{regression framework} in statistics (cf., \cite{Lin2006, Rad2010, Storlie2011}). 
Here, $(x_i,f(x_i))_{i=1}^{n}$ are typically samples from an unknown probability measure $\prob$, with the samples moreover assumed 
to be corrupted with (i.i.d) stochastic noise. In this paper, we consider the \emph{approximation theoretic} setting where we have the 
freedom to query $f$ at any desired set of points (cf. \cite{Devore2011,Fornasier2012,Tyagi14_nips}). 
We propose \hemant{strategies for querying $f$, along with efficient recovery algorithms}, which leads to much stronger guarantees than 
known in the regression setting. In particular, we provide the first \emph{finite sample bounds} for exactly recovering $\univsupp$ and $\bivsupp$. 
This is shown for (i) the noiseless setting where exact samples are observed, as well as (ii) the noisy setting, where 
the samples are corrupted with noise (either i.i.d Gaussian or arbitrary but bounded noise models). 

Once $\univsupp$, $\bivsupp$ are identified, we show in Section \ref{sec:est_comp} how the individual components: $\phi_p, \phi_{\lpair}$ of the model 
can be estimated, with \emph{uniform} error bounds. This is shown for both the noiseless and noisy query settings. It is accomplished by \hemant{additionally sampling} 
$f$ along the identified one/two dimensional subspaces corresponding to $\univsupp,\bivsupp$ respectively, and by employing standard estimators from approximation theory and 
statistics. 

\subsection{Our contributions} 
We make the following contributions for learning models of the form \eqref{eq:intro_gspam_form}.  
\begin{enumerate}
\item \label{cont:alg_gen_1_res} Firstly, we provide an efficient 
algorithm, namely Algorithm \ref{algo:gen_overlap}, 
which provably recovers $\univsupp, \bivsupp$ exactly with high probability\footnote{With probability $1 - O(d^{-c})$ for some constant $c > 0$.} (w.h.p), 
with $O(\totsparsity \maxdegree (\log \dimn)^3)$ noiseless queries. 
When the point queries are corrupted with (i.i.d) Gaussian noise, we 
show that Algorithm \ref{algo:gen_overlap} identifies $\univsupp$, $\bivsupp$ w.h.p, with $O(\maxdegree^5 \totsparsity^2 (\log \dimn)^4)$ 
noisy queries of $f$. We also analyze the setting of arbitrary but bounded noise, and derive sufficient conditions 
on the noise magnitude that enable recovery of $\univsupp, \bivsupp$. 

\item \label{cont:alg_gen_2_res} Secondly, we provide another efficient algorithm namely Algorithm \ref{algo:gen_overlap_alt}, 
which provably recovers $\univsupp, \bivsupp$ exactly w.h.p, with (i) $O(\totsparsity \maxdegree (\log \dimn)^2)$ noiseless queries 
and, (ii) $O(\maxdegree^5 \totsparsity^5 (\log \dimn)^3)$ noisy queries (i.i.d Gaussian noise). 
We also analyze the setting of arbitrary but bounded noise.

\item \label{cont:cont_cases} We provide an algorithm tailored to the special case where the underlying interaction graph corresponding to $\bivsupp$ is
known to be a \emph{perfect matching}, \textit{i.e.}, each variable interacts with at most one variable (so $\maxdegree = 1$). 
We show that the algorithm identifies $\univsupp,\bivsupp$ w.h.p, with (i) $O(\totsparsity (\log d)^2)$ 
noiseless queries and, (ii) $O(\totsparsity^2 (\log d)^3)$ noisy queries (i.i.d Gaussian noise). 
We also analyze the setting of arbitrary but bounded noise. 

\item \label{cont:sparse_hessians} An important part of Algorithms \ref{algo:gen_overlap}, \ref{algo:gen_overlap_alt} are two novel 
compressive sensing based methods, for estimating \emph{sparse}, $d \times d$ Hessian matrices. These might be of independent interest.
\end{enumerate}
We also provide simulation results on synthetic data, that validate our theoretical findings concerning the identification of $\univsupp,\bivsupp$.
\hemant{Algorithm \ref{algo:gen_overlap} appeared in AISTATS $2016$ \cite{Tyagi_aistats16}, in a preliminary version of this paper. The results in Section \ref{sec:est_comp} (estimating individual components of $f$) were part of the supplementary material in \cite{Tyagi_aistats16}.}
\subsection{Related work} 
We now provide a brief overview of related work, followed by an outline of our main contributions and an overview of the 
methods. A more detailed comparison with related work is provided in Section \ref{sec:discuss}.

\paragraph{Learning SPAMs.} 
This model was introduced in the nonparametric regression setting 
by Lin et al. \cite{Lin2006} who proposed the COSSO (Component selection and smoothing) method -- an extension of the lasso 
to the reproducing kernel Hilbert space (RKHS) setting. It essentially performs least squares minimization with a sparsity inducing penalty term 
involving the sum of norms of the function components. In fact, this method is designed to handle the more general 
smoothing spline analysis of variance (SS-ANOVA) model \cite{wabha03,Gu02}. It has since been studied extensively in the regression framework with 
a multitude of results involving: estimation of $f$ (cf.,\cite{Koltch08,Meier2009,Ravi2009,Raskutti2012,Koltch2010,Huang2010}) 
and/or variable selection, \textit{i.e.}, identifying the support $\totsupp$ (cf., \cite{Huang2010,Ravi2009,Wahl15}). 

A common theme behind (nearly all of) these approaches is to first (approximately) represent each $\phi_j$; $1 \leq j \leq d$, 
in a suitable basis of finite size. 
This is done for example via: B-splines (cf. \cite{Huang2010,Meier2009}), finite combination 
of kernel functions (cf. \cite{Raskutti2012,Koltch2010}) etc. Thereafter, the 
problem reduces to a finite dimensional one, that involves finding the values of the coefficients in the corresponding basis representation. 
This is accomplished by performing least squares minimization subject to sparsity and smoothness inducing penalty terms -- the optimization problem is convex 
on account of the choice of the penalty terms, and hence can be solved efficiently. 

With regards to the problem of estimating $f$, Koltchinskii et al. \cite{Koltch2010}, Raskutti et al. \cite{Raskutti2012} 
proposed a convex program for estimating $f$ in the RKHS setting along with $L_2$ error rates. 
These error rates were shown to be minimax optimal by Raskutti et al. \cite{Raskutti2012}. 
\hemant{For example, $f$ lying in a Sobolev space with smoothness parameter $\alpha > 1/2$, 
are estimated at the optimal $L_2$ rate: $\frac{\totsparsity \log d}{n} + \totsparsity n^{-\frac{2\alpha}{2\alpha+1}}$ 
where $n$ denotes the number of samples.}
There also exist results for the \emph{variable selection} problem, \textit{i.e.}, for estimating the support $\totsupp$. 
In contrast to the setting of sparse linear models, for which non-asymptotic sample complexity bounds are known \cite{Wainwright09b,Wainwright09a}, 
the corresponding results in the nonparametric setting are usually \emph{asymptotic}, \textit{i.e.}, derived in the limit of large $n$. 
This property is referred to as \emph{sparsistency} in the statistics literature; an estimator is called \emph{sparsistent} if $\est{\totsupp} = \totsupp$ 
with probability approaching one as $n \rightarrow \infty$. Variable selection results for SPAMs in the nonparametric regression setting can be 
found for instance in \cite{Ravi2009,Huang2010,Wahl15}. Recently, Tyagi et al. \cite{Tyagi14_nips} considered this problem in the approximation theoretic setting; 
they proposed a method that identifies $\totsupp$ w.h.p with sample complexities $O(\totsparsity \log d)$, $O(\totsparsity^3 (\log \dimn)^2)$ 
in the absence/presence of Gaussian noise, respectively. 

While there exists a significant amount of work in the literature for SPAMs, the aforementioned methods are designed for specifically learning SPAMs, and cannot handle generalized SPAMs  
of the form \eqref{eq:intro_gspam_form} containing interaction terms.

\paragraph{Learning generalized SPAMs.} There exist fewer results for generalized SPAMs of the form \eqref{eq:intro_gspam_form}, 
in the regression setting. The COSSO algorithm \cite{Lin2006} can handle \eqref{eq:intro_gspam_form}, 
however its convergence rates are shown only for the case of no interactions. Radchenko et al. \cite{Rad2010} proposed the VANISH 
algorithm -- a least squares method with sparsity constraints and show that their method is sparsistent.
Storlie et al. \cite{Storlie2011} proposed ACOSSO -- an adaptive version of the COSSO algorithm -- which can also handle \eqref{eq:intro_gspam_form}. 
They derived convergence rates and sparsistency results for their method, albeit for the case of no interactions.
Recently, Dalalayan et al. \cite{Dala2014}, Yang et al. \cite{Yang2015} studied a generalization of \eqref{eq:intro_gspam_form} that allows for the presence of 
a sparse number of $m$-wise interaction terms for some additional sparsity parameter $m$. While they derive non-asymptotic 
$L_2$ error rates for estimating $f$ in such generic setting, they do not guarantee unique identification of 
the interaction terms for any value of $m$. 

A special case of \eqref{eq:intro_gspam_form} -- where $\phi_p$'s are linear and each $\phi_{\lpair}$ is of the form $x_l x_{\lp}$ --
has been studied considerably. Within this setting, there exist algorithms that recover $\univsupp,\bivsupp$, along with 
convergence rates for estimating $f$ in the limit of large $n$ \cite{Choi2010,Rad2010,Bien2013}. 
\hemant{There also exist non-asymptotic sampling bounds for identifying the interaction terms in the noiseless setting (cf., \cite{Nazer2010,Kekatos11}).} 
However finite sample bounds for the non-linear model \eqref{eq:intro_gspam_form} are not known in general.

\paragraph{Other low-dimensional function models.} There exist results for other, more general classes 
of intrinsically low dimensional functions, that we now mention starting with the approximation theoretic setting. 
Devore et al. \cite{Devore2011} consider functions depending on an unknown subset $\totsupp$ 
of the variables with $\abs{\totsupp} = \totsparsity \ll d$. The functions do not necessarily possess an additive structure, 
so the function class is more general than \eqref{eq:intro_gspam_form}. 
They provide algorithms that recover $\totsupp$ exactly w.h.p, 
with $O(c^k \totsparsity \log d)$ noiseless queries of $f$, for some constant $c > 0$. 
Schnass et al. \cite{karin2011} derived a simpler algorithm for this problem in the noiseless setting. This function class was also studied 
by Comminges et al. \cite{Comming2011, Comming2012} in the nonparametric regression setting wherein they analyzed an estimator that 
identifies $\totsupp$ w.h.p, with $O(c^k \totsparsity \log d)$ samples of $f$.
Fornasier et al. \cite{Fornasier2012}, Tyagi et al. \cite{Tyagi2012_nips} considered a generalization of the above function class where $f$ is now of the form 
$f(\vecx) = g(\matA\vecx)$, for unknown $\matA \in \matR^{k \times d}$. They derived algorithms that approximately 
recover the row-span of $\matA$, with sample complexities typically polynomial in $\totsparsity, \dimn$. 
While the above methods could possibly recover the underlying support $\totsupp$ for the SPAM model \eqref{eq:intro_gspam_form}, 
their sample complexities are either exponential in $\totsparsity$ \cite{Devore2011,Comming2011,Comming2012} or 
polynomial in $\dimn$ \cite{Fornasier2012,Tyagi2012_nips}. As explained in Section \ref{sec:discuss}, the algorithm of Schnass et al. \cite{karin2011} 
would recover $\totsupp$ w.h.p, with $O(\maxdegree^4 \totsparsity (\log \dimn)^2)$ noiseless queries, with potentially large constants (depending on smoothness of $f$) 
within the $O(\cdot)$ term. Moreover, we note that the aforementioned methods are not designed for identifying \emph{interactions} among the variables. 

\subsection{Overview of methods used} 
We now describe the main underlying ideas behind the algorithms described in this paper, for identifying $\univsupp,\bivsupp$. 
On a top level, our methods are based on two simple observations for the model \eqref{eq:intro_gspam_form}, namely that for any $\vecx \in \matR^d$: 
\begin{itemize}
\item The gradient $\grad f(\vecx) \in \matR^d$ is $\totsparsity$ sparse.
\item The Hessian $\hess f(\vecx) \in \matR^{d \times d}$ is at most $\totsparsity (\maxdegree + 1)$ sparse. 
In particular, it has $\totsparsity$ non zero rows, with each such row having at most $\maxdegree + 1$ non zero entries.
\end{itemize}
For the special case of \emph{no overlap}, \textit{i.e.}, $\maxdegree = 1$, we proceed in two phases. 
In the first phase -- outlined as Algorithm \ref{algo:est_act} -- we identify all variables 
in $\totsupp$ by estimating $\grad f(\vecx)$ via $\ell_1$ minimization\footnote{We note that the idea of estimating a sparse 
gradient via $\ell_1$ minimization is motivated from Fornasier et al. \cite{Fornasier2012}; their algorithm however is for a more general function class than ours.}, 
for each $\vecx$ lying within a carefully constructed finite set $\baseset \in \matR^d$. The set $\baseset$ in particular is 
constructed\footnote{see Definition \ref{def:thash_fam} and ensuing discussion.} 
so that it provides a uniform discretization of all possible two dimensional canonical subspaces in $\matR^d$. 
In the second phase -- outlined as Algorithm \ref{algo:est_ind_sets} -- we identify the sets $\univsupp,\bivsupp$ via a 
simple (deterministic) binary search based procedure, over the rows of the corresponding $\totsparsity \times \totsparsity$ 
sub-matrix of the Hessian of $f$.

For the general case however where $\maxdegree \geq 1$, the above scheme does not guarantee identification of $\totsupp$; 
see discussion at beginning of Section \ref{subsec:noiseless_overlap_set_est}. Therefore now, we consider a different ``two phase'' 
approach where in the first phase, we query $f$ with the goal of identifying the set of interactions $\bivsupp$. 
This in fact entails estimating the sparse Hessian $\hess f(\vecx)$, at each $\vecx$ lying within $\baseset$. We propose two 
different methods for estimating $\hess f(\vecx)$, utilizing tools from compressive sensing (CS).
\begin{itemize}
\item The first method is a part of Algorithm \ref{algo:gen_overlap} where we estimate each row of 
$\hess f (\vecx)$ separately, via a ``difference of gradients'' approach. This is motivated by the following identity, 
based on the Taylor expansion of $\grad f$ at $\vecx$, for suitable $\vecvp \in \matR^d$, $\hessstep > 0$:
\begin{equation} \label{eq:intro_diff_grad_est}
 \frac{\grad f(\vecx + \hessstep\vecvp) - \grad f(\vecx)}{\hessstep} = \hess f(\vecx) \vecvp + O(\hessstep).
\end{equation}
We can see from \eqref{eq:intro_diff_grad_est}, that a difference of gradient vectors corresponds to obtaining a perturbed linear measurement of 
\emph{each} $\maxdegree + 1$ sparse row of $\hess f(\vecx)$. CS theory tells us that by collecting $O(\maxdegree \log d)$ such 
``gradient differences'' -- each difference term corresponding to a random choice of $\vecvp$ from a suitable distribution -- we can estimate each row of $\hess f (\vecx)$ 
via $\ell_1$ minimization. Since $\grad f$ is $\totsparsity$ sparse, it can also be estimated via $O(\totsparsity \log d)$ queries of $f$ -- this 
leads to obtaining an estimate of $\hess f (\vecx)$ with $O(\totsparsity \maxdegree (\log d)^2)$ queries of $f$ in total.

\item The second method is a part of Algorithm \ref{algo:gen_overlap_alt} where we estimate all entries of $\hess f(\vecx)$ in ``one go''. 
This is motivated by the following identity, based on the Taylor expansion of $f$ at $\vecx$, for suitable $\vecv \in \matR^d$, $\gradstep > 0$:
\begin{equation} \label{eq:intro_alt_hess_est}
\frac{f(\vecx + 2\gradstep\vecv) + f(\vecx - 2\gradstep\vecv) - 2f(\vecx)}{4\gradstep^2} = \dotprod{\vecv \vecv^T}{\hess f(\vecx)} + O(\gradstep).
\end{equation}
We see from \eqref{eq:intro_alt_hess_est} that the L.H.S corresponds to a perturbed linear measurement of the Hessian, 
with a rank one matrix. By leveraging recent results in CS -- most notably the work of Chen et al. \cite{Chen15} -- we recover an estimate of $\hess f(\vecx)$ through 
$\ell_1$ minimization, by choosing $\vecv$'s randomly from a suitable distribution. As described in detail in Section \ref{sec:algo_gen_overlap_alt}, 
this requires us to make $O(\totsparsity \maxdegree \log d)$ queries of $f$.
\end{itemize}
Once $\bivsupp$ is estimated, we estimate $\univsupp$ by invoking (a slightly improved version of) the method of Tyagi et al. \cite{Tyagi14_nips} 
for learning SPAMs, on the reduced variables set. 

\paragraph{Outline of the paper.} The rest of the paper is organized as follows. Section \ref{sec:problem} contains a formal description of the 
problem along with notation used. We begin by analyzing the special case of \emph{no overlap} between the elements of $\bivsupp$ (\textit{i.e.}, $\maxdegree = 1$), 
in Section \ref{sec:algo_nonoverlap}. Section \ref{sec:algo_gen_overlap} then considers the general setting where $\maxdegree \geq 1$. In particular, it describes 
Algorithm \ref{algo:gen_overlap} wherein the underlying sparse Hessian of $f$ is estimated via a difference of sparse gradients mechanism. Section \ref{sec:algo_gen_overlap_alt} 
also handles the general overlap setting, albeit with a different method for estimating the sparse Hessian of $f$. Once $\univsupp,\bivsupp$ are estimated, we describe how the individual 
components of $f$ can be estimated via standard tools from approximation theory and statistics, in Section \ref{sec:est_comp}. Section \ref{sec:sims} contains simulation results on synthetic examples. 
We provide a detailed discussion of related work in Section \ref{sec:discuss}, and conclude with directions for future work in Section \ref{sec:concl_rems}. All proofs are deferred to the appendix.
\section{Notation and problem setup} \label{sec:problem}
\paragraph{Notation.} Scalars are mostly denoted by plain letters (e.g. $\univsparsity$, $\bivsparsity$, $\dimn$), 
vectors by lowercase boldface letters (e.g., ${\vecx}$) or by lowercase Greek letters (\emph{e.g.}, $\zeta$), matrices by uppercase boldface
letters (e.g. ${\matA}$) and sets by uppercase calligraphic letters (e.g.
$\calS$), with the exception of $[\dimn]$ which denotes the index set $\set{1,
\ldots, \dimn}$.
Given a set $\mathcal{S} \subseteq [\dimn]$, we denote its complement by
$\mathcal{S}^c := [\dimn] \setminus \mathcal{S}$ and for vector $\vecx = (x_1,\dots,x_{\dimn}) \in \matR^{\dimn}$,
$(\vecx)_{\calS}$ denotes the restriction of $\vecx$ onto $\calS$, \textit{i.e.},
$((\vecx)_{\calS})_l = x_l$ if $l \in \calS$ and $0$ otherwise.
We use $\abs{\calS}$ to denote the cardinality of a set $\calS$.
The $\ell_p$ norm of a vector $\vecx \in \matR^{\dimn}$ is defined as $\norm{\vecx}_p
:= \left ( \sum_{l=1}^\dimn \abs{x_i}^p \right )^{1/p}$. 
Let $g$ be a function of $n$ variables, $g(x_1,\dots,x_n)$. 
$\expec_p[g]$, $\expec_{\lpair}[g]$ denote expectation w.r.t uniform distributions over 
$x_p$ and $(x_l, x_{l^{\prime}})$ respectively. $\expec[g]$ denotes expectation w.r.t. uniform distribution 
over $(x_1,\dots,x_n)$. For any compact $\Omega \subset \matR^n$, we denote by 
$\norm{g}_{\Linfnorm (\Omega)}$, the $\Linfnorm$ norm of $g$ in $\Omega$.
The partial derivative operator $\frac{\partial}{\partial x_i}$ is denoted by $\partial_i$, for $i=1,\dots,n$. 
So for instance, $\frac{\partial^3 g}{\partial x_1^2 \partial x_2}$ will be denoted by $\partial_1^2 \partial_2 g$.

We are interested in the problem of approximating functions $f:\matR^d \rightarrow \matR$ from point queries.
For some unknown sets $\univsupp \subset [\dimn], \bivsupp \subset {[\dimn] \choose 2}$, 
the function $f$ is assumed to have the following form. 
\begin{equation} \label{eq:gspam_form}
 f(x_1,\dots,x_d) = \sum_{p \in \univsupp}\phi_{p} (x_p) + \sum_{\lpair \in \bivsupp}\phi_{\lpair} \xlpair.
\end{equation}
Hence $f$ is considered to be a sum of a sparse number of univariate and bivariate functions, denoted by $\phi_p$ and $\phi_{\lpair}$ respectively. 
Here, $\phi_{\lpair}$ is considered to be ``truly bivariate'' meaning that $\partial_l \partial_{\lp} \phi_{\lpair} \not\equiv 0$. 
The set of coordinate variables that are in $\bivsupp$, is denoted by 
\begin{equation}
\bivsuppvar := \set{l \in [\dimn]: \exists l^{\prime} \in [d] \ \text{s.t} \ \lpair \in \bivsupp \ \text{or} \ \lpairi \in \bivsupp}. 
\end{equation}
For each $l \in \bivsuppvar$, we refer to the number of occurrences of $l$ in $\bivsupp$, 
as the \emph{degree} of $l$, formally denoted as follows. 
\begin{equation}
\degree(l) := \abs{\set{l^{\prime} \in \bivsuppvar : \lpair \in \bivsupp \ \text{or} \ \lpairi \in \bivsupp}}; \quad l \in \bivsuppvar. 
\end{equation}
\paragraph{Model Uniqueness.} We first note that representation \eqref{eq:gspam_form} is not unique. Firstly, we could add constants to
each $\phi_l, \phi_{\lpair}$, which sum up to zero. Furthermore, for each $l \in \bivsuppvar$ with $\degree(l) > 1$ we could add
univariates that sum to zero. We can do the same for $l \in \univsupp \cap \bivsuppvar : \degree(l) = 1$.
These ambiguities are thankfully avoided by re-writing \eqref{eq:gspam_form} \emph{uniquely} in the following ANOVA form. 
\begin{equation} \label{eq:unique_mod_rep}
f(x_1,\dots,x_d) = c + \sum_{p \in \univsupp}\phi_{p} (x_p) + \sum_{\lpair \in \bivsupp} \phi_{\lpair} \xlpair + 
\sum_{q \in \bivsuppvar: \degree(q) > 1} \phi_{q} (x_q); \quad \univsupp \cap \bivsuppvar = \emptyset.
\end{equation}
Here, $c = \expec[f]$ and $\expec_p[\phi_p] = \expec_{\lpair}[\phi_{\lpair}] = 0$; $\forall p \in \univsupp, \lpair \in \bivsupp$, 
with expectations being over uniform distributions w.r.t. variable range $[-1,1]$.   
\hemant{In addition, certain bivariate components have zero marginal mean with respect to either $l$ or $\lp$. In particular, $\expec_{l}[\phi_{\lpair}] = 0$ if $\degree(\lp) > 1$ and $\expec_{\lp}[\phi_{\lpair}] = 0$ if $\degree(l) > 1$}. 
The univariate $\phi_q$ corresponding to $q \in \bivsuppvar$ with $\degree(q) > 1$, 
represents the net marginal effect of the variable and has $\expec_q[\phi_q] = 0$. We note that $\univsupp, \bivsuppvar$ are disjoint in \eqref{eq:unique_mod_rep}.
This is due to the fact that each $p \in \univsupp \cap \bivsuppvar$ with $\degree(p) = 1$ can be merged with its bivariate form, while
each $p \in \univsupp \cap \bivsuppvar$ with $\degree(p) > 1$ can be merged with its net marginal univariate form.
The uniqueness of \eqref{eq:unique_mod_rep} is shown formally in the appendix. 

We assume the setting $\abs{\univsupp} = \univsparsity \ll \dimn$, $\abs{\bivsupp} = \bivsparsity \ll d$. 
Clearly, $\abs{\bivsuppvar} \leq 2\bivsparsity$ with equality iff elements in $\bivsupp$ are pairwise disjoint.
The set of \emph{all} active variables, \textit{i.e.}, $\univsupp \cup \bivsuppvar$ will be denoted by $\totsupp$.
We then define $\totsparsity: = \abs{\totsupp} = \univsparsity +  \abs{\bivsuppvar}$ to be 
the \emph{total sparsity} of the problem. 
The largest degree of a variable in $\bivsuppvar$, is defined to be $\maxdegree := \max_{l \in \bivsuppvar} \degree(l)$. 
Clearly, $1 \leq \maxdegree \leq \bivsparsity$. 

\paragraph{Goals.} Assuming that we have the freedom to query $f$ within its domain, our goal is now two fold.
\begin{itemize}
\item Firstly, we would like to exactly recover the unknown sets $\univsupp,\bivsupp$.

\item Secondly, we would like to estimate $c$ as well as each: (i) $\phi_p; p \in \univsupp$, (ii) $\phi_{\lpair}; \lpair \in \bivsupp$ and 
(iii) $\phi_{q}; q \in \bivsuppvar, \degree(q) > 1$, in \eqref{eq:unique_mod_rep}. 
In particular, we would like to estimate the univariate and bivariate components within compact domains 
$[-1,1]$, $[-1,1]^2$ respectively.
\end{itemize}

If $\univsupp, \bivsupp$ were known beforehand, then one can estimate $f$ via standard results from approximation theory or nonparametric 
regression \footnote{This is discussed later.}. Hence our primary focus in the paper is to recover $\univsupp, \bivsupp$.
Our main assumptions for this problem are listed below.
\begin{assumption} \label{assum:spamin_samp_reg}
We assume that $f$ can be queried from the slight enlargement: $[-(1+r),(1+r)]^d$ of $[-1,1]^d$ 
for some small $r > 0$. As will be seen later, the enlargement $r$ can be made arbitrarily close to $0$.
\end{assumption}
%

\begin{assumption} \label{assum:smooth}
We assume each $\phi_{\lpair},\phi_p$ to be three times continuously differentiable, 
within $[-(1+r),(1+r)]^2$ and $[-(1+r),(1+r)]$ respectively. Since these domains are compact, 
there then exist constants $\smconst_m \geq 0$; $m=0,1,2,3$, so that 
\begin{align}
\norm{\partial_l^{m_1} \partial_{l^{\prime}}^{m_2} \phi_{\lpair}}_{\Linfnorm[-(1+r),(1+r)]^2} \leq \smconst_m; \quad \lpair \in \bivsupp, \ m_1 + m_2 = m, \\
\norm{\partial_p^{m} \phi_{p}}_{\Linfnorm[-(1+r),(1+r)]} \leq \smconst_m; \quad p \in \univsupp \ \text{or}, \ p \in \bivsuppvar \ \& \ \degree(p) > 1.
\end{align}
\end{assumption}

Our next assumption is for the purpose of identification of active variables, \textit{i.e.}, the elements of $\univsupp \cup \bivsuppvar$.

\begin{assumption} \label{assum:actvar_iden}
For some constants $\idenconst_1, \critintmeas_1 > 0$, 
we assume that for each $\lpair \in \bivsupp$, $\exists$ 
connected $\calI_{l,1}, \calI_{l^{\prime},1}, \calI_{l,2}, \calI_{l^{\prime},2} \subset [-1,1]$, each of Lebesgue measure at least $\critintmeas_1 > 0$,
so that
\begin{align} \label{eq:parderiv_bd_assum}
\abs{\partial_l \phi_{\lpair} \xlpair}  &> \idenconst_1, \quad \forall \xlpair \in \calI_{l,1} \times \calI_{l^{\prime},1}, \\ 
\abs{\partial_{l^{\prime}} \phi_{\lpair} \xlpair}  &> \idenconst_1, \quad \forall \xlpair \in \calI_{l,2} \times \calI_{l^{\prime},2}.
\end{align}
Similarly, we assume that for each $p \in \univsupp$, $\exists$ connected $\calI_p \subset [-1,1]$,
of Lebesgue measure at least $\critintmeas_1 > 0$, such that $\abs{\partial_p \phi_p(x_p)} > \idenconst_1$, $\forall x_p \in \calI_p$.
These assumptions essentially serve to distinguish an active variable from a non-active one, and are also in a sense necessary.
For instance, if say $\partial_l \phi_{\lpair}$ was zero throughout $[-1,1]^2$, then it equivalently means that $\partial_l \phi_{\lpair}$
is only a function of $x_{l^{\prime}}$. If $\partial_l \phi_{\lpair} = \partial_{l^{\prime}} \phi_{\lpair} = 0$ in $[-1,1]^2$, then
$\phi_{\lpair} \equiv 0$  in $[-1,1]^2$. The same reasoning applies for $\phi_p$'s.
\end{assumption}

Our last assumption concerns the identification of $\bivsupp$.

\begin{assumption} \label{assum:pair_iden}
For some constants $\idenconst_2, \critintmeas_2 > 0$, we assume that 
for each $\lpair \in \bivsupp$, $\exists$ connected $\calI_{l}, \calI_{l^{\prime}} \subset [-1,1]$, each interval of 
Lebesgue measure at least $\critintmeas_2 > 0$, such that 
$\abs{\partial_l \partial_{l^{\prime}} \phi_{\lpair} \xlpair}  > \idenconst_2, \quad \forall \xlpair \in \calI_{l} \times \calI_{l^{\prime}}$.
\end{assumption}
Our problem specific parameters are: (i) $\smconst_i$; $i=0,\dots,3$, (ii) $\idenconst_j,\critintmeas_j$; $j=1,2$ and, (iii) $\totsparsity, \maxdegree$.
We do not assume $\univsparsity, \bivsparsity$ to be known but instead assume that $\totsparsity$ is known. Furthermore it suffices to use estimates for the problem 
parameters instead of exact values. In particular, we can use upper bounds for: $\totsparsity, \maxdegree$, $\smconst_i$; $i=0,\dots,3$ and lower bounds for: 
$\idenconst_j,\critintmeas_j$; $j=1,2$. 

\paragraph{Underlying interaction graph.} One might intuitively guess that the underlying ``structure'' of interactions between 
the elements in $\bivsuppvar$, shapes the difficulty of the problem. More formally, consider 
the graph $G = (V,E)$ where $V = [\dimn]$ and $E = \bivsupp \subset {V \choose 2}$ 
denote the set of vertices and edges, respectively. We refer to the induced subgraph $I_G = (\bivsuppvar, \bivsupp)$ 
of $G$, as the \emph{interaction graph}. We consider not only the general setting -- where no assumption is made on $I_G$ -- 
but also a special case where $I_G$ is a perfect matching. This is illustrated in Figure \ref{fig:inter_graphs}. 
\begin{figure}[!htp]
\begin{center}
   \subfloat[][$I_G$ is a perfect matching]{\includegraphics[width=0.3\textwidth]{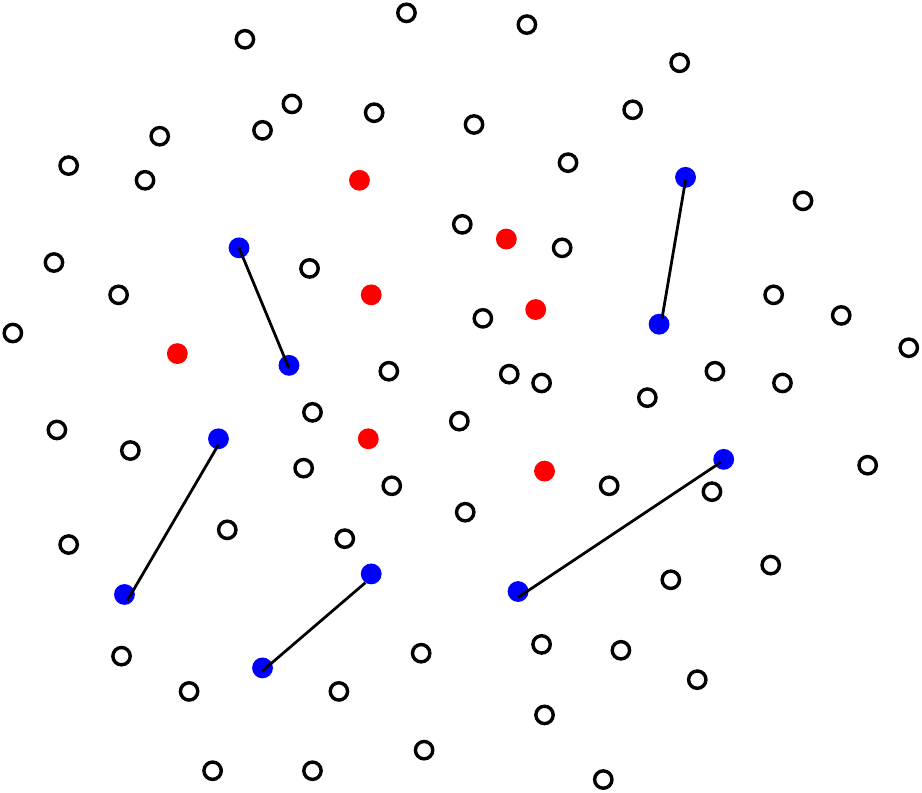} \label{subfig:perfmat}}  \hspace{4mm}
   \subfloat[][$I_G$ has arbitrary structure]{\includegraphics[width=0.3\textwidth]{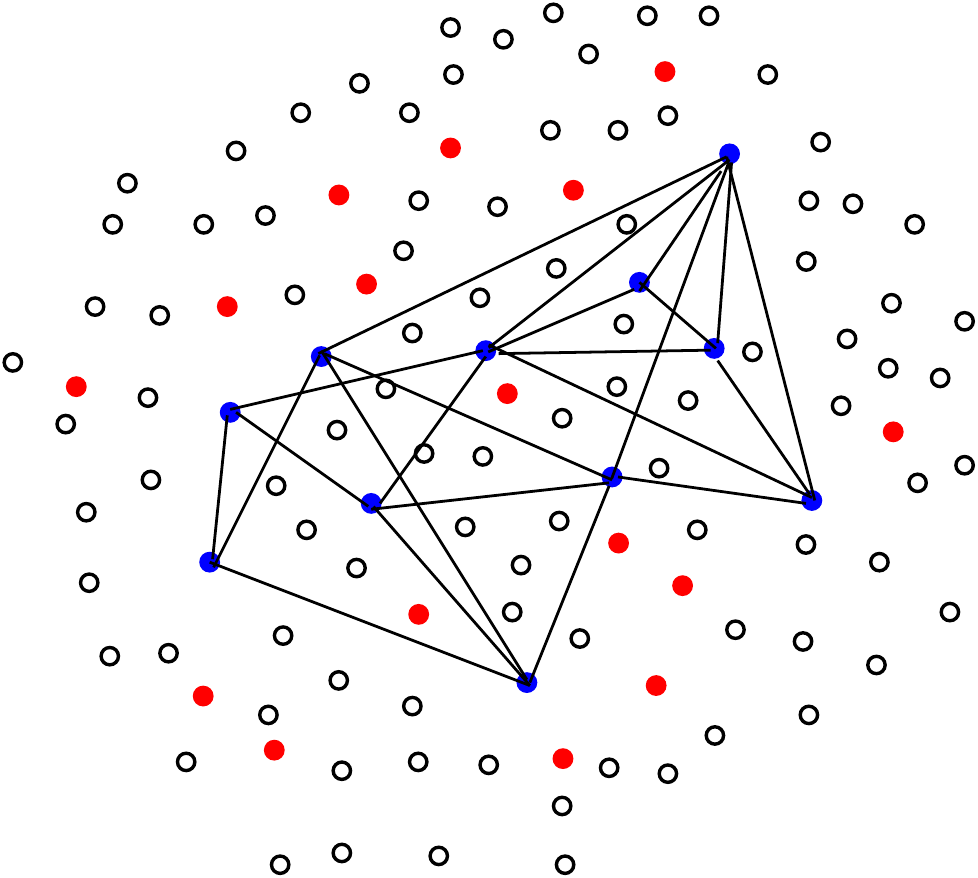} \label{subfig:arbstruc}}  \hspace{4mm}
\end{center} 
\caption{\small Blue (resp. red) disks denote elements of $\bivsuppvar$ (resp. $\univsupp$). 
Circles denote elements of $[\dimn] \setminus \{\univsupp \cup \bivsuppvar \}$. On the left, we have the special setting where 
$I_G$ is a perfect matching. On the right, we have the most general setting where no assumption is made on 
$I_G$.} \label{fig:inter_graphs}
\end{figure}
In Fig. \ref{subfig:perfmat}, $I_G$ is a perfect matching meaning that each vertex is of degree one. 
In other words, there is \emph{no overlap} between the elements of $\bivsupp$. In terms of the difficulty 
of interactions, this corresponds to the easiest setting. Fig. \ref{subfig:arbstruc} corresponds to the general setting where 
no structural assumption is placed on $I_G$. Therefore, we can now potentially \emph{have overlaps} between the elements of $\bivsupp$, 
since each element in $\bivsuppvar$ can be paired with up to $\maxdegree$ other elements. This corresponds to the hardest setting as far 
as the difficulty of interactions is concerned.

\section{Sampling scheme for the non-overlap case} \label{sec:algo_nonoverlap}
In this section we consider the special case where all elements in $\bivsupp$ are pair-wise disjoint. 
In other words, $\degree(i) = 1$, for each $i \in \bivsuppvar$. We first treat the noiseless setting in Section \ref{subsec:noiseless_no_overlap}, wherein 
the exact function values are obtained at each query. We then handle the noisy setting in Section \ref{subsec:noise_nooverlap}, where 
the function values are corrupted with external noise.

\subsection{Analysis for noiseless setting} \label{subsec:noiseless_no_overlap}
Our approach essentially consists of two
phases. In the first phase, we sample the function $f$ appropriately, and recover the \textit{complete set} of
active variables $\totsupp$. In the second phase, we focus on the reduced 
$\totsparsity$ dimensional subspace corresponding to $\totsupp$. We sample $f$ at appropriate 
points in this subspace, and consequently identify $\univsupp$ as well as $\bivsupp$. 
Let us now elaborate on these two phases in more detail.

\subsubsection{First Phase: Recovering all active variables} \label{subsec:phase_one_nooverlap} 
The crux of this phase is based on the following observation. On account 
of the structure of $f$, we see that at any $\vecx \in \matR^d$, the gradient $\grad f(\vecx) \in \matR^{\dimn}$ has the following form:
\begin{equation*}
(\grad f(\vecx))_q = \left\{
\begin{array}{rl}
\partial_q \phi_q(x_q) \quad ; & q \in \univsupp \\
\partial_q \phi_{\qpair} \xqpair \quad ; & \qpair \in \bivsupp\\
\partial_q \phi_{\qpairi} \xqpairi \quad ; & \qpairi \in \bivsupp\\
0 \quad ; & \text{otherwise}
\end{array} \right. ; \quad q=1,\dots,\dimn.
\end{equation*}
Hence $\grad f(\vecx)$ is at most $\totsparsity$-sparse, \textit{i.e.}, has at most $k$ non zero entries, 
for any $\vecx$. Note that the $q^{\text{th}}$ component of $\grad f(\vecx)$ is zero
if $q \notin \univsupp \cup \bivsuppvar$. Say we somehow recover $\grad f(\vecx)$ at sufficiently many $\vecx$'s
within $[-1,1]^{\dimn}$. Then, we would also have suitably many samples of the functions: 
$\partial_q \phi_q, \partial_{l} \phi_{\lpair}, \partial_{l^{\prime}} \phi_{\lpair}$, 
$\forall$ $p \in \univsupp, \lpair \in \bivsupp$. Specifically, if the number of samples is large enough, 
then we would have sampled each of $\partial_q \phi_q, \partial_{l} \phi_{\lpair}, \partial_{l^{\prime}} \phi_{\lpair}$, 
within their respective ``critical intervals'', as defined in Assumption \ref{assum:actvar_iden}. 
Provided that the estimation noise is sufficiently
small enough, this suggests that we should then, via a threshold operation, be able to detect all variables in $\univsupp \cup \bivsuppvar$.
We now proceed to formalize our above discussion, in a systematic manner.

\paragraph{Compressive sensing formulation.} We begin by discussing how a sparse gradient $\grad f$ can be estimated 
at any point $\vecx$, via compressive sensing (CS). 
As $f$ is $\calC^3$ smooth, therefore the Taylor's expansion of $f$ at $\vecx$, along $\vecv,-\vecv \in \matR^{\dimn}$, with step size
$\gradstep > 0$, and $\zeta = \vecx + \theta \vecv$, $\zeta^{\prime} = \vecx - \theta^{\prime} \vecv$; $\theta,\theta^{\prime} \in (0,\gradstep)$ gives us:
\begin{align} 
f(\vecx + \gradstep\vecv) &= f(\vecx) + \gradstep \dotprod{\vecv}{\grad f(\vecx)} + \frac{1}{2}\gradstep^2 \vecv^T \hess \hemant{f}(\vecx) \vecv + \thirdtayrem_3(\zeta), \label{eq:taylor_exp_1}\\
f(\vecx - \gradstep\vecv) &= f(\vecx) + \gradstep \dotprod{-\vecv}{\grad f(\vecx)} + \frac{1}{2}\gradstep^2 \vecv^T \hess \hemant{f}(\vecx) \vecv + \thirdtayrem_3(\zeta^{\prime}). \label{eq:taylor_exp_2}
\end{align}
Subtracting the above, and dividing by $2\gradstep$ leads to the standard ``central difference'' estimate of $\dotprod{\vecv}{\grad f(\vecx)}$.
\begin{align} \label{eq:taylor_exp_f}
\frac{f(\vecx + \gradstep\vecv) - f(\vecx - \gradstep\vecv)}{2\gradstep} = \dotprod{\vecv}{\grad f(\vecx)} 
+ \underbrace{\frac{\thirdtayrem_3(\zeta) - \thirdtayrem_3(\zeta^{\prime})}{2\gradstep}}_{O(\gradstep^2)}.
\end{align}
Notice that in \eqref{eq:taylor_exp_f}, the expression on the left hand side corresponds to a noisy-linear measurement of $\grad f(\vecx)$, with $\vecv$.
The ``noise'' here arises on account of the third order terms $\thirdtayrem_3(\zeta),\thirdtayrem_3(\zeta^{\prime}) = O(\gradstep^3)$, 
in the Taylor expansion. Now let the $\vecv$'s be chosen from the set:
\begin{align} \label{eq:samp_direc_set}
\calV &:= \set{v_j \in \matR^{\dimn} : v_{j,q} = \pm\frac{1}{\sqrt{\numdirec}} \ \text{w.p.} \ 1/2 \ \text{each};
 \ j=1,\dots,\numdirec \ \text{and} \ q=1,\dots,{\dimn}}. 
\end{align}
Then, employing \eqref{eq:taylor_exp_f} at each $\vecv_j \in \calV$ gives us the linear system:
\begin{equation} \label{eq:taylor_exp_f_1}
\underbrace{\frac{f(\vecx + \gradstep\vecv_j) - f(\vecx - \gradstep\vecv_j)}{2\gradstep}}_{y_j} = \dotprod{\vecv_j}{\grad f(\vecx)} 
+ \underbrace{\frac{\thirdtayrem_3(\zeta_{j}) - \thirdtayrem_3(\zeta^{\prime}_{j})}{2\gradstep}}_{\taynoissca_j}; \quad j=1,\dots,\numdirec.
\end{equation}
Denoting $\vecy = [y_1 \dots y_{\numdirec}]$, $\taynoisvec = [\taynoissca_1 \dots \taynoissca_{\numdirec}]$ and 
$\matV = [\vecv_1 \dots \vecv_{\numdirec}]^T \in \matR^{\numdirec \times {\dimn}}$, we can re-write \eqref{eq:taylor_exp_f_1}
succinctly as:
\begin{equation} \label{eq:cs_form}
 \vecy = \matV\grad f(\vecx) + \taynoisvec.
\end{equation}
As we know $\vecy, \matV$, therefore we can estimate the unknown $k$-sparse vector $\grad f(\vecx)$ via standard 
$\ell_1$ minimization \cite{Candes2006,Donoho2006}:
\begin{equation} \label{eq:l1_min_prog}
\est{\grad} f (\vecx) := \argmin{\vecy = \matV\vecz} \norm{\vecz}_1.  
\end{equation}
%
\begin{remark}
Estimating sparse gradients via compressive sensing was -- to the best of our knowledge -- first considered by Fornasier et al. \cite{Fornasier2012} 
for learning functions of the form: $f(\vecx) = g(\matA\vecx)$. It was then also employed by Tyagi et al. \cite{Tyagi14_nips} for learning 
SPAMs (without interaction terms). However, \cite{Fornasier2012,Tyagi14_nips} consider a ``forward difference'' estimate of $\dotprod{\vecv}{\grad f(\vecx)}$, 
via $(f(\vecx + \gradstep\vecv) - f(\vecx))/\gradstep$, resulting in $O(\gradstep)$ perturbation error in \eqref{eq:taylor_exp_f}.
\end{remark}

\begin{remark} \label{rem:connec_spsa_alg}
The above sampling mechanism is related to the ``simultaneous perturbation'' gradient approximation
method of \cite{Spall92}. Specifically in \cite{Spall92}, for a random $\vecv = (v_1,\dots,v_{\dimn})^T \in \matR^{\dimn}$, 
$\est{\grad} f(\vecx)$ is defined to be:

\begin{equation}
\left(\frac{f(\vecx + \gradstep\vecv) - f(\vecx - \gradstep\vecv)}{2\gradstep v_1},\dots,\frac{f(\vecx + \gradstep\vecv) - f(\vecx - \gradstep\vecv)}{2\gradstep v_{\dimn}} \right)^T 
\end{equation}
The bias of the above estimate can be shown to be $O(\gradstep^2)$ for $C^3$ smooth $f$. 
\end{remark}
The following theorem from \cite{Fornasier2012} provides guarantees for 
stable recovery via $\ell_1$ minimization: $\triangle(\vecy) := \argmin{\vecy = \matV \vecz} \norm{\vecz}_1$. 
While the first part is by now standard (see for example \cite{Baraniuk2008_simple}), 
the second result was stated in \cite{Fornasier2012} as a specialization of Theorem 1.2 from 
\cite{Wojta2012} to the case of Bernoulli measurement matrices.

\begin{theorem}[\cite{Wojta2012,Fornasier2012}] \label{thm:sparse_recon_bound}
Let $\matV$ be a $\numdirec \times \dimn$ random matrix with all entries being Bernoulli i.i.d random 
variables scaled with $1/\sqrt{\numdirec}$. Then the following results hold.
\begin{enumerate}
\item Let $0 < \ripconst < 1$. Then there are two positive constants $c_1,c_2 > 0$, such that the matrix $\matV$
has the Restricted Isometry Property
\begin{equation}
(1-\ripconst) \norm{\vecw}_2^2 \leq \norm{\matV\vecw}_2^2 \leq (1+\ripconst) \norm{\vecw}_2^2 \label{eq:RIP}
\end{equation}
for all $\vecw \in \matR^{\dimn}$ such that $|$supp($\vecw$)$|$ $\leq c_2 \numdirec/ \log(\dimn/\numdirec)$ with probability at least $1-e^{-c_1 \numdirec}$.

\item Let us suppose $\dimn > (\log 6)^2 \numdirec$. Then there are positive constants $C, c_1^{\prime}, c_2^{\prime} > 0$ such that 
with probability at least $1 - e^{-c_1^{\prime} \numdirec} - e^{-\sqrt{\numdirec\dimn}}$ the matrix $\matV$ has the following property.
For every $\vecw \in \matR^{\dimn}$, $\taynoisvec \in \matR^{\numdirec}$ and every natural number 
$K \leq c_2^{\prime} \numdirec / \log(\dimn/\numdirec)$, we have
\begin{equation} 
\norm{\triangle(\matV\vecw + \taynoisvec) - \vecw}_2 \leq C \left(K^{-1/2} \sigma_{K}(\vecw)_{1} + 
\max\set{\norm{\taynoisvec}_2, \sqrt{\log \dimn}\norm{\taynoisvec}_{\infty}}\right), \label{eq:sparse_recon_err}
\end{equation}
where
\begin{equation*}
\sigma_{K}(\vecw)_{1} := \inf\set{\norm{\vecw - \vecz}_1 : |\text{supp}(\vecz)| \leq K}
\end{equation*}
is the best $K$-term approximation of $\vecw$.
\end{enumerate}
\end{theorem}
\begin{remark}
The proof of the second part of Theorem \ref{thm:sparse_recon_bound} requires \eqref{eq:RIP} to hold, which is
the case in our setting with high probability. 
\end{remark}
\begin{remark} \label{rem:l1min_samp_bd}
Since $\numdirec \geq K$ is necessary, note that $K \leq c_2^{\prime} \numdirec / \log(d/\numdirec)$ is satisfied if 
$\numdirec > (1 / c_2^{\prime}) K \log(d/K)$. Also note that $K \log (d/K) > \log d$ in the 
regime\footnote{More precisely, if $d > K^{\frac{K}{K-1}}$.} $K \ll d$. \hemant{As pointed out by a reviewer, a slight improvement over Theorem \ref{thm:sparse_recon_bound} is given by \cite[Theorem 11.10]{Foucart13} where the $\log d$ term in \eqref{eq:sparse_recon_err} is replaced with $\log (d/\numdirec)$.}
\end{remark}
\paragraph{Estimating sufficiently many gradients.} Given the discussion above, the next natural question is - how should one choose the points $\vecx$, 
where the gradient $\grad f(\vecx)$ should be estimated? Note that $f$ is composed of the sum of univariate and bivariate functions, residing on mutually 
orthogonal $1$ or $2$ dimensional canonical subspaces of $\matR^{\dimn}$. Therefore, this suggests that 
it is sufficient if our set of points -- let us call it $\baseset$ -- has the property that it provides 
a $2$-dimensional discretization of \emph{any canonical $2$ dimensional subspace of $\matR^d$}. 
In order to construct $\baseset$ we will make use of hash
functions or more specifically - a family of hash functions, defined as follows.
\begin{definition} \label{def:thash_fam}
For some $t \in \mathbb{N}$ and $j=1,2,\dots$, let $h_j : [\dimn] \rightarrow \set{1,2,\dots,t}$.
We then call the set $\thashfam = \set{\hashfn_1,\hashfn_2,\dots}$ a $(\dimn,t)$-hash family if for any distinct $i_1,i_2,\dots,i_t \in [\dimn]$,
$\exists$ $\hashfn \in \thashfam$ such that $h$ is an injection when restricted to $i_1,i_2,\dots,i_t$.
\end{definition}
Hash functions are common in theoretical computer science, and are widely used such as in finding juntas \cite{Mossel03}.
There exists a fairly simple probabilistic method using which one can construct 
$\thashfam$ of size $O(t e^t \log \dimn)$ with high probability. 
The reader is for instance, referred to Section $5$ in \cite{Devore2011} where for any constant $C_1 > 1$,
the probabilistic construction yields $\thashfam$ of size $\abs{\thashfam} \leq (C_1 + 1)t e^t \log \dimn$ 
with probability at least $1 - \dimn^{-C_1 t}$, in time linear in the output size. We note that the size of $\thashfam$ is \textit{nearly optimal} - it is known 
that the size of any such family is $\Omega(e^t \log \dimn / \sqrt{t})$ \cite{Fredman84,Korner88,Nilli94}.
There also exist efficient \textit{deterministic} constructions 
for such families of partitions, with the size of the family being 
$O(t^{O(\log t)} e^t \log \dimn)$ and which take time linear in the output size \cite{Naor95}. 
For our purposes, we consider the probabilistic construction of the family due to its smaller resulting size. 
Specifically, we consider the family $\twohashfam$ so that for any distinct $i,j$, there exists $h \in \twohashfam$ s.t $h(i) \neq h(j)$. 
Let us first define for any $\hashfn \in \twohashfam$, 
the vectors $\canvec_1(\hashfn), \canvec_2(\hashfn) \in \matR^{\dimn}$ where:
\begin{equation}
(\canvec_i(\hashfn))_q := \left\{
\begin{array}{rl}
1 \quad ; & h(q) = i, \\
0 \quad ; & \hemant{\text{otherwise}}
\end{array} \right .  \quad \text{for} \ i=1,2 \ \text{and} \ q = 1,\dots,\dimn.
\end{equation}
Given at hand $\twohashfam$, we construct our set $\baseset$ using the procedure\footnote{Such sets were used in \cite{Devore2011}
for a more general problem involving functions that are intrinsically $\totsparsity$ variate, and do not necessarily have an additive structure.} 
in \cite{Devore2011}. Specifically, for some integer $\numcen > 0$, we construct for each $h \in \twohashfam$ the set $\baseset(\hashfn)$ as:  
\begin{equation} \label{eq:baseset_hash}
\baseset(\hashfn) := \set{\vecx(\hashfn) \in [-1,1]^{\dimn}: \vecx(\hashfn) = \sum_{i=1}^{2} c_i \canvec_i(\hashfn); c_1,c_2 \in
\set{-1,-\frac{\numcen-1}{\numcen},\dots,\frac{\numcen-1}{\numcen},1}}.
\end{equation}
Note that $\baseset(\hashfn)$ consists of $(2\numcen+1)^2$ points that discretize: span$(\canvec_1(h),\canvec_2(h))$, 
within $[-1,1]^{\dimn}$, with a spacing of $1/\numcen$ along each $\canvec_i$. Given this, we obtain the complete set as
$\baseset = \cup_{h \in \twohashfam} \baseset(h)$ so that $\abs{\baseset} \leq (2\numcen+1)^2 \abs{\twohashfam}$. 
Clearly, $\baseset$ discretizes \emph{any} $2$-dimensional canonical subspace, within $[-1,1]^{\dimn}$. 

\paragraph{Recovering set of active variables.} Our scheme for recovering the set 
of active variables is outlined formally in the form of Algorithm \ref{algo:est_act}. 
At each $\vecx \in \baseset$, we obtain the estimate $\est{\grad} f(\vecx)$ via $\ell_1$ minimization. We then perform a 
thresholding operation, \textit{i.e.}, set to zero those components of $\est{\grad} f(\vecx)$, whose magnitude is below a certain threshold. 
All indices then corresponding to non zero components are identified as active variables. 
\begin{algorithm} 
\caption{Sub-routine for estimating $\totsupp$} \label{algo:est_act} 
\begin{algorithmic}[1] 

\State Construct $(\dimn,2)$-hash family $\twohashfam$ and the set $\calV$ for suitable $\numdirec \in \matZ^{+}$. 
Choose suitable $\gradstep \in \matZ^{+}$ and initialize $\totsuppest = \emptyset$.

\State Choose suitable $\numcen \in \matZ^{+}$. For each $\hashfn \in \twohashfam$ do:

\begin{enumerate}
\item Create the set $\baseset(\hashfn)$. For $\vecx_i \in \baseset(h)$; $i=1,\dots,(2\numcen+1)^2$ do:

\begin{enumerate}
\item Construct $\vecy_i$ where $(\vecy_i)_j = \frac{f(\vecx_i + \gradstep \vecv_j)-f(\vecx_i - \gradstep \vecv_j)}{2\gradstep}$; 
$j=1,\dots,\numdirec$.

\item Set $\est{\grad} f(\vecx_i) := \argmin{\vecy_i = \matV\vecz} \norm{\vecz}_1$. For suitable $\derivsamperr > 0$, 
update: \label{alg:no_over_step_gradest}
\begin{equation*}
 \totsuppest = \totsuppest \cup \set{q \in \set{1,\dots,d}: \abs{(\est{\grad} f(\vecx_i))_q} > \derivsamperr}. 
\end{equation*}
\end{enumerate}
\end{enumerate}

\end{algorithmic}
\end{algorithm}

The following Lemma provides sufficient conditions on the sampling parameters: $\numcen,\numdirec,\gradstep$ and the threshold $\derivsamperr$, 
which guarantee that $\est{\totsupp} = \totsupp$ holds. 
\begin{lemma} \label{lem:rec_act_set}
Let $\twohashfam$ be of size $\abs{\twohashfam} \leq 2(C_1 + 1)e^2 \log \dimn$ for some constant $C_1 > 1$.
Then there exist constants $c_3^{\prime} \geq 1$ and $C, c_1^{\prime} > 0$ such that for any $\numcen, \numdirec, \gradstep$ satisfying 
\begin{equation}
c_3^{\prime} \totsparsity \log(\dimn/\totsparsity) < \numdirec < \dimn/(\log 6)^2, \quad 
\numcen \geq \critintmeas_1^{-1} \quad \text{and} \quad   
\gradstep < \left( \frac{3\idenconst_1 \numdirec}{4 C \smconst_3 \totsparsity} \right)^{1/2},
\end{equation}
the choice $\derivsamperr = \frac{2 C \gradstep^2 \smconst_3 \totsparsity}{3\numdirec}$ implies that
$\totsuppest = \totsupp$ holds with probability at least $1 - e^{-c_1^{\prime} \numdirec} - e^{-\sqrt{\numdirec \dimn}} - \dimn^{-2C_1}$.
Here $\critintmeas_1, \idenconst_1, \smconst_3 > 0$ are problem specific constants defined in Section \ref{sec:problem}.
\end{lemma}
\paragraph{Query complexity.} We estimate $\grad f$ at $(2\numcen + 1)^2 \abs{\twohashfam}$ many points. For each such estimate, we query $f$ at $2\numdirec$ points, 
leading to a total of $2\numdirec(2\numcen + 1)^2 \abs{\twohashfam}$ queries. From Lemma \ref{lem:rec_act_set}, we then obtain a query complexity 
of $O(\totsparsity (\log \dimn)^2 \critintmeas_1^{-2})$ for exact recovery of the set of active variables, \textit{i.e.}, $\univsupp \cup \bivsuppvar$. 

\paragraph{Computational complexity.} The family $\twohashfam$ can be constructed\footnote{Recall discussion following Definition \ref{def:thash_fam}.} in time polynomial in $d$. 
Step \ref{alg:no_over_step_gradest} involves solving a linear program in $O(d)$ variables, which can be done efficiently up to arbitrary accuracy, 
in time polynomial in $(\numdirec,d)$ (using for instance, interior point methods (cf., \cite{Nesterov94}). 
Since we solve $O(\critintmeas_1^{-2} \log d)$ such linear programs, 
hence the overall computation time is polynomial in the number of queries and dimension $d$. 

\begin{remark}{\label{rem:IHT}}
It is worth noting that in practice, it might be preferable to replace the $\ell_1$ minimization step with a 
non-convex algorithm such as ``Iterative hard thresholding'' (IHT) (cf., \cite{Blumens2009,Blumens2010,kyrillidis2011recipes,kyrillidis2012combinatorial,kyrillidis2012hard}). 
Such methods consider solving the non-convex optimization problem: 
\begin{equation*}
  \min_{\vecz} \norm{\matV\vecz - \vecy}^2 \quad \text{s.t.} \quad \norm{\vecz}_0 \leq K
\end{equation*}
for finding a $K$-sparse solution to an under-determined linear system of equations, and generally have a lower computational complexity than their convex analogues. 
Moreover, provided $\matV$ also satisfies the Restricted Isometry Property (as stated in \ref{eq:RIP}), they then also enjoy strong theoretical guarantees, 
similar to that for convex approaches.
\end{remark}

\begin{remark}
Algorithm \ref{algo:est_act} essentially estimates $\grad f$ at $O(\log d)$ points. 
The method of Fornasier et al. \cite{Fornasier2012} is designed for a more general function class than ours and hence 
involves estimating $\grad f$ on points sampled uniformly at random from the unit sphere $\mathbb{S}^{d-1}$ -- the size of such a set is typically polynomial in $d$. 
The method of Tyagi et al. \cite{Tyagi14_nips} is tailored towards SPAMs without interactions; it essentially estimates 
$\grad f$ along a uniform one-dimensional grid (hence at constantly many points). Hence conceptually, Algorithm \ref{algo:est_act} 
is a simple generalization of the scheme of Tyagi et al. \cite{Tyagi14_nips}.
\end{remark}
\subsubsection{Second Phase: Recovering individual sets} \label{subsec:phase_two_nooverlap}
Given that we have recovered $\totsupp = \univsupp \cup \bivsuppvar$, we now proceed to see how we can recover the individual sets: $\univsupp$ and $\bivsupp$.
Let us denote w.l.o.g, $\totsupp$ to be $\set{1,2,\dots,\totsparsity}$ and also denote $g : \matR^{\totsparsity} \rightarrow \matR$ to be
\begin{equation} \label{eq:fn_g_form}
 g(x_1,x_2,\dots,x_{\totsparsity}) = c + \sum_{p \in \univsupp}\phi_{p} (x_p) + \sum_{\lpair \in \bivsupp}\phi_{\lpair} \xlpair.
\end{equation}
Here $\bivsupp \subset {[\totsparsity] \choose 2}$ with $\bivsuppvar \cap \univsupp = \emptyset$. 
We have reduced our problem to that of querying some unknown function $\totsparsity$-variate function $g$, of the form \eqref{eq:fn_g_form}, 
with queries $\vecx \in \matR^{\totsparsity}$. Indeed, this is equivalent to querying $f$ at $(\vecx)_{\calS}$, \textit{i.e.}, 
the restriction of $\vecx$ onto $\totsupp$. 

In order to identify $\univsupp$ and $\bivsupp$, let us recall the discussion in Assumption \ref{assum:pair_iden} : for any $\lpair \in \bivsupp$, 
we will have that $\exists \xlpair \in [-1,1]^2$ such that 
$\partial_l \partial_{l^{\prime}} g(\vecx) = \partial_l \partial_{l^{\prime}} \phi_{\lpair} \xlpair \neq 0$. Furthermore 
for $p \in \univsupp$ and any $p^{\prime} \neq p$, we know that 
$\partial_p \partial_{p^{\prime}} g(\vecx) \equiv 0$, $\forall \vecx \in \matR^{\totsparsity}$.
In light of this, our goal will be now to query $g$ in order to estimate the \emph{off-diagonal} entries of its Hessian $\hess g$. 
This is a natural approach as these entries contain information about the mixed second order partial derivatives of $g$. 
We now proceed towards motivating our sampling scheme.

\paragraph{Motivation behind sampling scheme.} At any $\vecx \in \matR^{\totsparsity}$ the Hessian $\hess g (\vecx)$ 
is a $\totsparsity \times \totsparsity$ symmetric matrix with the following structure. 
\begin{equation*}
(\hess g(\vecx))_{i,j} = \left\{
\begin{array}{rl}
\partial^2_i \phi_i(x_i) \quad ; & i \in \univsupp, \ i=j \\
\partial^2_i \phi_{(i,i^{\prime})}(x_i,x_{i^{\prime}}) \quad ; & (i,i^{\prime}) \in \bivsupp, \ j = i \\
\partial^2_i \phi_{(i^{\prime},i)}(x_{i^{\prime}},x_i) \quad ; & (i^{\prime},i) \in \bivsupp, \ j = i \\
\partial_i \partial_j \phi_{(i,j)}(x_i,x_{j}) \quad ; & (i,j) \in \bivsupp\\
\partial_i \partial_j \phi_{(j,i)}(x_j,x_{i}) \quad ; & (j,i) \in \bivsupp\\
0 \quad ; & \text{otherwise}
\end{array} \right. .
\end{equation*}

Note that each row of $\hess g$ has at most $2$ non zero entries. 
If $i \in \univsupp$, then the non zero entry can only be the $(i,i)^{th}$ entry of $\hess g$.
If $i \in \bivsuppvar$, then the $i^{th}$ row can have two non zero entries. 
In this case, the non zero entries will be the $(i,i)^{th}$ and $(i,j)^{th}$ entries 
of $\hess g$, if $(i,j) \in \bivsupp$ or $(j,i) \in \bivsupp$.

Now, for $\vecx,\vecv \in \matR^{\totsparsity}$, $\hessstep > 0$, consider the Taylor expansion of $\grad g$ at 
$\vecx$ along $\vecv$, with step size $\hessstep$. 
For $\zeta_i = \vecx + \theta_i \vecv$, for some $\theta_i \in (0,\hessstep)$; $i=1,\dots,\totsparsity$, we have:
\begin{equation} \label{eq:grad_tay_exp}
\frac{\grad g(\vecx + \hessstep\vecv) - \grad g(\vecx)}{\hessstep} = \hess g(\vecx) \vecv + \frac{\hessstep}{2} \gradtayrem.
\end{equation} 
Alternately, we have the following identity for each individual $\partial_i g$.
\begin{equation} \label{eq:pardev_tay_exp}
\frac{\partial_i g(\vecx + \hessstep\vecv) - \partial_i g(\vecx)}{\hessstep} = \dotprod{\grad \partial_i g(\vecx)}{\vecv} + 
\frac{\hessstep}{2} \vecv^T \hess \partial_i g(\zeta_i) \vecv; \quad i = 1,\dots,\totsparsity.
\end{equation}
Say we estimate $\partial_i g(\vecx),\partial_i g(\vecx + \hessstep\vecv)$ with $\est{\partial_i} g(\vecx),\est{\partial_i} g(\vecx + \hessstep\vecv)$ respectively, 
using finite differences with step size parameter $\pardevstep > 0$. Then we can write
\begin{equation} \label{eq:pardev_est_exps}
\est{\partial_i} g(\vecx) = \partial_i g(\vecx) + \pardevestnois_i(\vecx,\pardevstep), \quad \est{\partial_i} g(\vecx + \hessstep\vecv) = 
\partial_i g(\vecx + \hessstep\vecv) + \pardevestnois_i(\vecx + \hessstep\vecv,\pardevstep)
\end{equation}
with $\pardevestnois_i(\vecx,\pardevstep),\pardevestnois_i(\vecx + \hessstep\vecv,\pardevstep) = O(\pardevstep^2)$ 
being the corresponding estimation errors. Plugging these estimates in \eqref{eq:pardev_tay_exp}, we finally obtain the following.
\begin{equation} \label{eq:pardev_lin_eq}
\frac{\est{\partial_i} g(\vecx + \hessstep\vecv) - \est{\partial_i} g(\vecx)}{\hessstep} = \dotprod{\grad \partial_i g(\vecx)}{\vecv} + 
\underbrace{\frac{\hessstep}{2} \vecv^T \hess \partial_i g(\zeta_i) \vecv + 
\frac{\pardevestnois_i(\vecx + \hessstep\vecv,\pardevstep) - \pardevestnois_i(\vecx,\pardevstep)}{\hessstep}}_{\text{Error term}}.
\end{equation}
We see in \eqref{eq:pardev_lin_eq} that the L.H.S can be viewed as taking a noisy linear measurement of the $i^{th}$ row of $\hess g(\vecx)$ with
 measurement vector $\vecv$. Hence for any $i \in \totsupp$ we can via \eqref{eq:pardev_lin_eq} hope to recover 
the $2$ sparse vector: $\grad \partial_i g(\vecx) \in \matR^{\totsparsity}$.
In fact, we are only interested in estimating the \emph{off-diagonal} entries of $\hess g$. Therefore while testing for $i \in \totsupp$, 
we can fix the $i^{th}$ component of $\vecv$ to be zero. This means that $\grad \partial_i g$ can in fact be considered as a 
$1$ sparse vector, and our task is to find the location of the non zero entry.
We now describe our sampling scheme that accomplishes this, by performing a  binary search over $\grad \partial_i g$. 

\paragraph{Sampling scheme.} Say that we are currently testing for variable $i \in \totsupp$, \textit{i.e.}, we would like to determine whether it is in $\univsupp$ or 
$\bivsuppvar$. Denote $\calT$ as the set of variables that have been classified so far.
We will first create our set of points $\baseset_i$ at which $\grad \partial_i g$ will be estimated, as follows. 
Consider $\canvec_1(i), \canvec_2(i) \in \matR^{\totsparsity}$ where for $j=1,\dots,\totsparsity$:
\begin{equation} \label{eq:canvec_hessbase}
(\canvec_1(i))_j := \left\{
\begin{array}{rl}
1 \quad ; & j = i, \\
0 \quad ; & \text{otherwise}
\end{array} \right .,  \quad 
(\canvec_2(i))_j := \left\{
\begin{array}{rl}
0 \quad ; & j = i \ \text{or} \ j \in \calT, \\
1 \quad ; & \text{otherwise}
\end{array}. \right .
\end{equation}
We then form the following set of points which corresponds to a discretization of the $2$-dimensional space spanned by $\canvec_1(i),\canvec_2(i)$, within $[-1,1]^{\totsparsity}$.
\begin{equation} \label{eq:base_hess_est}
\baseset_i := \set{\vecx \in [-1,1]^{\totsparsity} : \vecx = c_1\canvec_1(i) + c_2\canvec_2(i); c_1,c_2 \in \set{-1,-\frac{\numcenpair-1}{\numcenpair},\dots,\frac{\numcenpair-1}{\numcenpair},1}}.
\end{equation}
Now for each $\vecx \in \baseset_i$ and suitable step size parameter $\pardevstep > 0$, we will obtain the 
samples $g(\vecx + \pardevstep \canvec_1(i)), g(\vecx - \pardevstep \canvec_1(i))$. 
Then, we obtain via \emph{central differences}, the estimate: 
$\est{\partial_i} g(\vecx) = (g(\vecx + \pardevstep \canvec_1(i)) - g(\vecx - \pardevstep \canvec_1(i)))/(2\pardevstep)$. 
For our choice of $\vecv$ and parameter $\hessstep > 0$, we can similarly obtain $\est{\partial_i} g(\vecx + \hessstep \vecv)$. 
We now describe how the measurement vectors $\vecv$ can be chosen in an adaptive fashion, in order to identify $\univsupp,\bivsupp$. 

Firstly, we create a vector $\vecv_0(i)$ that enables us to test, whether there exists a variable $j \neq i$ such that $(i,j) \in \bivsupp$ (if $i > j$)
or $(j,i) \in \bivsupp$ (if $j > i$). To this end, we set $\vecv_0(i) = \canvec_2(i)$. 
Clearly, $i \in \bivsuppvar$ iff there exists $\vecx \in [-1,1]^{\totsparsity}$ such that $\dotprod{\grad \partial_i g(\vecx)}{\vecv_0(i)} \neq 0$. 
This suggests the following strategy. For each $\vecx \in \baseset_i$, we compute 
$(\est{\partial_i} g(\vecx + \hessstep\vecv_0(i)) - \est{\partial_i} g(\vecx))/(\hessstep)$ -- this will be a noisy estimate 
of $\dotprod{\grad \partial_i g(\vecx)}{\vecv_0(i)}$. Provided that the number of points is large enough and the noise is made suitably small, 
we see that via a threshold based procedure as in the previous phase, one would be able to correctly classify the other variable as either 
belonging to $\univsupp$ or $\bivsupp$. 
In case the above procedure classifies $i$ as being a part of $\bivsuppvar$, then we would still need to identify the other variable 
$j \in \bivsuppvar$, forming the pair. This can be handled via a binary search based procedure, as follows.

The measurement vectors $\vecv_1(i), \vecv_2(i), \dots$ are chosen adaptively, meaning that the choice of $\vecv_j(i)$ depends on 
the past choices: $\vecv_1(i), \dots,\vecv_{j-1}(i)$. $\vecv_1(i)$ is constructed as follows. 
We construct an equipartition $\binspart_1(i), \binspart_2(i) \subset \totsupp \setminus \set{\calT \cup \set{i}}$ such that: 
$\binspart_1(i) \cup \binspart_2(i) = \totsupp \setminus \set{\calT \cup \set{i}}$, $\binspart_1(i) \cap \binspart_2(i) = \emptyset$, 
$\abs{\binspart_1(i)} = \floor{\frac{\totsparsity - 1 - \abs{\calT}}{2}}$ and 
$\abs{\binspart_2(i)} = \totsparsity - 1 - \abs{\calT} - \abs{\binspart_1(i)}$. 
Then $\vecv_1(i)$ is chosen to be such that:
\begin{equation} \label{eq:binsearch_vec}
(\vecv_1(i))_l := \left\{
\begin{array}{rl}
1 \quad ; & l \in \binspart_1(i), \\
0 \quad ; & \text{otherwise}
\end{array} \right . ;\quad l = 1,\dots,\totsparsity. 
\end{equation}
Let $\vecx^{*} \in \baseset_i$ be the point, at which $\vecv_0(i)$ detects $i$.  We now find: 
$(\est{\partial_i} g(\vecx^{*} + \hessstep\vecv_1(i)) - \est{\partial_i} g(\vecx^{*}))/\hessstep$, 
and test whether it is larger then a certain threshold. This tells us whether the other active variable $j$ belongs to 
$\binspart_1(i)$ or to $\binspart_2(i)$. Then, we create $\vecv_2(i)$ by
partitioning the identified subset, in the same manner as $\vecv_1(i)$ and perform the same tests again. It is clear that we would need at most 
$\ceil{\log (\totsparsity-\abs{\calT})}$ many $\vecv(i)$'s in this process. Hence, if $i \in \bivsuppvar$ then we would need at most 
$\ceil{\log (\totsparsity-\abs{\calT})} + 1$
measurement vectors in order to find the other member of the pair in $\bivsuppvar$. In case $i \in \univsupp$, then $\vecv_0(i)$ by itself suffices.
The above procedure is outlined formally in Algorithm \ref{algo:est_ind_sets}. 
%
\begin{algorithm}[!htp] 
\caption{Sub-routine for estimating $\univsupp,\bivsupp$} \label{algo:est_ind_sets} 
\begin{algorithmic}[1] 

\State Initialize $\est{\univsupp},\est{\bivsupp} = \emptyset$.

\While{$\totsupp \setminus \set{\est{\univsupp} \cup \est{\bivsuppvar}} \neq \emptyset$} \label{step:main_while_loop}

\State Choose $i \in \totsupp \setminus \set{\est{\univsupp} \cup \est{\bivsuppvar}}$. For suitable $\numcenpair \in \matZ^{+}$,  
construct $\baseset_i$ as in \eqref{eq:base_hess_est}. Set $\vecv_0(i) = \canvec_2(i)$. \label{step:choose_proc_var}

\State Choose $\vecx \in \baseset_i$ that has not yet been chosen. \label{step:choose_base_pt}

\begin{enumerate}
\item Obtain estimates: $\est{\partial_i} g(\vecx), \est{\partial_i} g(\vecx + \hessstep \vecv_0(i))$ via central differences, for 
suitable $\hessstep, \pardevstep > 0$.

\item If $\frac{\abs{\est{\partial_i} g(\vecx + \hessstep \vecv_0(i)) - \est{\partial_i} g(\vecx)}}{\hessstep} > \hesssamperr$, 
then denote $\vecx^{*} \leftarrow \vecx$ and go to \ref{step:start_bin_sear}. Else goto \ref{step:choose_base_pt}. \label{step:thresh_test}
\end{enumerate}

\State Update $\est{\univsupp} = \est{\univsupp} \cup \set{i}$ and go to \ref{step:main_while_loop}.

\State Set $\rootset = \totsupp \setminus \set{\set{i} \cup \est{\univsupp} \cup \est{\bivsuppvar}}$. \label{step:start_bin_sear}

\While{$\abs{\rootset} > 1$} 

\State Initialize $\binspart_1(i),\binspart_2(i)$ as equipartition of $\rootset$. Construct $\vecv(i)$
w.r.t. $\binspart_1(i),\binspart_2(i)$ as defined in \eqref{eq:binsearch_vec}. \label{step:bins_1}

\State Obtain: $\est{\partial_i} g(\vecx^{*} + \hessstep \vecv(i))$. If 
$\frac{\abs{\est{\partial_i} g(\vecx^{*} + \hessstep \vecv(i)) - \est{\partial_i} g(\vecx^{*})}}{\hessstep} > \hesssamperr$, 
then $\rootset \leftarrow \binspart_1(i)$ else $\rootset \leftarrow \binspart_2(i)$. \label{step:bins_2}

\EndWhile

\State Denote $\rootset = \set{j}$. If $i < j$ then $\est{\bivsupp} = \est{\bivsupp} \cup \set{(i,j)}$, else 
$\est{\bivsupp} = \est{\bivsupp} \cup \set{(j,i)}$.

\EndWhile
\end{algorithmic}
\end{algorithm}

We now provide sufficient conditions on the parameters 
$\numcenpair > 0, \pardevstep$ and $\hessstep > 0$, along with a corresponding threshold, that together guarantee recovery of $\univsupp$ and $\bivsupp$.
This is stated in the following lemma.
\begin{lemma}\label{lem:est_ind_sets}
Let $\numcenpair > 0, \pardevstep$ and $\hessstep > 0$ be chosen to satisfy:
\begin{align}
\numcenpair \geq \critintmeas_2^{-1}, \quad
\pardevstep < \frac{\sqrt{3} \idenconst_2}{4\sqrt{2} \smconst_3}, \quad 
\hessstep \in \left(\frac{\idenconst_2 - \sqrt{\idenconst_2^2 - (32/3)\pardevstep^2 \smconst_3^2}}{8\smconst_3}, \frac{\idenconst_2 + 
\sqrt{\idenconst_2^2 - (32/3)\pardevstep^2 \smconst_3^2}}{8\smconst_3} \right).
\end{align}
Then for the choice $\hesssamperr = \frac{\pardevstep^2 \smconst_3}{3\hessstep} + 2\hessstep\smconst_3$,  we have for Algorithm \ref{algo:est_ind_sets} 
that $\est{\univsupp} = \univsupp$ and $\est{\bivsupp} = \bivsupp$. Here, $\smconst_3, \idenconst_2, \critintmeas_2 > 0$ are problem specific constants, 
defined in Section \ref{sec:problem}. 
\end{lemma}
\paragraph{Query complexity.} Note that for each $i \in \univsupp$ we make at most $4 {\numcenpair}^{2}$ queries. 
This is clear from Step \ref{step:choose_base_pt}: four queries are made for estimating the two 
partial derivatives and this is done at most ${\numcenpair}^2$ times. If $i \in \bivsuppvar$, then we notice that in Step \ref{step:bins_2}, we make two queries for each 
$\vecv(i)$ leading to at most $2\ceil{\log \totsparsity}$ queries during Steps \ref{step:bins_1}--\ref{step:bins_2}. In addition, we still make at 
most $4 {\numcenpair}^{2}$ queries during Step \ref{step:choose_base_pt}, as discussed earlier. Hence 
the total number of queries made is at most:
\begin{equation}
 \univsparsity \cdot 4 {\numcenpair}^{2} + \bivsparsity \cdot \left(4 {\numcenpair}^{2} + 2\ceil{\log \totsparsity} \right) 
< \totsparsity(4 {\numcenpair}^{2} + 2\ceil{\log \totsparsity}).
\end{equation}
Since $\numcenpair \geq \critintmeas_2^{-1}$, the query complexity for this phase is 
$O(\totsparsity(\critintmeas_2^{-2} + \log \totsparsity))$.

\paragraph{Computational complexity.} It is clear that the overall computation time is linear in the the number of queries and hence 
at most polynomial in $\totsparsity$.

\subsection{Analysis for noisy setting} \label{subsec:noise_nooverlap}
We now analyse the noisy setting where at each query $\vecx$, we observe: $f(\vecx) + \exnoisep$, 
where $\exnoisep \in \matR$ denotes external noise. In order to see how this affects Algorithm \ref{algo:est_act}, 
\eqref{eq:cs_form} now changes to $\vecy = \matV\grad f(\vecx) + \taynoisvec + \exnoisevec$, where 
$\exnoise_{j} = (\exnoisep_{j,1} - \exnoisep_{j,2})/(2\gradstep)$. Therefore while the Taylor's remainder 
term $\abs{\taynoissca_j} = O(\gradstep^2)$, the external noise term $\abs{\exnoise_j}$ scales as $\gradstep^{-1}$. 
Hence in contrast to Lemma \ref{lem:rec_act_set} the step-size $\gradstep$ needs to be chosen carefully now -- 
\hemant{a value which is too small} would blow up the external noise component while a large value would increase 
perturbation due to higher order Taylor's terms.

A similar problem would occur in the next phase when we try to identify $\univsupp, \bivsupp$. 
Indeed, due to the introduction of noise, we now observe 
$g(\vecx + \pardevstep \canvec_1(i)) + \exnoisep_{i,1}$, $g(\vecx - \pardevstep \canvec_1(i)) + \exnoisep_{i,2}$.
This changes the expression for $\est{\partial_i} g(\vecx)$ in \eqref{eq:pardev_est_exps} to: 
$\est{\partial_i} g(\vecx) = \partial_i g(\vecx) + \pardevestnois_i(\vecx,\pardevstep) + \exnoise_i(\vecx,\pardevstep)$ 
where $\exnoise_i(\vecx,\pardevstep) = (\exnoisep_{i,1} - \exnoisep_{i,2})/(2\pardevstep)$. 
Recall that $\pardevestnois_i(\vecx,\pardevstep) = O(\pardevstep^2)$ corresponds to the Taylor's remainder 
term. Hence we again see that in contrast to Lemma \ref{lem:est_ind_sets}, the step $\pardevstep$ cannot be chosen too 
small now, as it would blow up the external noise component.

\paragraph{Arbitrary bounded noise.} 
In this scenario, we assume the external noise to be arbitrary and bounded, meaning 
that $\abs{\exnoisep} < \exnoisemag$, for some finite $\exnoisemag \geq 0$. Clearly, if $\exnoisemag$ is too large, then we would expect recovery 
of $\totsupp = \univsupp \cup \bivsuppvar$ to be impossible, as the structure of $f$ would be destroyed. 
However we show that if $\exnoisemag = O(\frac{\idenconst_1^{3/2}}{\sqrt{\smconst_3 \totsparsity}})$, then Algorithm \ref{algo:est_act} recovers 
the total support $\totsupp$, with appropriate choice of sampling parameters. Furthermore, assuming $\totsupp$ is recovered exactly, 
and provided $\exnoisemag$ additionally satisfies $\exnoisemag = O(\frac{\idenconst_2^{3}}{\smconst_3^2})$, then with proper choice of 
sampling parameters, Algorithm \ref{algo:est_ind_sets} identifies $\univsupp, \bivsupp$. This is stated formally in the following Theorem.
\begin{theorem} \label{thm:gen_no_over_arbnois}
Let the constants $c_3^{\prime}, C, c_1^{\prime}, C_1$ and $\twohashfam, \numcen, \numdirec$ be as defined in Lemma \ref{lem:rec_act_set}. 
Say $\exnoisemag < \exnoisemag_1 = \frac{\idenconst_1^{3/2}}{3C\sqrt{4\smconst_3 \totsparsity C}}$.  
Then for $\theta_1 = \cos^{-1}(-\exnoisemag / \exnoisemag_1)$, let $\gradstep$ be chosen to satisfy: 
\begin{equation}
\gradstep \in \left(2\sqrt{\frac{\idenconst_1\numdirec}{4\smconst_3 \totsparsity}}\cos(\theta_1/3 - 2\pi/3) , 2\sqrt{\frac{\idenconst_1\numdirec}{4\smconst_3 \totsparsity}}\cos(\theta_1/3)\right) 
\end{equation}
We then have in Algorithm \ref{algo:est_act} for the choice: 
$\derivsamperr = C\left(\frac{2\gradstep^2 \smconst_3 \totsparsity}{3\numdirec} + \frac{\exnoisemag\sqrt{\numdirec}}{\gradstep}\right)$ 
that $\totsuppest = \totsupp$ holds with probability at least $1 - e^{-c_1^{\prime} \numdirec} - e^{-\sqrt{\numdirec \dimn}} - \dimn^{-2C_1}$. 
Given that $\totsuppest = \totsupp$, let $\numcenpair$ be as defined in Lemma \ref{lem:est_ind_sets}. 
Assuming $\exnoisemag < \frac{\idenconst_2^3}{384\sqrt{2}\smconst_3^2} = \exnoisemag_2$ holds, 
then for $\theta_2 = \cos^{-1}(-\exnoisemag / \exnoisemag_2)$ let $\pardevstep, \hessstep$ be chosen to satisfy:
\begin{align}
 \hessstep &\in \left(\frac{\idenconst_2 - \sqrt{\idenconst_2^2 - \frac{32}{3\pardevstep}\smconst_3(\pardevstep^3\smconst_3 + 6\exnoisemag)}}{8\smconst_3}, 
\frac{\idenconst_2 + \sqrt{\idenconst_2^2 - \frac{32}{3\pardevstep}\smconst_3(\pardevstep^3\smconst_3 + 6\exnoisemag)}}{8\smconst_3}\right), \\
\pardevstep &\in \left(\frac{\idenconst_2}{2\sqrt{2}\smconst_3}\cos(\theta_2/3 - 2\pi/3), \frac{\idenconst_2}{2\sqrt{2}\smconst_3}\cos(\theta_2/3) \right). 
\end{align}
Then the choice $\hesssamperr = \frac{\pardevstep^2 \smconst_3}{3\hessstep} + 2\hessstep\smconst_3 + \frac{2\exnoisemag}{\pardevstep\hessstep}$ 
implies in Algorithm \ref{algo:est_ind_sets} that $\est{\univsupp} = \univsupp$ and $\est{\bivsupp} = \bivsupp$.
\end{theorem}

\paragraph{Stochastic noise.}
We now assume that the point queries are corrupted with i.i.d Gaussian noise, so that $\exnoisep \sim \calN(0,\sigma^2)$ 
for $\sigma^2 < \infty$. In order to reduce $\sigma$, we consider resampling each point query a sufficient number of times, and averaging the 
values. In Algorithm \ref{algo:est_act}, \textit{i.e.}, during the estimation of $\totsupp$, 
we resample each query $N_1$ times so that $\exnoisep \sim \calN(0,\sigma^2/N_1)$. 
For any $\exnoisemag > 0$, if $N_1$ is chosen large enough, then we can obtain a uniform bound $\abs{\exnoisep} < \exnoisemag$ -- via standard tail 
bounds for Gaussian's -- over all noise samples, with high probability. Consequently, the noise model transforms to a bounded noise one which means that
by choosing $\exnoisemag < \exnoisemag_1$, we can use the result of Theorem \ref{thm:gen_no_over_arbnois} for estimating $\totsupp$. 
Similarly in Algorithm \ref{algo:est_ind_sets}, we resample each query $N_2$ times so that now $\exnoisep \sim \calN(0,\sigma^2/N_2)$. 
For any $\exnoisemagp > 0$, and $N_2$ large enough, we can again uniformly bound $\abs{\exnoisep} < \exnoisemagp$ with high probability.  
By now choosing $\exnoisemagp < \exnoisemag_2$, we can then use the result of Theorem \ref{thm:gen_no_over_arbnois} for estimating $\univsupp, \bivsupp$.
These conditions are stated formally in the following Theorem.
\begin{theorem} \label{thm:gen_no_over_gauss}
Let the constants $c_3^{\prime}, C, c_1^{\prime}, C_1$ and $\twohashfam, \numcen, \numdirec$ be as defined in 
Lemma \ref{lem:rec_act_set}. For any $\exnoisemag < \exnoisemag_1 = \frac{\idenconst_1^{3/2}}{3C\sqrt{4\smconst_3 \totsparsity C}}$, 
$0 < p_1 < 1$, $\theta_1 = \cos^{-1}(-\exnoisemag / \exnoisemag_1)$, say we resample each query in 
Algorithm \ref{algo:est_act}, $N_1 > \frac{\sigma^2}{\exnoisemag^2} {\hemantt \log (\frac{2}{p_1}\numdirec(2\numcen+1)^2\abs{\twohashfam})}$
times, and average the values. Then by choosing $\gradstep$ and $\derivsamperr$ as in Theorem \ref{thm:gen_no_over_arbnois}, 
we have that $\totsuppest = \totsupp$ holds with probability at least $1 - p_1 - e^{-c_1^{\prime} \numdirec} - e^{-\sqrt{\numdirec \dimn}} - \dimn^{-2C_1}$.

Given that $\totsuppest = \totsupp$, let $\numcenpair$ be as defined in Lemma \ref{lem:est_ind_sets}. For any 
$\exnoisemagp < \frac{\idenconst_2^3}{384\sqrt{2}\smconst_3^2} = \exnoisemag_2$, $0 < p_2 < 1$, $\theta_2 = \cos^{-1}(-\exnoisemagp / \exnoisemag_2)$, 
say we resample each query in Algorithm \ref{algo:est_ind_sets}, 
$N_2 > {\hemantt \frac{\sigma^2}{{\exnoisemagp}^2} \log\left(\frac{2}{p_2}(\totsparsity(2{\numcenpair}^2 + \lceil\log \totsparsity \rceil)) \right)}$
times. Then by choosing $\pardevstep, \hessstep, \hesssamperr$ as in Theorem \ref{thm:gen_no_over_arbnois}, we have that 
$\est{\univsupp} = \univsupp$ and $\est{\bivsupp} = \bivsupp$, with probability at least $1-p_2$. 
\end{theorem}

We now analyze the query complexity for the i.i.d Gaussian noise case. One can verify that $\exnoisemag_1 = O(\totsparsity^{-1/2})$. 
Since $\numdirec = O(\totsparsity \log \dimn)$, $\abs{\twohashfam} = O(\log \dimn), \numcen = O(\critintmeas_1^{-1})$, then 
by choosing $p_1 = O(\dimn^{-\delta})$ for any constant $\delta > 0$, we arrive at 
$N_1 = \hemant{O(\totsparsity \log((\dimn^{\delta}) (\totsparsity \log d) (\critintmeas_1^{-2} \log \dimn)))} = O(\totsparsity \log \dimn)$. 
This leads to a total sample complexity of $O(N_1 \totsparsity (\log \dimn)^2 \critintmeas_1^{-2}) = O(\totsparsity^2 (\log \dimn)^3 \critintmeas_1^{-2})$ 
for guaranteeing $\totsuppest = \totsupp$, with high probability. Next, we see that $\exnoisemagp = O(1)$ and thus 
$N_2 = O(\log (\totsparsity(\critintmeas_2^{-2} + \log \totsparsity)/p_2))$. 
Therefore with an additional 
$O(N_2 \totsparsity(\critintmeas_2^{-2} + \log \totsparsity)) = O(\totsparsity(\critintmeas_2^{-2} + \log \totsparsity) \log(\totsparsity/p_2))$ samples, 
we are guaranteed with probability at least $1-p_2$ that $\est{\univsupp} = \univsupp$ and $\est{\bivsupp} = \bivsupp$.

\section{Sampling scheme for the general overlap case} \label{sec:algo_gen_overlap}
We now analyze the general scenario where overlaps can occur amongst the elements of $\bivsupp$.
Therefore the degrees of the variables occurring in $\bivsuppvar$, can be greater than one. 
Contrary to the non-overlap case, we now sample $f$ in order to directly estimate 
its $\dimn \times \dimn$ Hessian $\hess{f}$, at suitably chosen points. In particular, this 
enables us to subsequently identify $\bivsupp$. Once $\bivsupp$ is identified, we are left with a SPAM -- with no variable interactions --
on the set $[\dimn] \setminus \bivsupp$. We then identify $\univsupp$ by employing the sampling scheme from \cite{Tyagi14_nips} 
on this reduced space.

\subsection{Analysis for noiseless setting} \label{subsec:noiseless_overlap_set_est}
In this section, we consider the noiseless scenario, \textit{i.e.}, we assume the exact sample $f(\vecx)$ is obtained 
for any query $\vecx$. To begin with, we explain why the sampling scheme for the non overlap case does not directly apply here. 
To this end, note that the gradient of $f$ has the following structure for each $q \in [\dimn]$.
\begin{equation*}
(\grad f(\vecx))_q = \left\{
\begin{array}{rl}
\partial_q \phi_q(x_q) \quad ; & q \in \univsupp \\
\partial_q \phi_{\qpair} \xqpair \quad ; & \qpair \in \bivsupp \ \& \ \degree(q) = 1, \\
\partial_q \phi_{\qpairi} \xqpairi \quad ; & \qpairi \in \bivsupp \ \& \ \degree(q) = 1, \\
\partial_q \phi_q(x_q) + \sum\limits_{\qpair \in \bivsupp} \partial_q \phi_{\qpair} \xqpair \\ 
+ \sum\limits_{\qpairi \in \bivsupp} \partial_q \phi_{\qpairi} \xqpairi \quad ; & q \in \bivsuppvar \ \& \ \degree(q) > 1, \\
0 \quad ; & \text{otherwise.}
\end{array} \right. 
\end{equation*}
Therefore, for any $q \in \bivsuppvar$ with $\degree(q) > 1$, we notice that $(\grad f(\vecx))_q$ is by itself the sum of
$\degree(q)$ many bivariate functions, and $\partial_q \phi_q$. This causes an issue as far as identifying $q$ -- 
via estimating $\grad{f}$ followed by thresholding -- is concerned, as was done for the non-overlap case. 
While we assume the magnitudes of $\partial_q \phi_{\qpair}$ to be sufficiently large within respective 
subsets of $[-1,1]^2$, it is not clear what that implies for $\abs{(\grad f(\vecx))_q}$. Note that 
$(\grad f(\vecx))_q \not\equiv 0$ since $q$ is an active variable.  However a lower bound on: $\abs{(\grad f(\vecx))_q}$, 
and also on the measure of the interval where it is attained, appears to be non-trivial to obtain.

\paragraph{Estimating sparse Hessian matrices} In light of the above discussion, we consider an alternative approach, 
wherein we directly estimate the 
Hessian $\hess{f}(\vecx) \in \matR^{\dimn \times \dimn}$, at suitably chosen $\vecx \in [-1,1]^{\dimn}$. 
Observe that $\hess{f}(\vecx)$ has the following structure for $i \in \bivsuppvar$ and $j=1,\dots,\dimn$:
\begin{equation*}
(\hess{f}(\vecx))_{i,j} = \left\{
\begin{array}{rl}
\partial_i^2 \phi_{(i,i^{\prime})} (x_i,x_{i^{\prime}}) \quad \quad ; & \degree(i) = 1, (i,i^{\prime}) \in \bivsupp, i = j, \\
\partial_i^2 \phi_{(i^{\prime},i)} (x_{i^{\prime}},x_i) \quad \quad ; & \degree(i) = 1, (i^{\prime},i) \in \bivsupp, i = j, \\
\partial_i^2 \phi_i(x_i) + \sum\limits_{(i,i^{\prime}) \in \bivsupp} \partial_i^2 \phi_{(i,i^{\prime})} (x_i,x_{i^{\prime}}) \\ 
+ \sum\limits_{(i^{\prime},i) \in \bivsupp} \partial_i^2 \phi_{(i^{\prime},i)} (x_{i^{\prime}},x_i) \quad \quad ; & \degree(i) > 1, i = j, \\
\partial_i \partial_j \phi_{(i,j)} (x_i,x_j) \quad \quad ; & (i,j) \in \hemant{\bivsupp}, \\
\partial_i \partial_j \phi_{(j,i)} (x_j,x_i) \quad \quad ; & (j,i) \in \hemant{\bivsupp}, \\
0 \quad ; & \text{otherwise}
\end{array} \right. ,
\end{equation*}
while if $i \in \univsupp$, we have for $j=1,\dots,\dimn$:
\begin{equation*}
(\hess{f}(\vecx))_{i,j} = \left\{
\begin{array}{rl}
\partial_i^2 \phi_{i} (x_i) \quad ; & i = j, \\
0 \quad ; & \text{otherwise}
\end{array} \right. .
\end{equation*}
The $l^{th}$ row of $\hess{f}(\vecx)$ can be denoted by 
$\grad \partial_l f(\vecx)^T \in \matR^{\dimn}$. If $l \in \univsupp$, then $\grad \partial_l f(\vecx)^T$ 
has at most one non-zero entry, namely the $l^{th}$ entry, and has all other entries equal to zero. In other words, 
$\grad \partial_l f(\vecx)^T$ is $1$-sparse for $l \in \univsupp$.
If $l \in \bivsuppvar$, then we see that $\grad \partial_l f(\vecx)^T$ will have at most 
$(\degree(l) + 1)$ non-zero entries, implying that it is $(\degree(l) + 1) \leq (\maxdegree + 1)$-sparse. 

At suitably chosen $\vecx$'s, our aim specifically is to detect the non-zero off diagonal 
entries of $\hess{f}(\vecx)$ since they correspond precisely to $\bivsupp$. To this end, we consider the 
``difference of gradients'' based approach used in Section \ref{subsec:phase_two_nooverlap}. Contrary to the setting in 
Section \ref{subsec:phase_two_nooverlap} however, we now have a 
$\dimn \times \dimn$ Hessian and have \emph{no knowledge} about the set of active variables: $\univsupp \cup \bivsuppvar$.
Therefore, the Hessian estimation problem is harder now, and requires a different sampling scheme.

\paragraph{Sampling scheme for estimating $\bivsupp$.} For $\vecx,\vecvp \in \matR^{\dimn}$, $\hessstep > 0$, consider the Taylor expansion of 
$\grad f$ at $\vecx$ along $\vecvp$, with step size $\hessstep$.
For $\zeta_i = \vecx + \theta_i \vecvp$, for some $\theta_i \in (0,\hessstep)$; $i=1,\dots,\dimn$, we obtain the following identity.
\begin{align} \label{eq:grad_tay_exp_f}
\frac{\grad f(\vecx + \hessstep\vecvp) - \grad f(\vecx)}{\hessstep} = \hess f(\vecx) \vecvp + \frac{\hessstep}{2} \gradtayremf 
= \hessrowips + \frac{\hessstep}{2} \gradtayremf.
\end{align} 
%
\begin{figure}
     \centering
     \subfloat[][]{\includegraphics[width=0.45\linewidth]{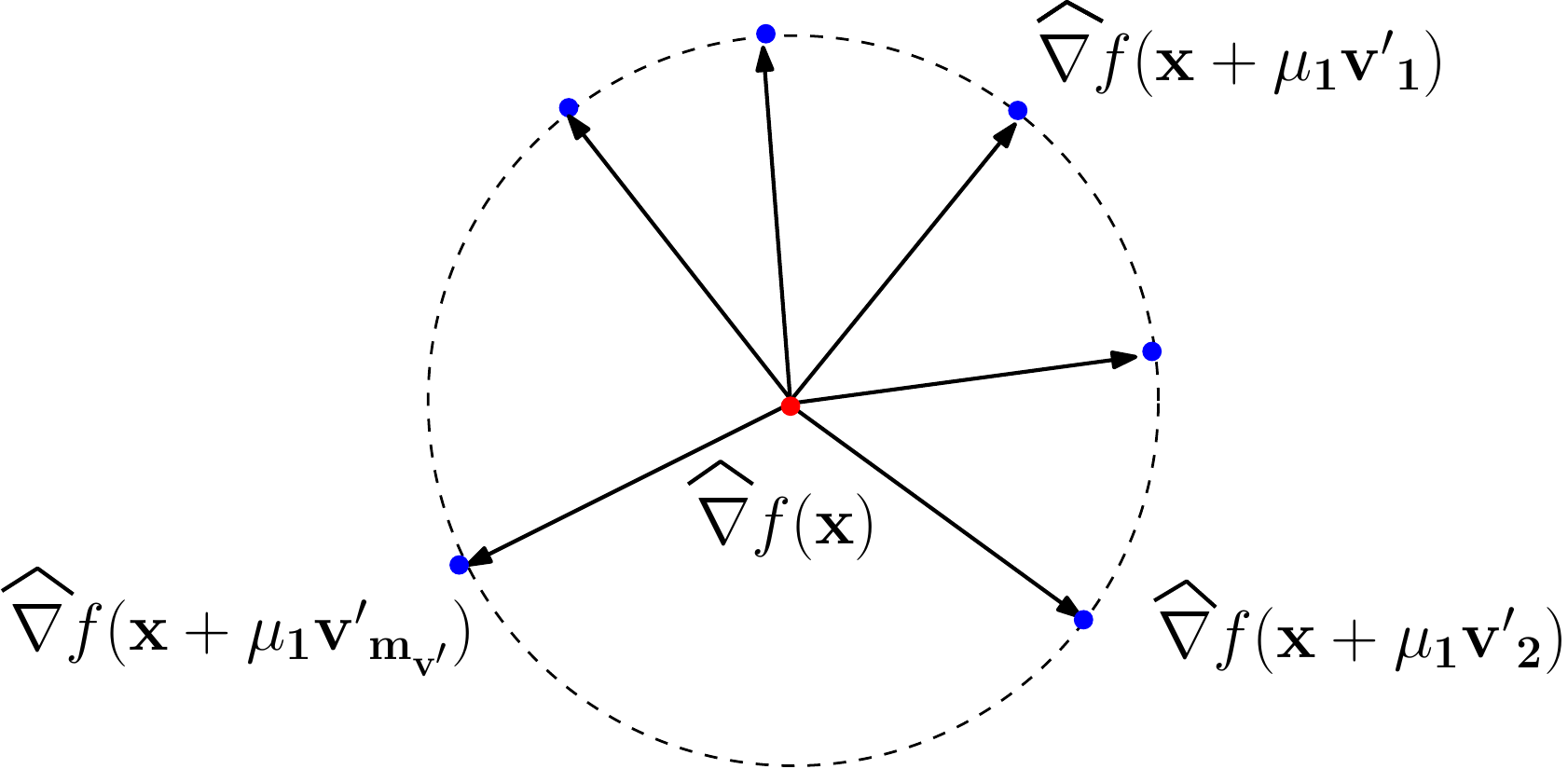}\label{fig:hess_samp}} \hspace{7mm}
     \subfloat[][]{\includegraphics[width=0.2\linewidth]{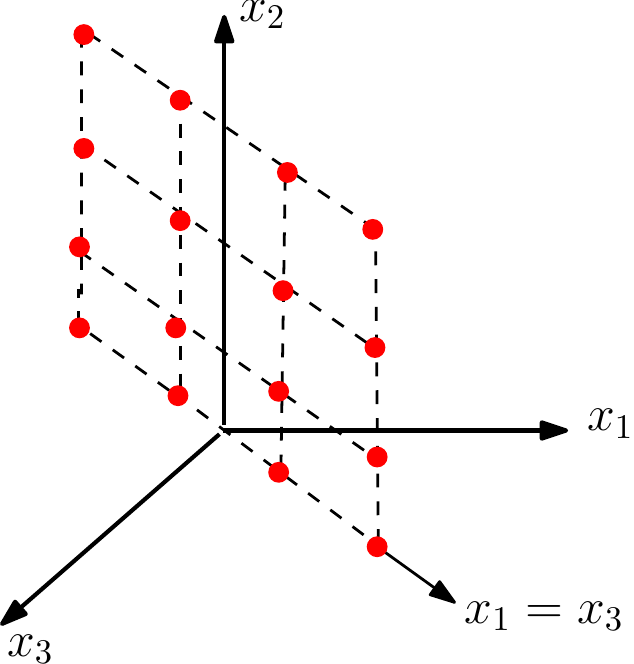}\label{fig:hash_samp}} 
     \caption{\small (a) $\hess f(\vecx)$ estimated using: $\est{\grad}f(\vecx)$ (at red disk) and 
		  neighborhood gradient estimates (at blue disks)
	    (b) Geometric picture: $\dimn = 3$, $\hashfn \in \calH_2^3$ with $\hashfn(1) = \hashfn(3) \neq \hashfn(2)$. 
	    Red disks are points in $\baseset(\hashfn)$. }
\end{figure}
We see from \eqref{eq:grad_tay_exp_f} that the $l^{th}$ entry of $(\grad f(\vecx + \hessstep\vecvp) - \grad f(\vecx))/\hessstep$,
corresponds to a linear measurement of the $l^{th}$ row of $\hess f(\vecx)$ with $\vecvp$. From the preceding discussion, we also know
that each row of $\hess f(\vecx)$ is at most $(\maxdegree + 1)$-sparse. This suggests the following idea: for any $\vecx$, if we obtain sufficiently 
many linear measurements of each row of $\hess f(\vecx)$, then we can \textit{estimate each row separately} via $\ell_1$ minimization. 
To this end, we first need an efficient way for estimating $\grad f(\vecx) \in \matR^{\dimn}$, 
at any point $\vecx$. Note that $\grad f(\vecx)$ is $\totsparsity$-sparse, therefore we can estimate it via the randomized scheme, explained 
in Section \ref{subsec:phase_one_nooverlap}, with $O(\totsparsity \log \dimn)$ queries of $f$. 
This gives us: $\est{\grad} f(\vecx) = \grad f(\vecx) + \vecw(\vecx)$, 
where $\vecw(\vecx) \in \matR^{\dimn}$ denotes the estimation noise. Plugging this in \eqref{eq:grad_tay_exp_f} results in 
the following identity.
\begin{align} \label{eq:hessrows_est_lineq}
\frac{\est{\grad} f(\vecx + \hessstep\vecvp) - \est{\grad} f(\vecx)}{\hessstep} 
= \hessrowips + \underbrace{\frac{\hessstep}{2} \gradtayremf + \frac{\vecw(\vecx + \hessstep\vecvp) - \vecw(\vecx)}{\hessstep}}_{\text{``Noise''}}.
\end{align} 
Now let $\vecvp$ be chosen from the set:
\begin{align} \label{eq:samp_direc_set_hess}
\calVp &:= \set{\vecvp_j \in \matR^{\dimn} : \vp_{j,q} = \pm\frac{1}{\sqrt{\numdirecp}} \ \text{w.p.} \ 1/2 \ \text{each};
 \ j=1,\dots,\numdirecp \ \text{and} \ q=1,\dots,{\dimn}}. 
\end{align}
Then, employing \eqref{eq:hessrows_est_lineq} at each $\vecvp_j \in \calVp$, and denoting 
$\matV^{\prime} = [\vecvp_1 \dots \vecvp_{\numdirecp}]^T \in \matR^{\numdirecp \times \dimn}$, we obtain $\dimn$ linear systems for $q=1,\dots,\dimn$:  
\begin{align} \label{eq:hessrow_est_linfin}
\underbrace{\hessrowmeas}_{\vecy_q} 
= \matV^{\prime} \grad \partial_q f(\vecx) + \underbrace{\frac{\hessstep}{2} \gradtayremfq}_{\hessestnoisa} + \underbrace{\gradestnoisq}_{\hessestnoisb}.
\end{align} 
Given the measurement vector $\vecy_q$, we can obtain the estimate $\est{\grad} \partial_q f(\vecx)$ individually for each $q$, 
via $\ell_1$ minimization:
\begin{equation} \label{eq:l1_min_prog_1}
\est{\grad} \partial_q f(\vecx) := \argmin{\vecy_q = \matV^{\prime} \vecz} \norm{\vecz}_1; \quad q=1,\dots,\dimn.
\end{equation}
Hence, we have obtained an estimate $\est{\hess}f(\vecx) := [\est{\grad} \partial_1 f(\vecx) \cdots \est{\grad} \partial_d f(\vecx)]^T$ 
of the Hessian $\hess f(\vecx)$, at the point $\vecx$. Next, we would like to have a suitable set of points $\vecx$, in the sense that it provides 
a sufficiently fine discretization, of any canonical $2$-dimensional subspace of $\matR^{\dimn}$.
To this end, we can simply consider the set $\baseset$ as defined in \eqref{eq:baseset_hash}, for the same reasons as before.

\paragraph{Sampling scheme for estimating $\univsupp$.} While the above sampling scheme enables us to recover $\bivsupp$, we can recover $\univsupp$ as follows.
Let $\est{\bivsuppvar}$ denote the set of variables in the estimated set $\est{\bivsupp}$, and let $\calP := [\dimn] \setminus \est{\bivsuppvar}$.
Assuming $\est{\bivsupp} = \bivsupp$, we have $\univsupp \subset \calP$. Therefore the model we are left with now is a 
SPAM with no variable interactions on the \emph{reduced} variable set $\calP$. 
For identification of $\univsupp$, we employ the sampling scheme of \cite{Tyagi14_nips}, 
wherein the gradient of $f$ is estimated along a discrete set of points on the line: 
$\set{(x,\dots,x) \in \matR^{\dimn} : x \in [-1,1]}$. For some $\numcenpair \in \matZ^{+}$, we denote this discrete set by:
\begin{equation}
\baseset_{\text{diag}} := \set{\vecx = (x \ x \ \cdots \ x) \in \matR^{\dimn}: x \in \set{-1,-\frac{\numcenpair-1}{\numcenpair},\dots,\frac{\numcenpair-1}{\numcenpair},1}}.
\end{equation}
Note that $\abs{\baseset_{\text{diag}}} = 2\numcenpair+1$. The motivation for  estimating $\grad f$ at $\vecx \in \baseset_{\text{diag}}$ is 
that we obtain estimates of $\partial_p \phi_p$ at equispaced points within $[-1,1]$, for $p \in \univsupp$. With a sufficiently fine discretization, 
we would ``hit'' the critical regions associated with each $\partial_p \phi_p$, as defined in Assumption \ref{assum:actvar_iden}. 
By applying a thresholding operation, we would then be able to identify each $p \in \univsupp$.
Let us denote $\calVpp$ to be the set of sampling directions in $\matR^{\dimn}$ -- analogous to 
$\calV,\calVp$ defined in \eqref{eq:samp_direc_set}, \eqref{eq:samp_direc_set_hess} respectively -- with $\abs{\calVpp} = \numdirecpp$:
\begin{equation} \label{eq:samp_direc_s1_fin}
\calVpp := \set{\vecvpp_j \in \matR^{\dimn} : \vpp_{j,q} = \pm\frac{1}{\sqrt{\numdirecpp}} \ \text{w.p.} \ 1/2 \ \text{each};
 \ j=1,\dots,\numdirecpp \ \text{and} \ q=1,\dots,{\dimn}}. 
\end{equation}
For each $\vecx \in \baseset_{\text{diag}}$, we will query $f$ at points 
$(\vecx + \gradstepp \vecvpp_j)_{\calP}, (\vecx - \gradstepp \vecvpp_j)_{\calP}$; $\vecvpp_j \in \calVpp$, restricted to  $\calP$.
Then by obtaining the measurements: 
$y_j = (f((\vecx + \gradstepp \vecvpp_j)_{\calP}) - f((\vecx - \gradstepp \vecvpp_j)_{\calP}))/(2\gradstepp); \ j=1,\dots,\numdirecpp$, 
and denoting $(\matVpp)_{\calP} = [(\vecvpp_1)_{\calP} \cdots (\vecvpp_{\numdirecpp})_{\calP}]^T$, we obtain the estimate 
$(\est{\grad} f((\vecx)_{\calP}))_{\calP} := \argmin{\vecy = (\matVpp)_{\calP}(\vecz)_{\calP}} \norm{(\vecz)_{\calP}}_1$. 
This notation simply means that we search over $\vecz \in \matR^{{\calP}}$, to form
the estimate $(\est{\grad} f((\vecx)_{\calP}))_{\calP}$. 

\begin{algorithm*}[!ht]
\caption{Algorithm for estimating $\univsupp,\bivsupp$} \label{algo:gen_overlap} 
\begin{algorithmic}[1] 
\State \textbf{Input:} $\numdirec,\numdirecp, \numcen, \numcenpair \in \matZ^{+}$; $\gradstep, \hessstep, \gradstepp > 0$; 
$\hesssamperr > 0, \derivsamperrpp > 0$.
\State \textbf{Initialization:} $\est{\univsupp}, \est{\bivsupp} = \emptyset$.
\State \textbf{Output:} Estimates $\est{\bivsupp}$, $ \est{\univsupp}$. \\
\hrulefill
\State Construct $(\dimn,2)$-hash family $\twohashfam$ and sets $\calV,\calV^{\prime}$. \label{algover:s2_step_1}
\For{$\hashfn \in \twohashfam$}
 	\State Construct the set $\baseset(\hashfn)$. 
	\For {$i = 1,\dots,(2\numcen+1)^2$ and $\vecx_i \in \baseset(\hashfn)$} \label{algover:s2_step_2}
			\State $(\vecy_i)_j = \frac{f(\vecx_i + \gradstep \vecv_j)-f(\vecx_i - \gradstep \vecv_j)}{2\gradstep}$; $j=1,\dots,\numdirec$; $\vecv_j \in \calV$. \label{algover:s2_query_1}
		\State $\est{\grad} f(\vecx_i) := \argmin{\vecy_i = \matV\vecz} \norm{\vecz}_1$. \label{algover:s2_grad_base}
		\For{$p = 1,\dots,\numdirecp$} 
			\State $(\vecy_{i,p})_j = \frac{f(\vecx_i + \hessstep\vecvp_p + \gradstep \vecv_j)-f(\vecx_i + \hessstep\vecvp_p - \gradstep \vecv_j)}{2\gradstep}$; \label{algover:s2_query_2}
			$j=1,\dots,\numdirec$; $\vecvp_p \in \calVp$.  \hfill \textsc{// Estimation of } $\bivsupp$
			\State $\est{\grad} f(\vecx_i + \hessstep\vecvp_p) := \argmin{\vecy_{i,p} = \matV \vecz} \norm{\vecz}_1$.\label{algover:s2_grad_1}	
		\EndFor
		\For{$q = 1,\dots,\dimn$}
			\State $(\vecy_q)_j = \frac{(\est{\grad} f(\vecx_i + \hessstep\vecvp_j) - \est{\grad} f(\vecx_i))_q}{\hessstep}$; 
			$j=1,\dots,\numdirecp$. \label{algover:s2_grad_2}
			\State $\est{\grad} \partial_q f(\vecx_i) := \argmin{\vecy_q = \matV^{\prime} \vecz} \norm{\vecz}_1$. \label{algover:s2_grad_hess_row}
			\State $\est{\bivsupp} = \est{\bivsupp} \cup \set{(q,q^{\prime}) : q^{\prime} \in \set{q+1,\dots,d} \ \& \ \abs{(\est{\grad} \partial_q f(\vecx_i))_{q^{\prime}}} > \hesssamperr}$.
		\EndFor
	\EndFor 
\EndFor \\
\hrulefill
\State Construct the sets $\baseset_{\text{diag}}, \calVpp$ and initialize $\calP := [\dimn] \setminus \est{\bivsuppvar}$.
\For {$i=1,\dots,(2\numcenpair+1)$ and $\vecx_i \in \baseset_{\text{diag}}$} \label{algover:s1_step} 
		\State $(\vecy_i)_j = \frac{f((\vecx_i + \gradstepp \vecvpp_j)_{\calP})-f((\vecx_i - \gradstepp \vecvpp_j)_{\calP})}{2\gradstepp}$; 
		$j=1,\dots,\numdirecpp$; $\vecv_j \in \calVpp$. \label{algover:s1_grad}
		\State $(\est{\grad} f((\vecx_i)_{\calP}))_{\calP} := \argmin{\vecy_i = (\matVpp)_{\calP}(\vecz)_{\calP}} \norm{(\vecz)_{\calP}}_1$. \hfill 
		\textsc{// Estimation of } $\univsupp$
		\State $\est{\univsupp} = \est{\univsupp} \cup \set{q \in \calP : \abs{((\est{\grad} f((\vecx_i)_{\calP})_q} > \derivsamperrpp}.$
\EndFor
\end{algorithmic}
\end{algorithm*}
The complete procedure for estimating $\univsupp,\bivsupp$, is described formally in Algorithm \ref{algo:gen_overlap}. 
Next, we provide sufficient conditions on our sampling parameters that guarantee exact recovery of $\univsupp, \bivsupp$ 
by the algorithm. This is stated in the following Theorem.
\begin{theorem} \label{thm:gen_overlap}
Let $\twohashfam$ be of size $\abs{\twohashfam} \leq 2(C + 1)e^2 \log \dimn$ for some constant $C > 1$. 
Then $\exists$ constants $c_1^{\prime},c_2^{\prime} \geq 1$ and $C_1, C_2, c_4^{\prime}, c_5^{\prime} > 0$, 
such that the following is true. Let $\numcen, \numdirec, \numdirecp$ satisfy 
\begin{equation*}
\numcen \geq \critintmeas_2^{-1}, \quad 
c_1^{\prime} \totsparsity \log\left(\frac{\dimn}{\totsparsity}\right) < \numdirec < \frac{\dimn}{(\log 6)^2}, \quad 
c_2^{\prime} \maxdegree \log\left(\frac{\dimn}{\maxdegree}\right) < \numdirecp < \frac{\dimn}{(\log 6)^2}. 
\end{equation*}
Denoting $a = \frac{(4\maxdegree+1)\smconst_3}{2\sqrt{\numdirecp}}$, $b = \frac{C_1\sqrt{\numdirecp}((4\maxdegree+1)\totsparsity)\smconst_3}{3\numdirec}$, 
let $\gradstep, \hessstep$ satisfy
\begin{align*}
\gradstep^2 < \frac{\idenconst_2^2}{16 a b C_2^2}, \ 
\hessstep \in \left((\idenconst_2/(4aC_2)) - \sqrt{(\idenconst_2/(4aC_2))^2 - (b\gradstep^2/a)}, (\idenconst_2/(4aC_2)) + \sqrt{(\idenconst_2/(4aC_2))^2 - (b\gradstep^2/a)} \right). 
\end{align*}
We then have that the choice
\begin{align*}
\hesssamperr = C_2 \left(\frac{\hessstep(4\maxdegree + 1)\smconst_3}{2\sqrt{\numdirecp}} + \frac{C_1\sqrt{\numdirecp} \gradstep^2((4\maxdegree+1)\totsparsity)\smconst_3}{3\hessstep\numdirec} \right), 
\end{align*}
implies $\est{\bivsupp} = \bivsupp$ with probability at least 
$1 - e^{-c_4^{\prime} \numdirec} - e^{-\sqrt{\numdirec \dimn}} - e^{-c_5^{\prime} \numdirecp} - e^{-\sqrt{\numdirecp \dimn}} - \dimn^{-2C}$.

Given that $\est{\bivsupp} = \bivsupp$, then $\exists$ constants $c_3^{\prime} \geq 1$ and $C_3, c_6^{\prime} > 0$, such that for $\numcenpair, \numdirecpp, \gradstepp$ satisfying 
\begin{equation*}
\numcenpair \geq \critintmeas_1^{-1}, \quad 
c_3^{\prime} (\totsparsity-\abs{\est{\bivsuppvar}}) \log\left(\frac{\abs{\calP}}{\totsparsity-\abs{\est{\bivsuppvar}}}\right) < \numdirecpp < \frac{\abs{\calP}}{(\log 6)^2}, \quad 
{\gradstepp}^2 < \frac{3\numdirecpp \idenconst_1}{C_3 (\totsparsity-\abs{\est{\bivsuppvar}}) \smconst_3}, 
\end{equation*}
the choice: $\derivsamperrpp = \frac{C_3 (\totsparsity-\abs{\est{\bivsuppvar}}) {\gradstepp}^2 \smconst_3}{6\numdirecpp}$, implies 
$\est{\univsupp} = \univsupp$ with probability at least $1 - e^{-c_6^{\prime} \numdirecpp} - e^{-\sqrt{\numdirecpp \abs{\calP}}}$.
\end{theorem}
%
\paragraph{Query complexity.} Estimating $\grad f(\vecx)$ at 
some fixed $\vecx$ requires $2\numdirec = O(\totsparsity\log \dimn)$ queries. Estimating $\hess f(\vecx)$ involves the estimation of $\grad f(\vecx)$ --  
along with an additional $\numdirecp$ gradient vectors in a neighborhood of $\vecx$ -- implying $O(\numdirec\numdirecp) = O(\totsparsity\maxdegree(\log \dimn)^2)$ 
point queries of $f$. Since $\hess f$ is estimated at all points in $\baseset$ in the worst case, this consequently implies a total query complexity of 
$O(\totsparsity\maxdegree(\log \dimn)^2 \abs{\baseset}) = O(\critintmeas_2^{-2}\totsparsity\maxdegree(\log \dimn)^3)$, for estimating $\bivsupp$. 
We make an additional $O(\critintmeas_1^{-1} (\totsparsity-\abs{\est{\bivsuppvar}}) \log (\dimn - \abs{\est{\bivsuppvar}}))$ queries of $f$, 
in order to estimate $\univsupp$. Therefore, the overall query complexity for estimating $\univsupp,\bivsupp$ is 
$O(\critintmeas_2^{-2}\totsparsity\maxdegree(\log \dimn)^3)$. 

\paragraph{Computational complexity.} The family $\twohashfam$ can be 
constructed\footnote{Recall discussion following Definition \ref{def:thash_fam}.} 
in time polynomial in $d$. For each $\vecx \in \baseset$, we first solve $\numdirecp + 1$ 
linear programs in $O(d)$ variables (Steps \ref{algover:s2_grad_base}, \ref{algover:s2_grad_1}), each solvable in time polynomial in $(\numdirec, d)$. 
We then solve $d$ linear programs in $O(d)$ variables (Step \ref{algover:s2_grad_hess_row}), each of which takes time polynomial in $(\numdirecp, d)$. 
Since this is done at $\abs{\baseset} = O(\critintmeas_2^{-2} \log d)$ many points, hence the 
overall computation time for estimation of $\bivsupp$ (and subsequently $\univsupp$) is polynomial in the number of queries, and in $d$.

\begin{remark}
In Algorithm \ref{thm:gen_overlap}, we could have optimized the procedure for identifying $\univsupp$ as follows. 
Observe that for each $\hashfn \in \twohashfam$, 
we always have a subset of points (\textit{i.e.}, $\subset \baseset(\hashfn)$) that discretize $\set{(x,\dots,x) \in \matR^{\dimn} : x \in [-1,1]}$. 
Therefore for each $\vecx$ lying in this subset, we could go through $\est{\grad} f(\vecx)$, and check via a thresholding operation,  
whether there exists a variables(s) in $\univsupp$. If $\numcen$ is large enough ($\geq \critintmeas_1^{-1}$), 
then it would also enable us to recover $\univsupp$ completely. A downside of this approach is that we would require additional, stronger conditions 
on the step size parameter $\gradstep$ to guarantee identification of $\univsupp$. Since the estimation procedure for $\univsupp$ in 
Algorithm \ref{algo:gen_overlap} comes \hemant{at the same order-wise} sampling cost, therefore we choose to query $f$ again, in order to identify $\univsupp$.
\end{remark}

\begin{remark} \label{rem:gen_over_s1_step_impr}
We also note that the condition on $\gradstepp$ is less strict than in \cite{Tyagi14_nips} for identifying 
$\univsupp$. This is because in \cite{Tyagi14_nips}, the gradient is estimated via a forward difference procedure, 
while we perform a central difference procedure in \eqref{eq:taylor_exp_f}.
\end{remark}
\subsection{Analysis for noisy setting} \label{subsec:noise_overlap_set_est}
We now consider the case where at each query $\vecx$, 
we observe $f(\vecx) + \exnoisep$, with $\exnoisep \in \matR$ denoting external noise. In order to estimate $\grad f(\vecx)$, 
we obtain the samples : $f(\vecx + \gradstep \vecv_j) + \exnoisep_{j,1}$ and $f(\vecx - \gradstep \vecv_j) + \exnoisep_{j,2}$; 
$j = 1,\dots,\numdirec$. 
This changes \eqref{eq:cs_form} to the linear system $\vecy = \matV\grad f(\vecx) + \taynoisvec + \exnoisevec$, where 
$\exnoise_{j} = (\exnoisep_{j,1} - \exnoisep_{j,2})/(2\gradstep)$. 

\paragraph{Arbitrary bounded noise.} In this scenario, we assume the external noise to be arbitrary and bounded, meaning 
that $\abs{\exnoisep} < \exnoisemag$, for some finite $\exnoisemag \geq 0$. Theorem \ref{thm:gen_overlap_arbnois} 
shows that Algorithm \ref{algo:gen_overlap} recovers $\univsupp,\bivsupp$ with 
appropriate choice of sampling parameters, provided $\exnoisemag$ is not too large.
\begin{theorem} \label{thm:gen_overlap_arbnois}
Assuming the notation in Theorem \ref{thm:gen_overlap}, let $a, b, \numcen, \numdirec, \numdirecp, \twohashfam$ be as defined in 
Theorem \ref{thm:gen_overlap}. Say $\exnoisemag < \exnoisemag_1 = \frac{\idenconst_2^{3}}{192\sqrt{3} C_1 C_2^3 \sqrt{a^3 b \numdirecp \numdirec}}$. 
Then for $\theta_1 = \cos^{-1}(-\exnoisemag / \exnoisemag_1)$, let $\gradstep, \hessstep$ satisfy: 
\begin{align}
\gradstep &\in \left(\sqrt{\frac{\idenconst_2^2}{12 a b C_2^2}}\cos(\theta_1/3 - 2\pi/3) , \sqrt{\frac{\idenconst_2^2}{12 a b C_2^2}}\cos(\theta_1/3)\right), \label{eq:gradstep_arb_nois}\\ 
\hessstep &\in \left(\frac{\idenconst_2}{4aC_2} - \sqrt{\left(\frac{\idenconst_2}{4aC_2}\right)^2 - \left(\frac{b\gradstep^2 + 2C_1\sqrt{\numdirec\numdirecp}\exnoisemag}{a}\right)}, 
\frac{\idenconst_2}{4aC_2} + \sqrt{\left(\frac{\idenconst_2}{4aC_2}\right)^2 - \left(\frac{b\gradstep^2 + 2C_1\sqrt{\numdirec\numdirecp}\exnoisemag}{a}\right)} \right). \label{eq:hessstep_arb_nois}
\end{align}
We then have in Algorithm \ref{algo:gen_overlap} for the choice
\begin{equation} \label{eq:thresh_s2_arbnois}
\hesssamperr = C_2 \left(\frac{\hessstep(4\maxdegree + 1)\smconst_3}{2\sqrt{\numdirecp}} + 
\frac{C_1\sqrt{\numdirecp} \gradstep^2((4\maxdegree+1)\totsparsity)\smconst_3}{3\hessstep\numdirec} + 
\frac{2C_1\exnoisemag\sqrt{\numdirecp\numdirec}}{\gradstep\hessstep}\right), 
\end{equation}
that $\est{\bivsupp} = \bivsupp$ with probability at least 
$1 - e^{-c_4^{\prime} \numdirec} - e^{-\sqrt{\numdirec \dimn}} - e^{-c_5^{\prime} \numdirecp} - e^{-\sqrt{\numdirecp \dimn}} - \dimn^{-2C}$. 
Given that $\est{\bivsupp} = \bivsupp$, let $\numcenpair, \numdirecpp$ be chosen as in Theorem \ref{thm:gen_overlap}.
Let $a_1 = \frac{(\totsparsity-\abs{\est{\bivsuppvar}}) \smconst_3}{6\numdirecpp}$, $b_1 = \sqrt{\numdirecpp}$ 
and assume $\exnoisemag < \exnoisemag_2 = \frac{\idenconst_1^{3/2}}{3\sqrt{6 a_1 C_3^3 b_1^2}}$. For $\theta_2 = \cos^{-1}(-\exnoisemag / \exnoisemag_2)$, 
let $\gradstepp \in (2\sqrt{\idenconst_1/(6 a_1 C_3)} \cos(\theta_2/3 - 2\pi/3), 2\sqrt{\idenconst_1/(6 a_1 C_3)} \cos(\theta_2/3))$. 
We then have in Algorithm \ref{algo:gen_overlap} for the choice 
$\derivsamperrpp = C_3\left(\frac{(\totsparsity-\abs{\est{\bivsuppvar}}) {\gradstepp}^2 \smconst_3}{6\numdirecpp} + \frac{b_1\exnoisemag}{\gradstep}\right)$
that $\est{\univsupp} = \univsupp$ with probability at least $1 - e^{-c_6^{\prime} \numdirecpp} - e^{-\sqrt{\numdirecpp \abs{\calP}}}$.
\end{theorem}
We see that in contrast to Theorem \ref{thm:gen_overlap}, the step sizes: $\gradstep, \gradstepp$ cannot be chosen too small now, on account of external noise. 
Also note that the parameters $\pi/2 \leq \theta_1,\theta_2 \leq \pi$ arising due to $\exnoisemag$, affect the size of the intervals from which $\gradstep, \gradstepp$ 
can be chosen respectively. One can verify that plugging $\exnoisemag = 0$ in Theorem \ref{thm:gen_overlap_arbnois} (implying $\theta_1,\theta_2 = \pi/2$), 
gives us the sampling conditions of Theorem \ref{thm:gen_overlap}.
\paragraph{Stochastic noise.} We now consider i.i.d Gaussian noise, so that $\exnoisep \sim \calN(0,\sigma^2)$ 
for variance $\sigma^2 < \infty$. As in Section \ref{subsec:noise_nooverlap}, we 
resample each point query a sufficient number of times and average, in order to reduce $\sigma$.  
Doing this $N_1$ times in Steps \ref{algover:s2_query_1},\ref{algover:s2_query_2}, and 
$N_2$ times in Step \ref{algover:s1_grad}, for $N_1, N_2$ large enough, we can recover $\univsupp,\bivsupp$ 
as shown formally in the following theorem.
\begin{theorem} \label{thm:gen_overlap_gaussnois}
Assuming the notation in Theorem \ref{thm:gen_overlap}, let $a, b, \numcen, \numdirec, \numdirecp, \twohashfam$ be as defined in 
Theorem \ref{thm:gen_overlap}. For any $\exnoisemag < \exnoisemag_1 = \frac{\idenconst_2^{3}}{192\sqrt{3} C_1 C_2^3 \sqrt{a^3 b \numdirecp \numdirec}}$, $0 < p_1 < 1$ 
and $\theta_1 = \cos^{-1}(-\exnoisemag / \exnoisemag_1)$, say we resample each query in Steps \ref{algover:s2_query_1}-\ref{algover:s2_query_2} of 
Algorithm \ref{algo:gen_overlap}, $N_1 > {\hemantt \frac{\sigma^2}{\exnoisemag^2} \log (\frac{2}{p_1}\numdirec(\numdirecp+1)(2\numcen+1)^2\abs{\twohashfam})}$ times, 
and average the values. Let $\gradstep, \hessstep, \hesssamperr$ be chosen to satisfy \eqref{eq:gradstep_arb_nois}, \eqref{eq:hessstep_arb_nois} and 
\eqref{eq:thresh_s2_arbnois} respectively. We then have in Algorithm \ref{algo:gen_overlap}, that $\est{\bivsupp} = \bivsupp$ with probability 
$1 -p_1 - e^{-c_4^{\prime} \numdirec} - e^{-\sqrt{\numdirec \dimn}} - e^{-c_5^{\prime} \numdirecp} - e^{-\sqrt{\numdirecp \dimn}} - \dimn^{-2C}$. 

Given that $\est{\bivsupp} = \bivsupp$, let $\numcenpair, \numdirecpp, a_1, b_1$ be as stated in Theorem \ref{thm:gen_overlap_arbnois}.
For any $\exnoisemagp < \exnoisemag_2 = \frac{\idenconst_1^{3/2}}{\sqrt{6 a_1 C_3^3 b_1^2}}$, $0 < p_2 < 1$, and $\theta_2 = \cos^{-1}(-\exnoisemagp / \exnoisemag_2)$, 
say we resample each query in Step \ref{algover:s1_grad} of Algorithm \ref{algo:gen_overlap}, 
$N_2 > \frac{\sigma^2}{{\exnoisemagp}^2} {\hemantt \log(\frac{2 (2\numcenpair+1)\numdirecpp}{p_2})}$ times. 
Furthermore, let $\gradstepp, \derivsamperrpp$ be chosen as stated in Theorem \ref{thm:gen_overlap_arbnois}. We then have in Algorithm \ref{algo:gen_overlap} 
that $\est{\univsupp} = \univsupp$ with probability at least $1 - p_2 - e^{-c_6^{\prime} \numdirecpp} - e^{-\sqrt{\numdirecpp \abs{\calP}}}$. 
\end{theorem}
\paragraph{Query complexity.} Let us analyze the query complexity when the noise is i.i.d Gaussian. 
For estimating $\bivsupp$, we have $\exnoisemag = O(\maxdegree^{-2} \totsparsity^{-1/2})$. 
Furthermore: $(2\numcen+1)^2 = \critintmeas_2^{-2}$, $\abs{\twohashfam} = O(\log d)$, $\numdirec = O(\totsparsity \log \dimn)$ and 
$\numdirecp = O(\maxdegree \log \dimn)$.  Choosing $p_1 = \dimn^{-\delta}$ for any constant $\delta > 0$ gives us 
$$N_1 = \hemant{O(\maxdegree^4 \totsparsity \log((\dimn^{\delta})\totsparsity \maxdegree (\log \dimn)^3 ))} 
= O(\maxdegree^4 \totsparsity \log \dimn)$$. This means that our total sample complexity for estimating $\bivsupp$ is:  
$$\hemant{O(N_1 \totsparsity\maxdegree(\log \dimn)^2 \abs{\baseset})} = O(\maxdegree^5 \totsparsity^2 (\log \dimn)^4 \critintmeas_2^{-2}).$$ This ensures 
$\est{\bivsupp} = \bivsupp$ with high probability.
Next, for estimating $\univsupp$, we have $\exnoisemagp = O((\totsparsity - \abs{\bivsuppvar})^{-1/2})$. Choosing $p_2 = ((\dimn - \abs{\bivsuppvar})^{-\delta})$ 
for any constant $\delta > 0$, we get $N_2 = O((\totsparsity - \abs{\bivsuppvar}) \log (\dimn - \abs{\bivsuppvar}))$. This means the total sample 
complexity for estimating $\univsupp$ is 
$O(N_2 \critintmeas_1^{-1} (\totsparsity-\abs{\est{\bivsuppvar}}) \log (\dimn - \abs{\est{\bivsuppvar}})) = O(\critintmeas_1^{-1} (\totsparsity-\abs{\est{\bivsuppvar}})^2 (\log (\dimn - \abs{\est{\bivsuppvar}}))^2)$. 
Putting it together, we have that in case of i.i.d Gaussian noise, the sampling complexity of Algorithm \ref{algo:gen_overlap} for estimating $\univsupp, \bivsupp$ is 
$O(\maxdegree^5 \totsparsity^2 (\log \dimn)^4)$. 

\begin{remark} \label{rem:gen_over_s1_gnoise_impr}
We saw above that $O(\totsparsity^2 (\log \dimn)^2)$ samples are sufficient for estimating $\univsupp$ 
in presence of i.i.d Gaussian noise. This improves the corresponding bound in \cite{Tyagi14_nips} 
by a $O(\totsparsity)$ factor, and is due to the less strict condition on $\gradstepp$ (cf., Remark \ref{rem:gen_over_s1_step_impr}).
\end{remark}
\section{Alternate sampling scheme for the general overlap case} \label{sec:algo_gen_overlap_alt}
We now derive an alternate algorithm for estimating the sets $\univsupp, \bivsupp$, for the general overlap case. 
This algorithm differs from Algorithm \ref{algo:gen_overlap} with respect to the scheme for estimating $\bivsupp$ -- 
the procedure for estimating $\univsupp$ is the same as Algorithm \ref{algo:gen_overlap}. In order to estimate $\bivsupp$, 
we now make use of recent results from CS, for recovering sparse symmetric matrices from few \emph{linear measurements}. 
More precisely, we leverage these results for estimating the sparse Hessian $\hess f(\vecx)$ at any fixed $\vecx \in \matR^d$. 
This is in stark contrast to the approaches we proposed so far, wherein, each row of the Hessian $\hess f(\vecx)$ was approximated separately. 
As we will show, this results in slightly improved sampling bounds for estimating $\bivsupp$ in the noiseless setting 
as opposed to those stated in Theorem \ref{thm:gen_overlap}. 

\subsection{Analysis for noiseless setting} \label{subsec:noiseless_overlap_set_est_alt}
We begin with the setting of noiseless point queries, and show how the problem of estimating $\hess f(\vecx)$ at any 
$\vecx \in \matR^d$ can be formulated as one of recovering an unknown sparse, symmetric matrix from linear measurements. 
To this end, first note that for $\vecx, \vecv \in \matR^{\dimn}$, step size
$\gradstep > 0$, and $\zeta = \vecx + \theta \vecv$, $\zeta^{\prime} = \vecx - \theta^{\prime} \vecv$; $\theta,\theta^{\prime} \in (0,2\gradstep)$, one 
obtains via Taylor expansion of the $C^3$ smooth $f,$ the following identity: 

\begin{align} \label{eq:quad_meas_hess}
\frac{f(\vecx + 2\gradstep\vecv) + f(\vecx - 2\gradstep\vecv) - 2f(\vecx)}{4\gradstep^2} = \vecv^T\hess f(\vecx)\vecv  
+ \underbrace{\frac{\thirdtayrem_3(\zeta) + \thirdtayrem_3(\zeta^{\prime})}{4\gradstep^2}}_{O(\gradstep)}.
\end{align}

Here $\thirdtayrem_3(\zeta), \thirdtayrem_3(\zeta^{\prime}) = O(\gradstep^3)$ denote the third order Taylor terms. Importantly, 
\eqref{eq:quad_meas_hess} corresponds to a ``noisy'' \emph{linear measurement} of $\hess f(\vecx)$ \textit{i.e.}, 
$\vecv^T\grad f(\vecx)\vecv = \dotprod{\vecv\vecv^T}{\hess f(\vecx)}$, via the measurement matrix $\vecv\vecv^T$. 
The noise arises on account of the Taylor remainder terms. We now present a recent result for recovering sparse symmetric matrices 
\cite{Chen15}, that we leverage for estimating $\hess f(\vecx)$.

\paragraph{Recovering sparse symmetric matrices via $\ell_1$ minimization.} 
{\hemantt Let $\vecv$ be composed of 
i.i.d sub-Gaussian entries with $v_i = a_i/\sqrt{\numdirec}$, and the $a_i$'s drawn in an i.i.d manner 
from a distribution satisfying: 
\begin{equation} \label{eq:genalt_samp_moment_conds}
\expec[a_i] = 0, \ \expec[a_i^2] = 1 \ \text{and} \ \expec[a_i^4] > 1. 
\end{equation}
For concreteness, we will consider the following set whose elements clearly meet these moment conditions:

\begin{align} \label{eq:samp_direc_alt_hess}
\calV &:= \set{\vecv_j \in \matR^{\dimn} : v_{j,q} = \left\{
\begin{array}{rl}
\pm \sqrt{\frac{3}{\numdirec}} ; & \text{w.p} \ 1/6 \ \text{each}, \\
0 ; & \text{w.p} \ 2/3
\end{array} \right\};
 \ j=1,\dots,\numdirec \ \text{and} \ q=1,\dots,{\dimn}}. 
\end{align}

Note that a symmetric Bernoulli distribution does not meet the aforementioned fourth order moment condition. 
}
Furthermore, let $\linoptmat: \matR^{d \times d} \rightarrow \matR^{\numdirec}$ denote a linear operator 
acting on square matrices, with 
\begin{align} \label{eq:lin_opt_altgen}
\linoptmat(\matH) := [\dotprod{\vecv_1\vecv_1^T}{\matH} \cdots \dotprod{\vecv_{\numdirec}\vecv_{\numdirec}^T}{\matH}]^T; \quad \matH \in \matR^{d \times d}.
\end{align}
For an unknown symmetric matrix $\matH_0 \in \matR^{d \times d}$, say we have at hand $\numdirec$ linear measurements 
\begin{equation} \label{eq:sparse_symm_lin_meas}
\vecy = \linoptmat(\matH_0) +  \taynoisvec; \quad \vecy, \taynoisvec \in \matR^{\numdirec}; \norm{\taynoisvec}_1 \leq \hessestnoisebd. 
\end{equation}
Then as shown in \cite[Section C]{Chen15}, 
we can recover an estimate $\est{\matH_0}$ to $\matH_0$ via $\ell_1$ minimization, by solving:
\begin{equation} \label{eq:l1_min_sparse_symm}
\est{\matH}_0 = \argmin{\matH} \norm{\matH}_1 \quad \text{s.t} \quad \matH^T = \matH, \quad \norm{\vecy - \linoptmat(\matH)}_1 \leq \hessestnoisebd.
\end{equation} 
\begin{remark}
\eqref{eq:l1_min_sparse_symm} was proposed in \cite[Section C]{Chen15} for recovering sparse covariance matrices (which are positive 
semidefinite (PSD)) with the symmetry constraint replaced by a PSD constraint. However as noted in the discussion in \cite[Section E]{Chen15}, 
one can replace the PSD constraint by a symmetry constraint, in order to recover more general symmetric matrices (which are not necessarily PSD).  
\end{remark}
\begin{remark}
Note that \eqref{eq:l1_min_sparse_symm} can be reformulated as a linear program in $O(d^2)$ variables, and hence can be solved efficiently up to 
arbitrary accuracy (using for instance, interior point methods (cf., \cite{Nesterov94})). 
\end{remark}
The estimation property of \eqref{eq:l1_min_sparse_symm} is captured in the following Theorem.
%
\begin{theorem}{\cite[Theorem 3]{Chen15}} \label{thm:chen_sparse_symm_rec}
Consider the sampling model in \eqref{eq:sparse_symm_lin_meas} \hemant{with $\vecv_i$'s satisfying \eqref{eq:genalt_samp_moment_conds}}, and let 
$(\matH_0)_{\Omega}$ denote the best $\sparsparam$ term approximation of $\matH_0$. 
Then there exist constants $c_1, c_1^{\prime}, c_2, C_1, C_2 > 0$ such that 
with probability exceeding $1 - c_1e^{-c_2 \numdirec}$, the solution $\est{\matH}_0$ to \eqref{eq:l1_min_sparse_symm} satisfies 
\begin{equation}
\norm{\est{\matH}_0 - \hemant{\matH_0}}_F \leq C_2 \frac{\norm{\matH_0 - (\matH_0)_{\Omega}}_1}{\sqrt{\sparsparam}} + C_1 \hessestnoisebd, 
\end{equation}
simultaneously for all (symmetric) $\matH_0 \in \matR^{d \times d}$, provided $\numdirec > c_1^{\prime} \sparsparam \log(d^2/\sparsparam)$.  
\end{theorem}
The proof of Theorem \ref{thm:chen_sparse_symm_rec} relies on the $\ell_2 / \ell_1$ Restricted Isometry Property (RIP) 
for sparse symmetric matrices, introduced by Chen et al. \cite{Chen15}:
\begin{definition}{\cite{Chen15}} \label{def:rip_l2_l1_chen}
For the set of symmetric $K$ sparse matrices, the operator $\linoptmatB$ is said to satisfy 
the $\ell_2 / \ell_1$ Restricted Isometry Property (RIP) with constants $\gamma_1,\gamma_2 > 0$, 
if for all such matrices $\matX$:
\begin{equation*}
 (1-\gamma_1)\norm{\matX}_F \leq \norm{\linoptmatB(\matX)}_1 \leq (1+\gamma_2)\norm{\matX}_F.
\end{equation*}
\end{definition}
While the operator $\linoptmat$ defined in \eqref{eq:lin_opt_altgen} does not satisfy $\ell_2 / \ell_1$ RIP (since 
each $\vecv_i\vecv_i^T$ has non-zero mean), one could consider instead a set of debiased measurement matrices 
$\matB_i := \vecv_{2i-1}\vecv_{2i-1}^T - \vecv_{2i}\vecv_{2i}^T$, with $\linoptmatB_i(\matX) := \dotprod{\matB_i}{\matX}$ for $i=1,\dots,m$. 
Chen et al. \cite[Corollary 2]{Chen15} then show that the linear map $\linoptmatB: \matR^{d \times d} \rightarrow \matR^m$ satisfies $\ell_2 / \ell_1$ RIP, \hemant{for $\vecv_i$'s satisfying \eqref{eq:genalt_samp_moment_conds}}, 
provided $m > K \log (d^2/K)$. 
\begin{remark} \label{rem:alt_gen_scheme_IHT}
Observe that the $\ell_1$ norm constraint in \eqref{eq:l1_min_sparse_symm} arises due to the $\ell_2 / \ell_1$ RIP in 
Definition \ref{def:rip_l2_l1_chen}. It is unclear whether the linear map $\linoptmatB$ also satisfies the conventional 
$\ell_2 /\ell_2$ RIP\footnote{We are not aware of a formal proof of this fact in the literature.}. However assuming it were 
do so, the $\ell_1$ norm constraint in \eqref{eq:l1_min_sparse_symm} could then be replaced by $\norm{\vecy - \linoptmat(\matH)}_2 \leq \hessestnoisebd$.
In particular, it might then be possible to use faster non-convex IHT based methods (cf., Remark \ref{rem:IHT}). 
%
\end{remark}
%
%
\paragraph{Estimating $\bivsupp, \univsupp$.} Given the linear program defined in \eqref{eq:l1_min_sparse_symm}, we can estimate 
$\hess f(\vecx)$ in a straightforward manner, at any fixed $\vecx \in \matR^d$. Indeed, for some suitable step size 
$\gradstep > 0$, we first collect the samples: 
$f(\vecx), \set{f(\vecx - 2\gradstep\vecv_j)}_{j=1}^{\numdirec}, \set{f(\vecx + 2\gradstep\vecv_j)}_{j=1}^{\numdirec}$, \hemant{with $\vecv_j \in \calV$}. 
Then, we form the linear system $\vecy = \linoptmat(\hess f(\vecx)) + \taynoisvec$, where
\begin{equation} \label{eq:lin_sys_alt_hess_samp}
y_j = \frac{f(\vecx + 2\gradstep\vecv_j) + f(\vecx - 2\gradstep\vecv_j) - 2f(\vecx)}{4\gradstep^2}, \quad 
\taynoissca_j = \frac{\thirdtayrem_3(\zeta_j) + \thirdtayrem_3(\zeta_j^{\prime})}{4\gradstep^2}; j=1,\dots,\numdirec.
\end{equation}
Since $\hess f(\vecx)$ is at most $\totsparsity (\maxdegree+1)$ sparse, therefore we obtain an estimate $\est{\hess} f(\vecx)$ 
to $\hess f(\vecx)$ with $2\numdirec + 1$ queries of $f$ with 
$\numdirec > c_1^{\prime}\totsparsity \maxdegree \log(\frac{d^2}{\totsparsity\maxdegree})$. 
Thereafter, we proceed as in Section \ref{sec:algo_gen_overlap}, \textit{i.e.}, we estimate $\hess f$ at each 
$\vecx \in \baseset = \cup_{\hashfn \in \twohashfam} \baseset(\hashfn)$, with $\baseset(\hashfn)$ as defined in 
\eqref{eq:baseset_hash}. 
\begin{remark}
Note that $\hess f(\vecx)$ actually has at most $\totsparsity + 2\abs{\bivsupp}$ non-zero entries. 
Therefore, if we had assumed $\abs{\bivsupp}$ to be known as part of our 
problem setup (in Section \ref{sec:problem}), then the choice
$\numdirec > c_1^{\prime} (\totsparsity + 2\abs{\bivsupp}) \log(\frac{d^2}{\totsparsity + 2\abs{\bivsupp}})$ 
would suffice for estimating $\hess f(\vecx)$. 
We can bound $2\abs{\bivsupp} \leq \totsparsity\maxdegree$ -- this is also tight in the worst case -- 
however in certain settings this would be pessimistic\footnote{For example when $O(1)$ variables have degree $\maxdegree$, 
and the remaining variables have degree $1$ leading to $\abs{\bivsupp} = O(\totsparsity + \maxdegree)$.}
\end{remark}
Once $\bivsupp$ is identified, we can simply reuse the procedure in Algorithm \ref{algo:gen_overlap}, 
for estimating $\univsupp$. The above discussion for identifying $\univsupp,\bivsupp$ is formally outlined 
in Algorithm \ref{algo:gen_overlap_alt}. 
%
\begin{algorithm*}[!ht]
\caption{Algorithm for estimating $\univsupp,\bivsupp$} \label{algo:gen_overlap_alt} 
\begin{algorithmic}[1] 
\State \textbf{Input:} $\numdirec, \numcen \in \matZ^{+}$; $\gradstep > 0$; $\hessestnoisebd, \derivsamperr > 0$.
\State \textbf{Initialization:} $\est{\univsupp}, \est{\bivsupp} = \emptyset$.
\State \textbf{Output:} Estimates $\est{\bivsupp}$, $ \est{\univsupp}$. \\
\hrulefill \\
\textsc{// Estimation of } $\bivsupp$ \hfill
\State Construct $(\dimn,2)$-hash family $\twohashfam$ and sets $\calV$. \label{algover:s2_step_1_alt}
\For{$\hashfn \in \twohashfam$}
 	\State Construct the set $\baseset(\hashfn)$. 
	\For {$i = 1,\dots,(2\numcen+1)^2$ and $\vecx_i \in \baseset(\hashfn)$} \label{algover:s2_step_2_alt}
	
			\State $(\vecy_i)_j = \frac{f(\vecx_i + 2\gradstep \vecv_j) + f(\vecx_i - 2\gradstep \vecv_j) - 2f(\vecx_i)}{4\gradstep^2}$; 
			$j=1,\dots,\numdirec$; $\vecv_j \in \calV$. \label{algover:s2_query_1_alt}
		
			\State $\est{\hess} f(\vecx_i) := \argmin{\matH} \norm{\matH}_1$ s.t. $\matH^T = \matH$, $\norm{\vecy_i - \linoptmat(\matH)}_1 \leq \hessestnoisebd$. \label{algover:s2_query_hess_est_alt}
			
			\State $\est{\bivsupp} = \est{\bivsupp} \cup \set{(q,q^{\prime}) \in {[d] \choose 2} : \abs{(\est{\hess} f(\vecx_i))_{q,q^{\prime}}} > \derivsamperr}$. 
		
	\EndFor 
\EndFor \\
\hrulefill \\
\textsc{// Estimation of } $\univsupp$ \hfill
\State \hemant{Estimate $\univsupp$ as in Algorithm \ref{algo:gen_overlap}.}
\end{algorithmic}
\end{algorithm*}
The following Theorem provides sufficient conditions on the sampling parameters in Algorithm \ref{algo:gen_overlap_alt}, 
that guarantee $\est{\bivsupp} = \bivsupp$, $\est{\univsupp} = \univsupp$ with high probability.
\begin{theorem} \label{thm:gen_overlap_alt}
Let $\twohashfam$ be of size $\abs{\twohashfam} \leq 2(C + 1)e^2 \log \dimn$ for some constant $C > 1$. 
Then $\exists$ constants $c_1, c_1^{\prime}, c_2, C_1 > 0$, 
such that the following is true. Let $\numcen, \numdirec, \gradstep$ satisfy
\begin{equation} 
\numcen \geq \critintmeas_2^{-1}, \ \numdirec > c_1^{\prime}\totsparsity \maxdegree \log\left(\frac{d^2}{\totsparsity\maxdegree}\right), \ 
\gradstep < \hemant{\frac{\sqrt{\numdirec}\idenconst_2}{2\sqrt{6} C_1 \smconst_3(4\maxdegree+1)\totsparsity}}.
\end{equation}
We then have for the choices 
%
$\hessestnoisebd = \hemant{\frac{2\sqrt{3}\gradstep\smconst_3(4\maxdegree+1)\totsparsity}{\sqrt{\numdirec}}}, \ \derivsamperr = C_1\hessestnoisebd$  
%
that $\est{\bivsupp} = \bivsupp$ with probability at least $1 - c_1 e^{-c_2 \numdirec} - \dimn^{-2C}$. 
Given that $\est{\bivsupp} = \bivsupp$, the sampling conditions for estimating $\est{\univsupp}$ are identical to 
Theorem \ref{thm:gen_overlap}. 
\end{theorem}
\paragraph{Query complexity.} Estimating $\hess f(\vecx)$ at 
some fixed $\vecx$ requires $2\numdirec + 1 = O(\totsparsity\maxdegree\log(\frac{d^2}{\totsparsity\maxdegree}))$ queries. 
Since $\hess f$ is estimated at all points in $\baseset$ in the worst case, this consequently implies a total query complexity of 
$O(\totsparsity\maxdegree\log(\frac{d^2}{\totsparsity\maxdegree}) \abs{\baseset}) = O(\critintmeas_2^{-2}\totsparsity\maxdegree(\log \dimn)^2)$, for estimating $\bivsupp$. 
As seen in Theorem \ref{thm:gen_overlap}, we make an additional 
$O(\critintmeas_1^{-1} (\totsparsity-\abs{\est{\bivsuppvar}}) \log (\dimn - \abs{\est{\bivsuppvar}}))$ queries of $f$, 
in order to estimate $\univsupp$. Therefore, the overall query complexity for estimating $\univsupp,\bivsupp$ is 
$O(\critintmeas_2^{-2}\totsparsity\maxdegree(\log \dimn)^2)$. Observe that this is better by a $\log d$ factor as compared to the sampling 
bound for Algorithm \ref{algo:gen_overlap} (in the noiseless setting). 

\paragraph{Computational complexity.} The family $\twohashfam$ can be 
constructed\footnote{Recall discussion following Definition \ref{def:thash_fam}.} 
in time polynomial in $d$. At each $\vecx \in \baseset$, we solve a linear program (Step \ref{algover:s2_query_hess_est_alt}) in $O(d^2)$ variables, which 
can be done up to arbitrary accuracy in time polynomial in $(\numdirec,d)$. Since this is done at $\abs{\baseset} = O(\critintmeas_2^{-2} \log d)$ many points, hence the 
overall computation time for estimation of $\bivsupp$ (and subsequently $\univsupp$) is polynomial in the number of queries, and in $d$.

\subsection{Analysis for noisy setting} \label{subsec:noise_overlap_set_est_alt}
We now consider the case where at each query $\vecx$, 
we observe $f(\vecx) + \exnoisep$, with $\exnoisep \in \matR$ denoting external noise. 
In order to estimate $\hess f(\vecx)$, we obtain the samples : $f(\vecx + 2\gradstep\vecv_j) + \exnoisep_{j,1}$, 
$f(\vecx - 2\gradstep \vecv_j) + \exnoisep_{j,2}$ and $f(\vecx) + \exnoisep_{3}$; $j = 1,\dots,\numdirec$. 
This changes \eqref{eq:lin_sys_alt_hess_samp} to the linear system $\vecy = \linoptmat(\hess f(\vecx)) + \taynoisvec + \exnoisevec$, where
$\exnoise_{j} = (\exnoisep_{j,1} + \exnoisep_{j,2} - 2\exnoisep_{3})/(4\gradstep^2)$. 

\paragraph{Arbitrary bounded noise.} 
Assuming the external noise to be arbitrary and bounded, meaning 
that $\abs{\exnoisep} < \exnoisemag$, Theorem \ref{thm:gen_overlap_arbnois_alt} 
shows that Algorithm \ref{algo:gen_overlap_alt} recovers $\univsupp,\bivsupp$ with appropriate choice of sampling 
parameters provided $\exnoisemag$ is not too large.
\begin{theorem} \label{thm:gen_overlap_arbnois_alt}
Assuming the notation in Theorem \ref{thm:gen_overlap_alt}, let $\numcen,\numdirec$ and $\twohashfam$ be as defined in 
Theorem \ref{thm:gen_overlap_alt}. Denoting $a = \hemant{\frac{\sqrt{6}\smconst_3(4\maxdegree+1)\totsparsity}{\sqrt{\numdirec}}}$, 
say $\exnoisemag$ satisfies $\exnoisemag < \exnoisemag_1 = \frac{\sqrt{2} \idenconst_2^{3}}{54 a^2 C_1^3 \numdirec}$ and 
$\theta_1 = \cos^{-1}\left(1-\frac{2\exnoisemag}{\exnoisemag_1}\right)$. Let 
\begin{equation}
\gradstep \in \left(-\frac{\idenconst_2}{3 a C_1}\cos\left(\frac{\theta_1}{3} + \frac{\pi}{3}\right) + \frac{\idenconst_2}{6 a C_1}, 
\frac{\idenconst_2}{3 a C_1}\cos\left(\frac{\theta_1}{3}\right) + \frac{\idenconst_2}{6 a C_1} \right).
\end{equation}
We then have in Algorithm \ref{algo:gen_overlap_alt} for the choices 
$\hessestnoisebd = \left(\hemant{\frac{2\sqrt{3}\gradstep\smconst_3(4\maxdegree+1)\totsparsity}{\sqrt{\numdirec}}} + \frac{\exnoisemag\numdirec}{\gradstep^2}\right)$, 
$\derivsamperr = C_1 \hessestnoisebd$, that $\est{\bivsupp} = \bivsupp$ with probability at least $1 - c_1 e^{-c_2 \numdirec} - \dimn^{-2C}$. 
Given that $\est{\bivsupp} = \bivsupp$, the sampling conditions for estimating $\est{\univsupp}$ are identical to 
Theorem \ref{thm:gen_overlap_arbnois}.
\end{theorem} 

\paragraph{Stochastic noise.} 
We now consider i.i.d Gaussian noise, so that $\exnoisep \sim \calN(0,\sigma^2)$ 
for variance $\sigma^2 < \infty$. As in Sections \ref{subsec:noise_nooverlap}, \ref{subsec:noise_overlap_set_est}, 
we reduce $\sigma$ via resampling and averaging. Doing this $N_1$ times in Step \ref{algover:s2_query_1_alt}, 
and $N_2$ times during estimation of $\univsupp$, for $N_1, N_2$ large enough, we can recover 
$\univsupp,\bivsupp$ as shown formally in the following Theorem.
\begin{theorem} \label{thm:gen_overlap_gaussnois_alt}
Assuming the notation in Theorem \ref{thm:gen_overlap_alt}, let $\numcen,\numdirec$ and $\twohashfam$ be as defined in 
Theorem \ref{thm:gen_overlap_alt}. For any $\exnoisemag < \exnoisemag_1 = \frac{\sqrt{2} \idenconst_2^{3}}{54 a^2 C_1^3 \numdirec}$, $0 < p_1 < 1$, 
say we resample each query in Step \ref{algover:s2_query_1_alt} of Algorithm \ref{algo:gen_overlap_alt}, 
\hemant{$N_1 > \frac{3\sigma^2}{4\exnoisemag^2} \log (\frac{2}{p_1}\numdirec(2\numcen+1)^2\abs{\twohashfam})$} 
times, and average the values. We then have in Algorithm \ref{algo:gen_overlap_alt} for the choices of $\hessestnoisebd$, $\derivsamperr$, $\gradstep$ 
as in Theorem \ref{thm:gen_overlap_arbnois_alt}, that $\est{\bivsupp} = \bivsupp$ with probability at least $1 - c_1 e^{-c_2 \numdirec} - \dimn^{-2C} - p_1$.
Given that $\est{\bivsupp} = \bivsupp$, the sampling conditions for estimating $\est{\univsupp}$ are identical to 
Theorem \ref{thm:gen_overlap_gaussnois}. 
\end{theorem}
%
%
\paragraph{Query complexity.} We now analyze the query complexity for Algorithm \ref{algo:gen_overlap_alt}, when the noise is i.i.d Gaussian. 
For estimating $\bivsupp$, we have $\exnoisemag = O(\maxdegree^{-2} \totsparsity^{-2})$. 
Furthermore: $(2\numcen+1)^2 = \critintmeas_2^{-2}$, $\abs{\twohashfam} = O(\log d)$, $\numdirec = O(\totsparsity\maxdegree \log \dimn)$.  
Choosing $p_1 = \dimn^{-\delta}$ for any constant $\delta > 0$ gives us 
$$N_1 = O(\maxdegree^4 \totsparsity^4 \log(\dimn^{\delta} (\totsparsity\maxdegree \log d) \critintmeas_2^{-2} \log d)) 
= O(\maxdegree^4 \totsparsity^4 \log \dimn).$$ 
This means that our total sample complexity for ensuring $\est{\bivsupp} = \bivsupp$ with high probability is:  
$$\hemant{O(N_1 \totsparsity\maxdegree \log \dimn \abs{\baseset})} = O(\maxdegree^5 \totsparsity^5 (\log \dimn)^3 \critintmeas_2^{-2}).$$

Lastly, by noting the sample complexity for estimating $\univsupp$ from Theorem \ref{thm:gen_overlap_gaussnois}, 
we conclude that the overall sample complexity for ensuring $\est{\univsupp} = \univsupp$ and $\est{\bivsupp} = \bivsupp$, 
in the presence of i.i.d Gaussian noise, is $O(\maxdegree^5 \totsparsity^5 (\log \dimn)^3 \critintmeas_2^{-2})$. 
Observe that this bound has a relatively worse scaling w.r.t $\maxdegree$ compared to that for Algorithm \ref{algo:gen_overlap} (derived after Theorem \ref{thm:gen_overlap_gaussnois}); 
specifically, by a factor of \hemant{$\totsparsity^3$}. On the other hand, the scaling w.r.t $d$ is better by a logarithmic factor, compared to that for Algorithm \ref{algo:gen_overlap}. 
\section{Learning individual components of model} \label{sec:est_comp}
Recall from \eqref{eq:unique_mod_rep} the unique representation of the model:
\begin{equation} 
f(x_1,\dots,x_d) = c + \sum_{p \in \univsupp}\phi_{p} (x_p) + \sum_{\lpair \in \bivsupp} \phi_{\lpair} \xlpair + \sum_{q \in \bivsuppvar: \degree(q) > 1} \phi_{q} (x_q), 
\end{equation}
where $\univsupp \cap \bivsuppvar = \emptyset$. Having estimated the sets $\univsupp$ and $\bivsupp$, 
we now show how the individual univariate and bivariate functions in the model can be estimated. 
We will see this for the settings of noiseless, as well as noisy (arbitrary, bounded noise and stochastic noise) point queries.

\subsection{Noiseless queries}
In this scenario, we obtain the exact value $f(\vecx)$ at each query $\vecx \in \matR^{\dimn}$.
Let us first see how each $\phi_p$; $p \in \univsupp$ can be estimated.
For some $-1 = t_1 < t_2 < \dots < t_n = -1 $, consider the set
\begin{equation} \label{eq:univ_est_set}
\baseset_p := \left\{\vecx_i \in \matR^{\dimn}: (\vecx_i)_j =  \left\{
\begin{array}{rl}
t_i ; & j = p, \\
0 ; & j \neq p 
\end{array} \right\} ; 1 \leq i \leq n; 1 \leq j \leq \dimn\right\}; \quad p \in \univsupp.
\end{equation}
We obtain the samples $\set{f(\vecx_i)}_{i=1}^{n}$; $\vecx_i \in \baseset_p$. Here $f(\vecx_i) = \phi_p(t_i) + C$ with $C$ being a constant that 
depends on the other components in the model. Given the samples, one can then employ spline based ``quasi interpolant operators'' \cite{deBoor78}, 
to obtain an estimate $\phitil_p :[-1,1] \rightarrow \matR$, to $\phi_p + C$. 
Construction of such operators can be found for instance in \cite{deBoor78} (see also \cite{Gyorfi2002}). 
One can suitably choose the $t_i$'s and construct quasi interpolants that approximate any $C^m$ 
smooth univariate function with optimal $\Linfnorm[-1,1]$ error rate $O(n^{-m})$ \cite{deBoor78, Gyorfi2002}. 
Having obtained $\phitil_p$, we then define  
\begin{equation} \label{eq:univ_est}
\est{\phi}_p := \phitil_p -\expec_p[\phitil_p]; \quad p \in \univsupp, 
\end{equation}
to be the estimate of $\phi_p$. The bivariate components corresponding to each $\lpair \in \bivsupp$ 
can be estimated in a similar manner as above. To this end, for some 
strictly increasing sequences: $(-1 = \tp_1,\tp_2,\dots,\tp_{n_1} = 1)$, $(-1 = t_1,t_2,\dots,t_{n_1} = 1)$, 
consider the set 
\begin{equation} \label{eq:biv_est_set}
\baseset_{\lpair} := \left\{\vecx_{i,j} \in \matR^{\dimn}: (\vecx_{i,j})_q =  \left\{
\begin{array}{rl}
\tp_i ; & q = l, \\
t_j ; & q = l^{\prime}, \\
0 ; & q \neq l,l^{\prime} 
\end{array} \right\}  ; 1 \leq i,j \leq n_1; 1 \leq q \leq \dimn \right\}; \quad \lpair \in \bivsupp.
\end{equation}
We then obtain the samples $\set{f(\vecx_{i,j})}_{i,j=1}^{n_1}$; $\vecx_{i,j} \in \baseset_{\lpair}$ where
\begin{align} \label{eq:biv_part_fn_exp}
 f(\vecx_{i,j}) &= \phi_{\lpair}(\tp_i,t_j) + \sum_{\substack{l_1:(l,l_1) \in \bivsupp \\ l_1 \neq \lp}} \phi_{(l,l_1)} (\tp_i,0) + \sum_{\substack{l_1:(l_1,l) \in \bivsupp \\ l_1 \neq \lp}} 
\phi_{(l_1,l)} (0,\tp_i) \nonumber \\
&+ \sum_{\substack{\lp_1:(\lp,\lp_1) \in \bivsupp \\ \lp_1 \neq l}} \phi_{(\lp,\lp_1)} (t_j,0) + \sum_{\substack{\lp_1:(\lp_1,\lp) \in \bivsupp \\ \lp_1 \neq l}} \phi_{(\lp_1,\lp)} (0,t_j) + 
\phi_l(\tp_i) + \phi_{\lp}(t_j) + C, \\
&= g_{\lpair}(\tp_i,t_j) + C,
\end{align}
with $C$ being a constant. \eqref{eq:biv_part_fn_exp} is a general expression -- if for example $\degree(l) = 1$, then 
the terms $\phi_l,\phi_{(l,l_1)},\phi_{(l_1,l)}$ will be zero.
Given this, we can again obtain estimates $\phitil_{\lpair}:[-1,1]^2 \rightarrow \matR$ to $g_{\lpair} + C$, 
via spline based quasi interpolants. 
Let us denote $n = n_1^2$ to be the total number of samples of $f$. For an appropriate choice of $(\tp_i,t_j)$'s, 
one can construct bivariate quasi interpolants that approximate any $C^m$ smooth bivariate function, with optimal 
$\Linfnorm[-1,1]^2$ error rate $O(n^{-m/2})$ \cite{deBoor78, Gyorfi2002}. 
Subsequently, we define the final estimates $\est{\phi}_{\lpair}$ to $\phi_{\lpair}$ as follows.
\begin{equation} \label{eq:biv_est}
\est{\phi}_{\lpair} :=  \left\{
\begin{array}{rl}
\phitil_{\lpair} - \expec_{\lpair}[\phitil_{\lpair}] ; & \degree(l), \degree(l^{\prime}) = 1, \\
\phitil_{\lpair} - \expec_{l}[\phitil_{\lpair}] ; &  \degree(l) = 1, \degree(l^{\prime}) > 1, \\
\phitil_{\lpair} - \expec_{l^{\prime}}[\phitil_{\lpair}] ; & \degree(l) > 1, \degree(l^{\prime}) = 1, \\
\phitil_{\lpair} - \expec_{l}[\phitil_{\lpair}] - \expec_{l^{\prime}}[\phitil_{\lpair}] + \expec_{\lpair}[\phitil_{\lpair}] ; & \degree(l) > 1, \degree(l^{\prime}) > 1.
\end{array} \right. 
\end{equation}
Lastly, we require to estimate the univariate's : $\phi_l$ for each $l \in \bivsuppvar$ such that $\degree(l) > 1$.
As above, for some strictly increasing sequences: $(-1 = \tp_1,\tp_2,\dots,\tp_{n_1} = 1)$, $(-1 = t_1,t_2,\dots,t_{n_1} = 1)$, 
consider the set 
\begin{equation} \label{eq:biv_univ_est_set}
\baseset_{l} := \Biggl\{\vecx_{i,j} \in \matR^{\dimn}: (\vecx_{i,j})_q =  \left\{
\begin{array}{rl}
\tp_i ; & q = l, \\
t_j ; & q \neq l \ \& \ q \in \bivsuppvar, \\
0; & q \notin \bivsuppvar, 
\end{array} \right\} ; \\ 1 \leq i,j \leq n_1; 1 \leq q \leq \dimn \Biggr\}; \quad l \in \bivsuppvar: \degree(l) > 1.
\end{equation}
We obtain $\set{f(\vecx_{i,j})}_{i,j=1}^{n_1}$; $\vecx_{i,j} \in \baseset_{l}$ where this time 
\begin{align}
 f(\vecx_{i,j}) &= \phi_l(\tp_i) + \sum_{\degree(\lp) > 1, \lp \neq l} \phi_{\lp}(t_j) + \sum_{\lp:\lpair \in \bivsupp} \phi_{\lpair}(\tp_i,t_j) \\
&+ \sum_{\lp:\lpairi \in \bivsupp} \phi_{\lpairi}(t_j,\tp_i) + \sum_{\qpair \in \bivsupp : q,\qp \neq l} \phi_{\qpair}(t_j,t_j) + C \\
&= g_l(\tp_i,t_j) + C
\end{align}
for a constant, $C$. Denoting $n = n_1^2$ to be the total number of samples of $f$, we can again obtain an estimate 
$\phitil_l(x_l,x)$ to $g_l(x_l,x) + C$, with $\Linfnorm[-1,1]^2$ error rate $O(n^{-3/2})$.
Then with $\phitil_l$ at hand, we define the estimate $\est{\phi}_l: [-1,1] \rightarrow \matR$ as
\begin{equation} \label{eq:biv_univ_est}
\est{\phi}_l := \expec_x[\phitil_l] - \expec_{(l,x)}[\phitil_l]; \quad l \in \bivsuppvar : \degree(l) > 1. 
\end{equation}
The following proposition formally describes the error rates for the aforementioned estimates.

\begin{proposition} \label{prop:no_nois_est_comp}
For $C^3$ smooth components $\phi_p, \phi_{\lpair},\phi_{l}$,  let $\est{\phi}_p$, $\est{\phi}_{\lpair}, \est{\phi}_{l}$ 
be the respective estimates as defined in \eqref{eq:univ_est}, \eqref{eq:biv_est} and \eqref{eq:biv_univ_est} respectively.
Also, let $n$ denote the number of queries (of $f$) made per component. We then have that: 
\begin{enumerate}
\item $\norm{\est{\phi}_p - \phi_p}_{\Linfnorm[-1,1]} = O(n^{-3}); \forall p \in \univsupp$, 
\item $\norm{\est{\phi}_{\lpair} - \phi_{\lpair}}_{\Linfnorm[-1,1]^2} = O(n^{-3/2}); \forall \lpair \in \bivsupp$, and 
\item $\norm{\est{\phi}_{l} - \phi_{l}}_{\Linfnorm[-1,1]} = O(n^{-3/2}); \forall l \in \bivsuppvar: \degree(l) > 1$.
\end{enumerate}
\end{proposition}
%
%
\subsection{Noisy queries}
We now look at the case where for each query $\vecx \in \matR^d$, we obtain a noisy value $f(\vecx) + \exnoisep$.
\paragraph{Arbitrary bounded noise.} We begin with the scenario where $\exnoisep_i$ is arbitrary and bounded 
with $\abs{\exnoisep_i} < \exnoisemag; \ \forall i$. Since the noise is arbitrary in nature, therefore we simply proceed 
\emph{as in the noiseless case}, \textit{i.e.}, by approximating each component via a quasi-interpolant. As the magnitude of the noise is bounded 
by $\exnoisemag$, it results in an additional $O(\exnoisemag)$ term in the approximation error rates of Proposition \ref{prop:no_nois_est_comp}. 

To see this for the univariate case, let us denote $Q: C(\matR) \rightarrow \spacespl$ to be a quasi-interpolant operator. 
This a linear operator, with $C(\matR)$ denoting the space of continuous functions defined over $\matR$ and $\spacespl$ denoting a univariate 
spline space. Consider $u \in C^m[-1,1]$ for some positive integer $m$, and let $g:[-1,1] \rightarrow \matR$ be an arbitrary continuous function with 
$\norm{g}_{\Linfnorm[-1,1]} < \exnoisemag$. Denote $\est{u} = u + g$ to be the ``corrupted'' version of $u$, 
and let $n$ be the number of samples of $\est{u}$ used by $Q$. We then have by linearity of $Q$ that: 

\begin{equation}
\norm{Q(\est{u}) - u}_{\Linfnorm[-1,1]} = \norm{Q(u) + Q(g) - u}_{\Linfnorm[-1,1]} \leq 
\underbrace{\norm{Q(u) - u}_{\Linfnorm[-1,1]}}_{= O(n^{-m})} + \norm{Q}\underbrace{\norm{g}_{\Linfnorm[-1,1]}}_{\leq \norm{Q}\exnoisemag}, 
\end{equation}

with $\norm{Q}$ being the operator norm of $Q$. One can construct $Q$ with $\norm{Q}$ 
bounded\savefootnote{foot:quasinterp_ref}{For instance, see Theorems $14.4, 15.2$ in \cite{Gyorfi2002}}
from above by a constant depending only on $m$. The above argument can be extended easily to the multivariate case. We state this for the bivariate 
case for completeness. Denote $Q_1: C(\matR^2) \rightarrow \spacespl$ to be a quasi-interpolant operator, with $\spacespl$ denoting a bivariate spline space. 
Consider $u_1 \in C^m[-1,1]^2$ for some positive integer $m$, and let $g_1:[-1,1] \rightarrow \matR$ be an arbitrary continuous function with 
$\norm{g_1}_{\Linfnorm[-1,1]^2} < \exnoisemag$. Let $\est{u}_1 = u_1 + g_1$ and let $n$ be the number of samples of $\est{u_1}$ used by $Q_1$. 
We then have by linearity of $Q_1$ that:  

\begin{equation}
\norm{Q_1(\est{u_1}) - u_1}_{\Linfnorm[-1,1]^2} = \norm{Q_1(u_1) + Q_1(g_1) - u_1}_{\Linfnorm[-1,1]^2} \leq 
\underbrace{\norm{Q_1(u_1) - u_1}_{\Linfnorm[-1,1]^2}}_{= O(n^{-m/2})} + \norm{Q_1}\underbrace{\norm{g_1}_{\Linfnorm[-1,1]^2}}_{\leq \norm{Q_1}\exnoisemag}, 
\end{equation}

with $\norm{Q_1}$ being the operator norm of $Q_1$. As for the univariate case, one can construct $Q_1$ with $\norm{Q_1}$ 
bounded\repeatfootnote{foot:quasinterp_ref} from above by a constant depending only on $m$. 

Let us define our final estimates $\est{\phi}_p$, $\est{\phi}_{\lpair}$ and $\est{\phi}_{l}$ as in 
\eqref{eq:univ_est}, \eqref{eq:biv_est} and \eqref{eq:biv_univ_est}, respectively. 
The following proposition formally states the error bounds, for this particular noise model.

\begin{proposition}[Arbitrary bounded noise] \label{prop:arb_nois_est_comp}
For $C^3$ smooth components $\phi_p, \phi_{\lpair},\phi_{l}$,  let $\est{\phi}_p$, $\est{\phi}_{\lpair}, \est{\phi}_{l}$ 
be the respective estimates as defined in \eqref{eq:univ_est}, \eqref{eq:biv_est} and \eqref{eq:biv_univ_est} respectively.
Also, let $n$ denote the number of noisy queries (of $f$) made per component with the external noise magnitude being bounded 
by $\exnoisemag$. We then have that 

\begin{enumerate}
\item $\norm{\est{\phi}_p - \phi_p}_{\Linfnorm[-1,1]} = O(n^{-3}) + O(\exnoisemag); \forall p \in \univsupp$, 
\item $\norm{\est{\phi}_{\lpair} - \phi_{\lpair}}_{\Linfnorm[-1,1]^2} = O(n^{-3/2}) + O(\exnoisemag); \forall \lpair \in \bivsupp$, and 
\item $\norm{\est{\phi}_{l} - \phi_{l}}_{\Linfnorm[-1,1]} = O(n^{-3/2}) + O(\exnoisemag); \forall l \in \bivsuppvar: \degree(l) > 1$. 
\end{enumerate}
\end{proposition}

The proof is similar to that of Proposition \ref{prop:no_nois_est_comp} and hence skipped.

\paragraph{Stochastic noise.} We now consider the setting where $\exnoisep_i \sim \calN(0,\sigma^2)$ 
are i.i.d Gaussian random variables. Similar to the noiseless case, 
estimating the individual components again involves sampling $f$ along the 
subspaces corresponding to $\univsupp$, $\bivsupp$. Due to the presence of stochastic noise however, 
we now make use of \emph{nonparametric regression} techniques to compute the estimates. 
While there exist a number of methods that could be used for this 
purpose (cf. \cite{tsyba08}), we only discuss a specific one for clarity of exposition.

To elaborate, we again construct the sets defined in \eqref{eq:univ_est_set},\eqref{eq:biv_est_set} and\eqref{eq:biv_univ_est_set}. 
In particular, we uniformly discretize the domains 
$[-1,1]$ and $[-1,1]^2$, by choosing the respective $t_i$'s and $(\tp_i,t_j)$'s accordingly. 
This is the so called ``fixed design'' setting in nonparametric statistics. 
Upon collecting the samples $\set{f(\vecx_i) + \exnoisep_i}_{i=1}^{n}$ one can then derive estimates $\phitil_p$, $\phitil_{\lpair}, \phitil_l$, to 
$\phi_p + C$, $g_{\lpair} + C$ and $g_l + C$ respectively, by using \emph{local polynomial 
estimators} (cf. \cite{tsyba08, Fan96} and references within). 
It is known that these estimators achieve the (minimax optimal) $\Linfnorm$ error rate: 
$\Omega((n^{-1} \log n)^{\frac{m}{2m+\dimn}})$, for estimating $d$-variate, $C^m$ smooth functions over compact 
domains\footnote{See \cite{tsyba08} for $\dimn=1$, and \cite{Nemi00} for $\dimn \geq 1$}. 
Translated to our setting, we then have that the functions: $\phi_p + C$, $g_{\lpair} + C$ and $g_l + C$ are estimated at the rates: 
$O((n^{-1} \log n)^{\frac{3}{7}})$ and $O((n^{-1} \log n)^{\frac{3}{8}})$ respectively.

Denoting the above intermediate estimates by $\phitil_{p}$, $\phitil_{\lpair}$, $\phitil_l$, we define our final estimates 
$\est{\phi}_p$, $\est{\phi}_{\lpair}$ and $\est{\phi}_{l}$ as in 
\eqref{eq:univ_est}, \eqref{eq:biv_est} and \eqref{eq:biv_univ_est}, respectively. 
The following Proposition describes the error rates of these estimates. 

\begin{proposition}[i.i.d Gaussian noise] \label{prop:gauss_nois_est_comp}
For $C^3$ smooth components $\phi_p, \phi_{\lpair},\phi_{l}$,  let $\est{\phi}_p$, $\est{\phi}_{\lpair}, \est{\phi}_{l}$ 
be the respective estimates as defined in \eqref{eq:univ_est}, \eqref{eq:biv_est} and \eqref{eq:biv_univ_est} respectively.
Let $n$ denote the number of noisy queries (of $f$) made per component, with noise samples $\exnoisep_1,\exnoisep_2,\dots,\exnoisep_n$ 
being i.i.d Gaussian. Furthermore, let $\expec_{z}[\cdot]$ denote expectation w.r.t the 
joint distribution of $\exnoisep_1,\exnoisep_2,\dots,\exnoisep_n$. We then have that 

\begin{enumerate}
\item $\expec_{z}[\norm{\est{\phi}_p - \phi_p}_{\Linfnorm[-1,1]}] = O((n^{-1} \log n)^{\frac{3}{7}}); \forall p \in \univsupp$, 
\item $\expec_{z}[\norm{\est{\phi}_{\lpair} - \phi_{\lpair}}_{\Linfnorm[-1,1]^2}] = O((n^{-1} \log n)^{\frac{3}{8}}); \forall \lpair \in \bivsupp$, and 
\item $\expec_{z}[\norm{\est{\phi}_{l} - \phi_{l}}_{\Linfnorm[-1,1]}] = O((n^{-1} \log n)^{\frac{3}{8}}); \forall l \in \bivsuppvar: \degree(l) > 1$.
\end{enumerate}
\end{proposition}
%
%

\section{Simulation results} \label{sec:sims}
We now provide some simulation results for our methods on synthetic examples. The main goal of our 
experiments is to provide a proof of concept, validating some of the theoretical results that were derived 
earlier. We consider both non-overlapping (Section \ref{sec:sims_non}) and overlapping settings (Section \ref{sec:sims_over}). 
In our experiments, we use the ALPS algorithm \cite{kyrillidis2011recipes} 
as our CS solver -- an efficient first-order method.

Starting with the non-overlapping case, we present phase transition results and also show the dependence of $d$ on the number of samples, 
for recovery of $\univsupp$, $\bivsupp$. We then empirically demonstrate the dependence of the number of samples on $k$. 
In both cases, our findings support our theory for sample complexities. 
We conduct similar experiments for the overlapping case, and also additionally demonstrate empirically 
the dependence of the number of samples on the parameter $\rho_m$.

\subsection{Non-overlapping setting} \label{sec:sims_non}
We consider the following experimental setup: $\univsupp = \set{1, 2}$ and $\bivsupp = \set{(3, 4), (5, 6)}$,
which implies $\univsparsity = 2$, $\bivsparsity = 2$ and $\totsparsity = 6$.  
Moreover, we consider three different types of $f$ namely:
\begin{itemize}
\item [$(i)$] {\small $f_1(\vecx) = 2x_1 - 3x_2^2 + 4x_3x_4 - 5x_5x_6$}, 
\item [$(ii)$] {\small $f_2(\vecx) = 10 \sin(\pi \cdot x_1) + 5 e^{-2x_2} + 10 \sin(\pi \cdot x_3 x_4) + 5 e^{-2x_5 x_6}$},
\item [$(iii)$] {\small $f_3(\vecx) = \frac{10}{3} \cos(\pi \cdot x_1) + 8 x_1^2 + 5 ( x_2^4 - x_2^2 + \frac{4}{5} x_4) + \frac{10}{3} \cos(\pi \cdot x_3x_4) 
+ 8 (x_3x_4)^2 + 5 ( (x_5x_6)^4 - (x_5x_6)^2 + \frac{4}{5} x_5x_6)$}.
\end{itemize} 
For all cases, we use Algorithms \ref{algo:est_act} and \ref{algo:est_ind_sets}. For $f_1$, the problem parameters are set to
$\critintmeas_1 = 0.3$, $\critintmeas_2 = 1$, $\idenconst_1 = 2$, $\idenconst_2 = 3$, $\smconst_3 = 6$, while for 
$f_2, f_3$: $\critintmeas_1 = \critintmeas_2 = 0.3$, $\idenconst_1 = 8$, $\idenconst_2 = 4$, $\smconst_3 = 35$.
Given these constants, we obtain $\numcen = 1$, $\numcenpair = 4$ for $f_1$ and $\numcen = \numcenpair = 4$ for $f_2, f_3$. 
We use constant $\widetilde{C}$ (to be defined next)
when we set $\numdirec := \widetilde{C} \totsparsity \log\left(\dimn/\totsparsity\right)$. 
For the construction of the hash functions, we set the size to $|\twohashfam| = C^{\prime} \log d$ with $C^{\prime} = 1.7$, 
leading to $|\mathcal{H}_2^d| \in [8, 12]$ for $10^2 \leq d \leq 10^3$. 
For the noiseless setting, we choose step sizes: $\gradstep, \hessstep, \pardevstep$ and thresholds: $\hesssamperr, \derivsamperr$ 
as in Lemma \ref{lem:rec_act_set} and Lemma \ref{lem:est_ind_sets}. 

\begin{figure}[!ht]
\begin{center}
	\includegraphics[width=0.32\textwidth]{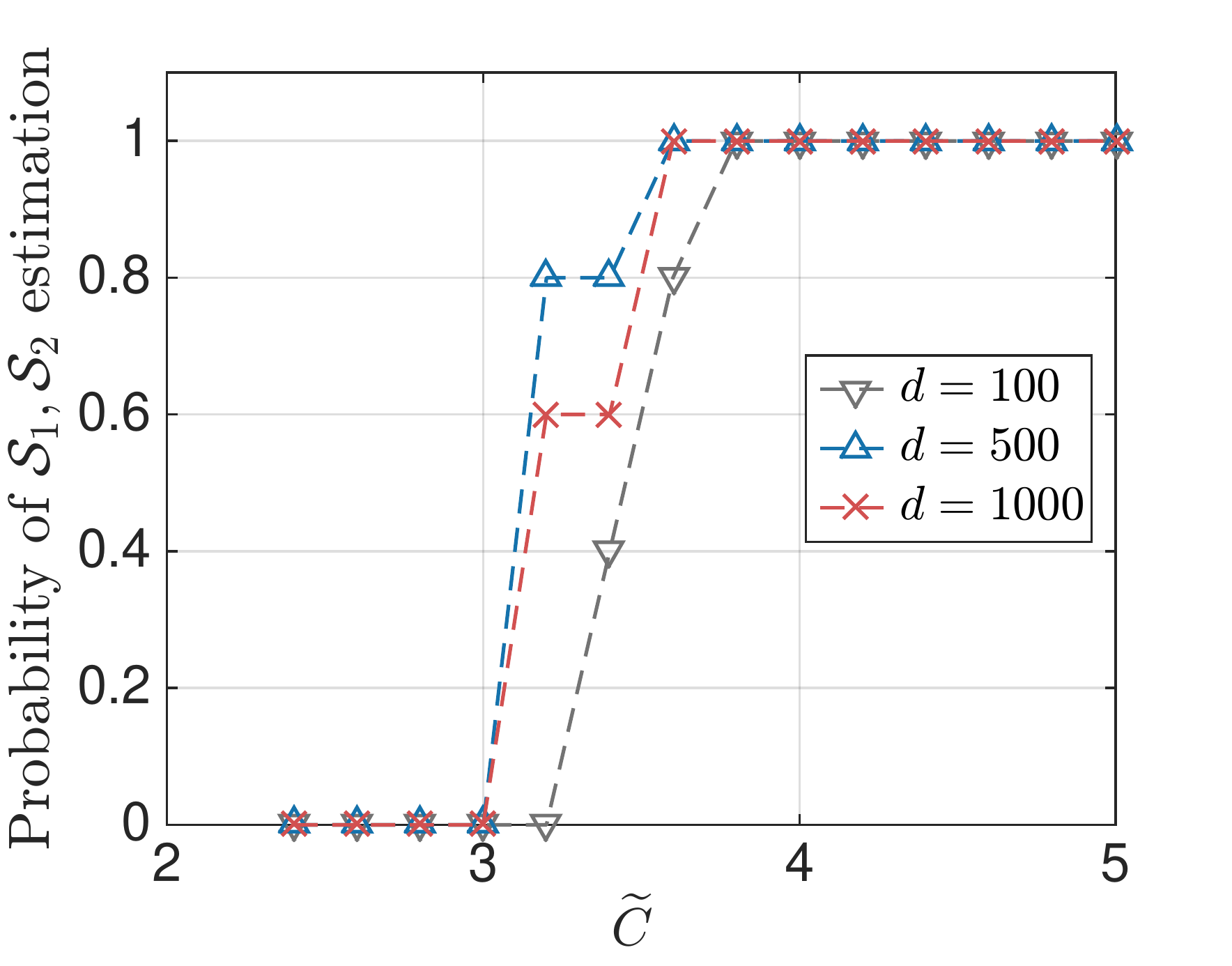} \includegraphics[width=0.32\textwidth]{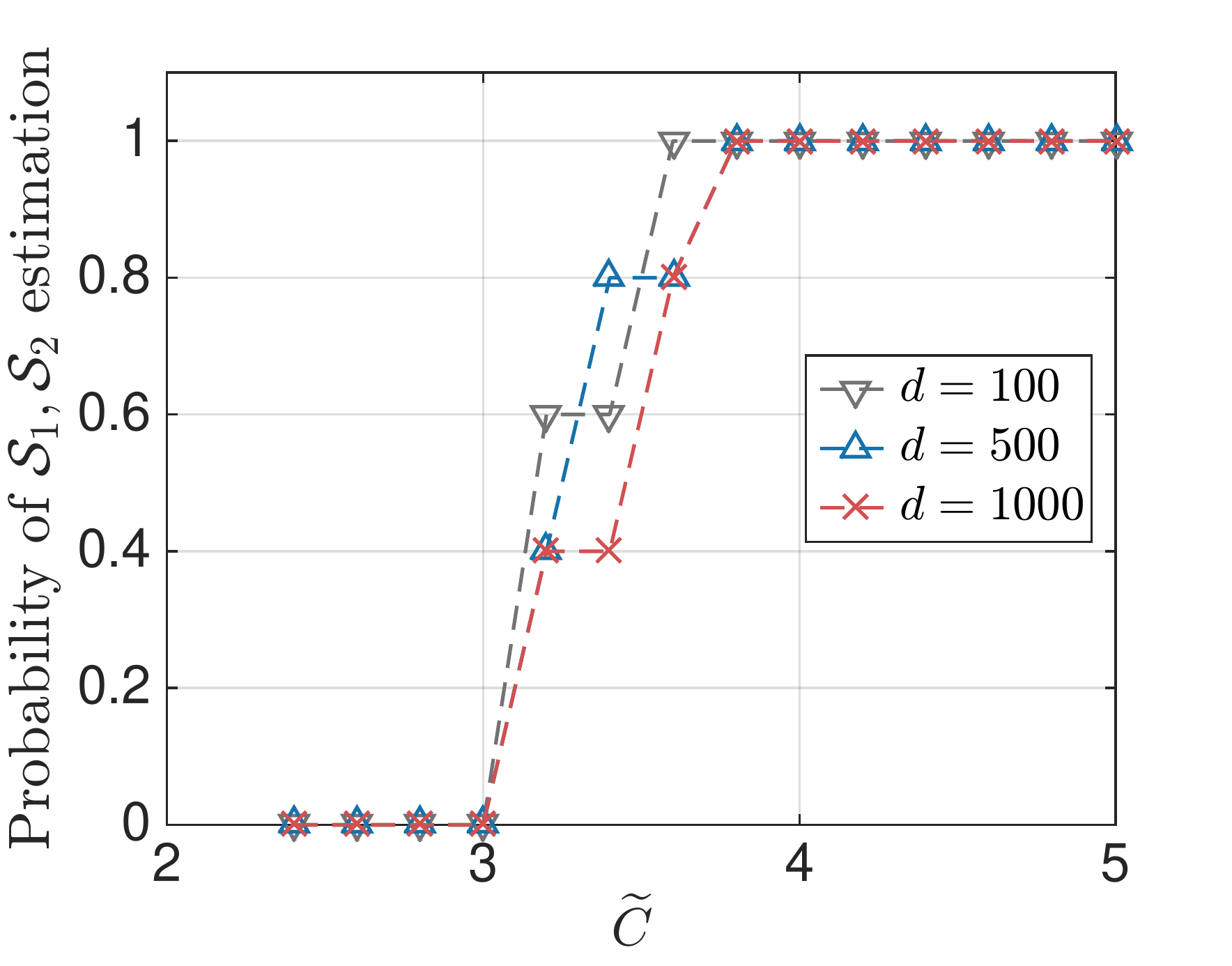} \includegraphics[width=0.32\textwidth]{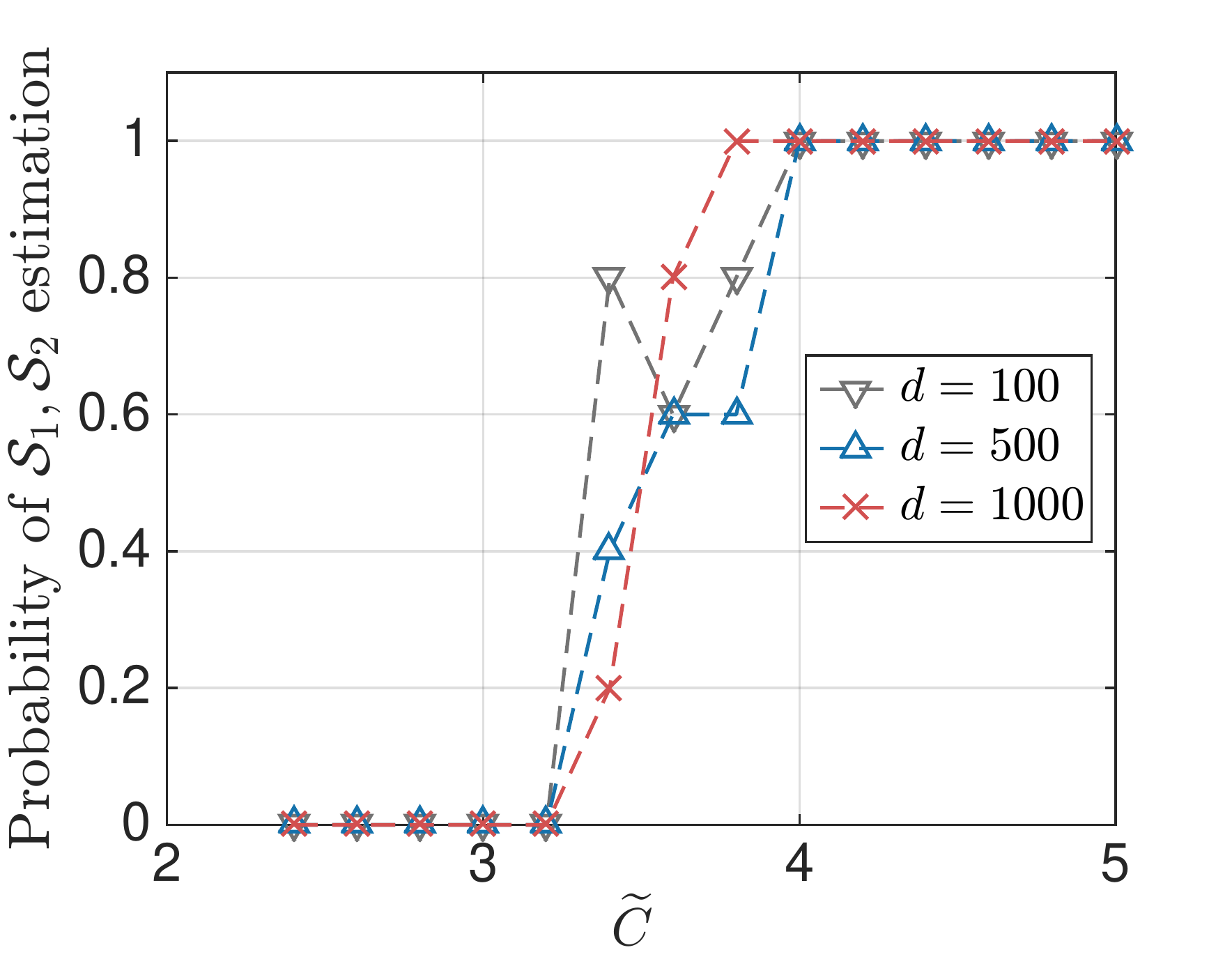}  \\ 
		\includegraphics[width=0.32\textwidth]{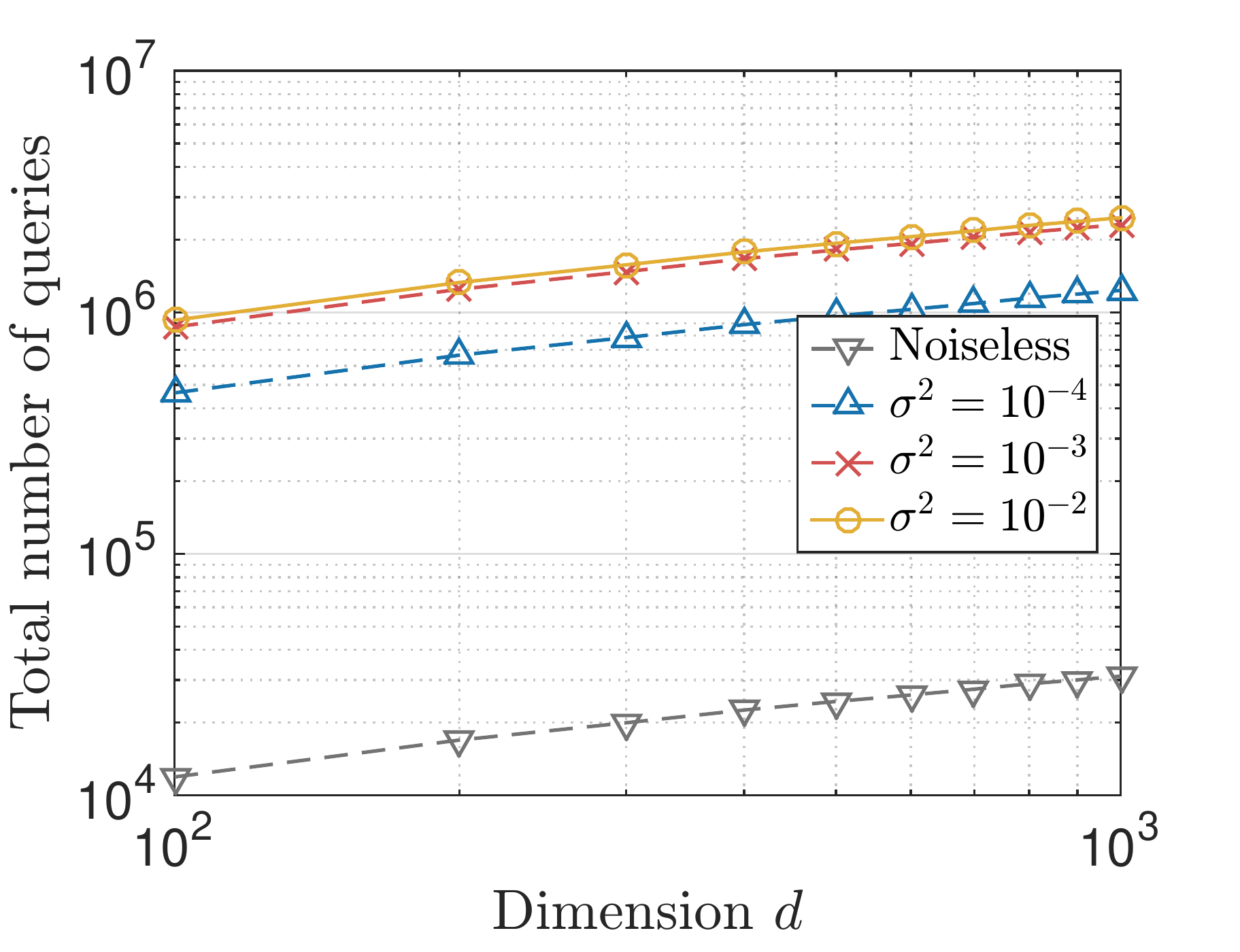} \includegraphics[width=0.32\textwidth]{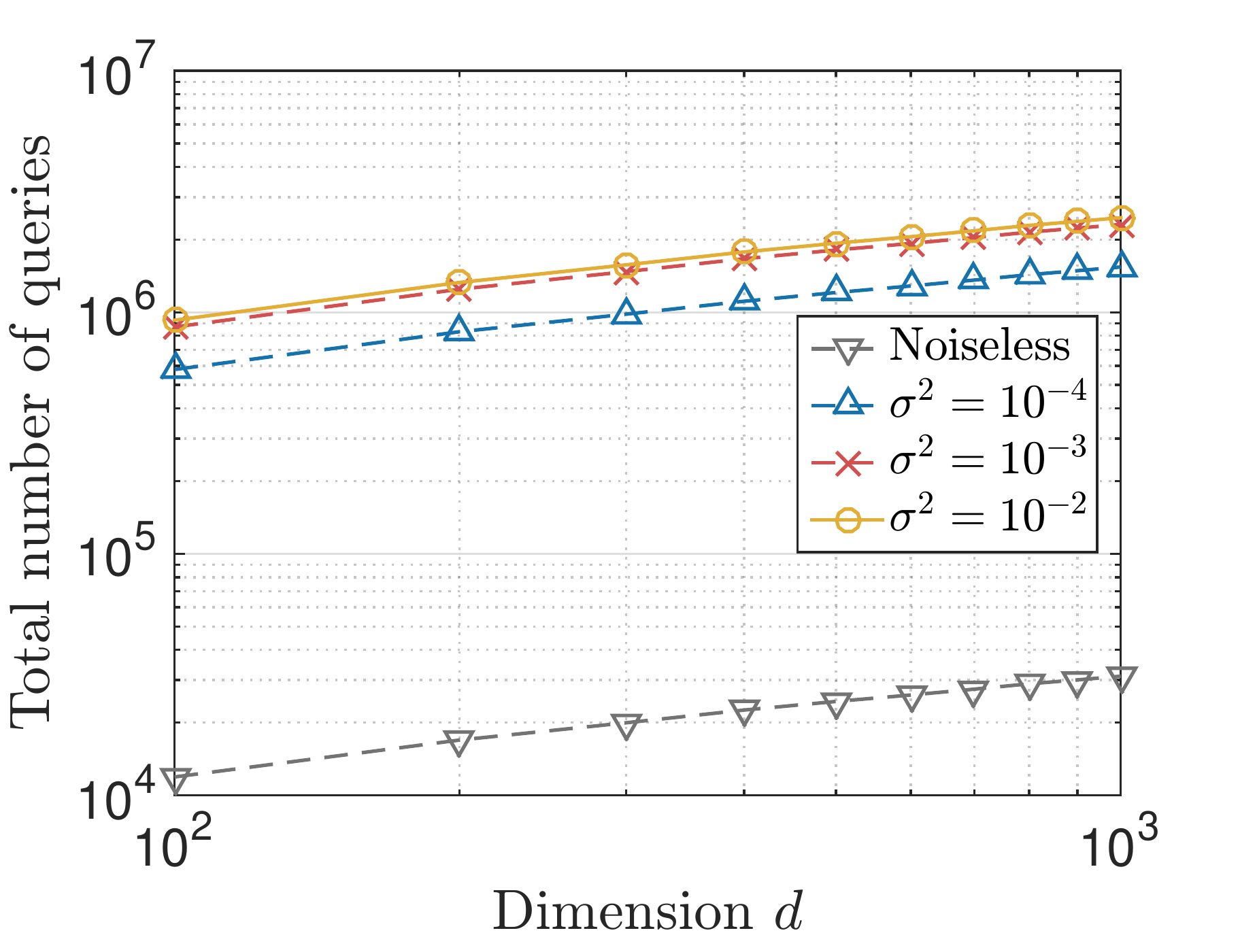} 
		\includegraphics[width=0.32\textwidth]{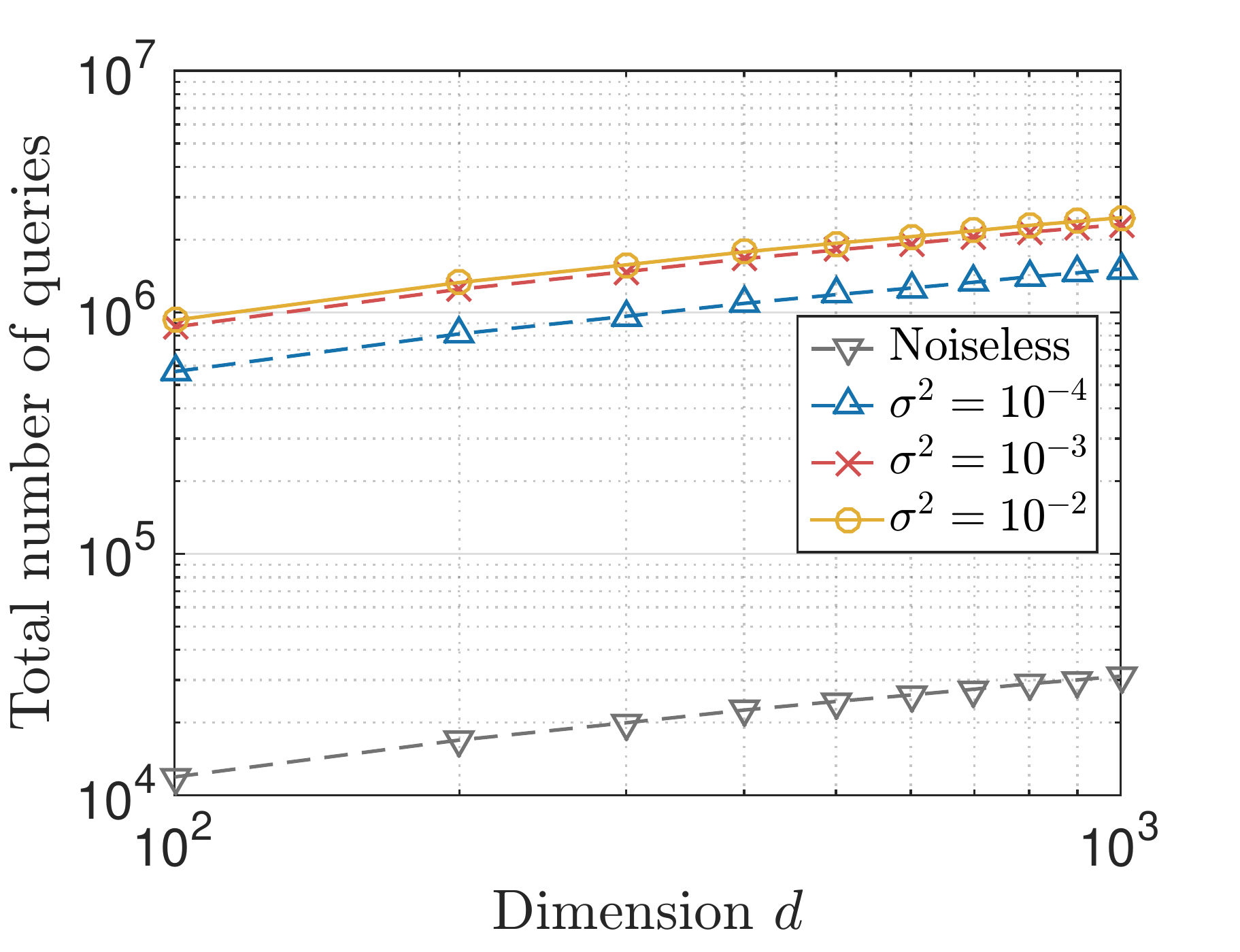}
\end{center} 
\caption{\small First (resp. second and third) column is for $f_1$ (resp. $f_2$ and $f_3$). Top row depicts
the success probability of identifying \empty{exactly} $\univsupp,\bivsupp$, in the noiseless case. 
$x$-axis represent the constant $\widetilde{C}$. The bottom panel depicts 
total queries vs. $\dimn$ for exact recovery, with $\widetilde{C} = 3.8$ and various noise settings.
All results are over $5$ independent Monte Carlo trials.} \label{exp:f1_f2_f3_plots_non}
\end{figure}

For the noisy setting, we consider the function values to be corrupted with i.i.d. Gaussian noise. 
We reduce the noise variance by repeating each query $N_1$ and $N_2$ times respectively, and averaging. 
The noise variance values considered are $\sigma^2 \in \left\{10^{-4}, 10^{-3}, 10^{-2}\right\}$ for which we choose: 
\begin{align*} 
(N_1,N_2) &\in \set{(40,15), (75,31), (80,35)} \quad \text{for} \quad f_1, \\ 
(N_1,N_2) &\in \set{(60,30), (85,36), (90,40)} \quad \text{for} \quad f_2, \\
\text{and} \quad (N_1,N_2) &\in \set{(59,30), (85,35), (90,40)} \quad \text{for} \quad f_3.
\end{align*}
Moreover, we now choose parameters $\gradstep, \hessstep, \pardevstep, \hesssamperr, \derivsamperr$ 
as in Theorem \ref{thm:gen_no_over_arbnois}. 

\paragraph{Dependence on $\dimn$.} 
We see in Fig. \ref{exp:f1_f2_f3_plots_non}, that for $\widetilde{C} \approx 3.8$ the probability of successful identification (noiseless case)
undergoes a phase transition and becomes close to $1$, for different values of $\dimn$. This validates 
the statements of Lemmas \ref{lem:rec_act_set}-\ref{lem:est_ind_sets}. Fixing $\widetilde{C} = 3.8$, we then see that with the total number of queries 
growing slowly with $\dimn$, we have successful identification. For the noisy case, the total number of queries is roughly 
$10^2$ times that in the noiseless setting, however the scaling with $\dimn$ is similar to that for noiseless case. 
Focusing on the function models $f_2$ and $f_3$, observe that the number of queries is seen to be slightly larger 
than that for $f_1$ in the noisy settings; this fact becomes more obvious in the overlapping case later on.

\paragraph{Dependence on $\totsparsity$.} We now demonstrate the scaling of the total number of queries versus 
the sparsity $\totsparsity$ for identification of $\univsupp,\bivsupp$. Consider the model  
\begin{small}
\begin{align}
f(\vecx) &= \sum_{i = 1}^T \Big(\alpha_1 \vecx_{(i-1)5 + 1} - \alpha_2\vecx_{(i-1)5 + 2}^2
+ \alpha_3 \vecx_{(i-1)5 + 3} \vecx_{(i-1)5 + 4} - \alpha_4 \vecx_{(i-1)5 + 5} \vecx_{(i-1)5 + 6}\Big) \label{eq:f_diff_k_exp_non}
\end{align} 
\end{small} 
where $\vecx \in \matR^{\dimn}$ for $d  = 500$. Here, $\alpha_i \in [2, 5], \forall i$; \textit{i.e.}, 
we randomly selected $\alpha_i$'s within range and kept the values fixed for all $5$ Monte Carlo iterations. 
Note that sparsity $\totsparsity = 6T$; we consider $T \in \left\{1, 2, \dots, 10\right\}$. 
We set $\critintmeas_1 = 0.3$, $\critintmeas_2 = 1$, $\idenconst_1 = 2$, $\idenconst_2 = 3$, $\smconst_3 = 6$ and 
$\widetilde{C} = 3.8$, \textit{i.e.}, the same setting with model $f_1$ above. 
For the noisy cases, we consider $\sigma^2$ as before, and choose the same values for $(N_1,N_2)$ as for $f_1$. 
In Figure \ref{fig:exp_k_non}, we see that the number of queries scales as $\sim \totsparsity \log(\dimn/\totsparsity)$, and 
is roughly $10^2$ more in the noisy case as compared to the noiseless setting.

\begin{figure}[!htp]
    \begin{center}
        \includegraphics[width=0.5\textwidth]{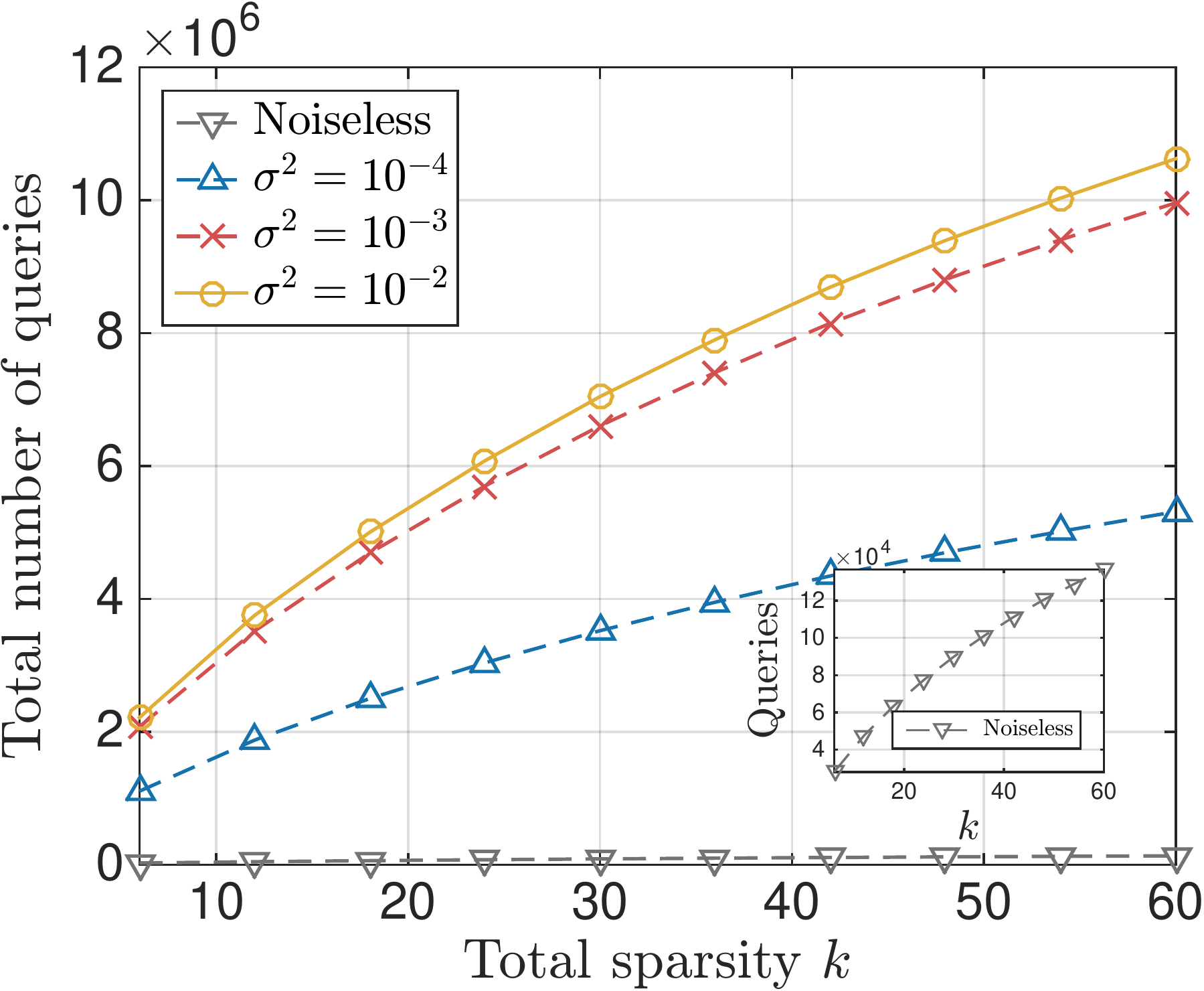}
    \end{center}
    \caption{\small Total number of queries versus different sparsity values $k$, for \eqref{eq:f_diff_k_exp_non}.
    This is for both noiseless and noisy cases (i.i.d Gaussian) with variances
    $\sigma^2 \in \left\{10^{-4}, 10^{-3}, 10^{-2}\right\}$.} \label{fig:exp_k_non}
\end{figure}

\subsection{Overlapping setting} \label{sec:sims_over}
For the overlapping case, we set $\univsupp = \left\{1, 2 \right\}$ and $\bivsupp = \{(3, 4), (4, 5)\}$,
which implies $\univsparsity = 2$, $\bivsparsity = 2$, $\maxdegree = 2$ and $\totsparsity = 5$. 
Due to the presence of overlap between the elements of $\bivsupp$, we now employ Algorithm \ref{algo:gen_overlap} 
for identifying $\univsupp,\bivsupp$. 
\begin{remark} \label{rem:overlap_exps_spams}
We deliberately avoid using 
Algorithm \ref{algo:gen_overlap_alt} on account of Remark \ref{rem:alt_gen_scheme_IHT} -- it is unclear to us whether 
IHT based methods could be employed for solving \eqref{eq:l1_min_sparse_symm}, with provable recovery guarantees. 
While we could instead use standard interior point solvers, they will be slow, especially for the range of values of dimension $d$
that we will be considering.
\end{remark}
For an easier comparison with the non-overlapping case, we consider similar models as the previous subsection; 
observe that there are now common variables across the components of $f$.
\begin{itemize}
\item [$(i)$] {\small $f_1(\vecx) = 2x_1 - 3x_2^2 + 4x_3x_4 - 5x_4x_5$}, 
\item [$(ii)$] {\small $f_2(\vecx) = 10 \sin(\pi \cdot x_1) + 5 e^{-2x_2} + 10\sin(\pi \cdot x_3 x_4) + 5 e^{-2x_4 x_5}$},
\item [$(iii)$] {\small $f_3(\vecx) = \frac{10}{3} \cos(\pi \cdot x_1) + 8 x_1^2 + 5 ( x_2^4 - x_2^2 + \frac{4}{5} x_4) + \frac{10}{3} \cos(\pi \cdot x_3x_4) 
+ 8 (x_3x_4)^2 + 5 ( (x_4x_5)^4 - (x_4x_5)^2 + \frac{4}{5} x_4x_5)$}.
\end{itemize} 

Parameters $\critintmeas_1, \critintmeas_2, \idenconst_1, \idenconst_2, \smconst_3$ are set as in the previous subsection.
For a constant $\widetilde{C}$ (chosen later), we set $\numdirec := \widetilde{C} \totsparsity \log\left(\dimn/\totsparsity\right)$, 
$\numdirecp := \widetilde{C} \maxdegree \log(\dimn/\maxdegree),$ and 
$\numdirecpp := \widetilde{C}  (\totsparsity-\abs{\est{\bivsuppvar}}) \log(\frac{\abs{\calP}}{\totsparsity-\abs{\est{\bivsuppvar}}})$. 
The size of the hash family $\abs{\twohashfam}$ for different values of $d$ is set as before for the non-overlapping setting.
For the noiseless setting, we choose step sizes: $\gradstep, \hessstep, \gradstepp$ and thresholds: $\hesssamperr, \derivsamperrpp$ 
as in Theorem \ref{thm:gen_overlap}. 

For the noisy setting, we consider the function values to be corrupted with i.i.d. Gaussian noise. 
We reduce the noise variance by repeating each query $N_1$ and $N_2$ times respectively, and averaging. 
The noise variance values considered are: $\sigma^2 \in \left\{10^{-4}, 10^{-3}, 10^{-2}\right\}$ for which we choose: 
\begin{align*} 
(N_1,N_2) &\in \set{(50,20), (85,36), (90,40)} \quad \text{for} \quad f_1, \\ 
(N_1,N_2) &\in \set{(60,30), (90,40), (95,43)} \quad \text{for} \quad f_2, \\
\text{and} \quad (N_1,N_2) &\in \set{(59,30), (89,40), (93,43)} \quad \text{for} \quad f_3.
\end{align*}
Moreover, we now choose the parameters: $\gradstep, \hessstep, \gradstepp, \hesssamperr, \derivsamperrpp$ as 
in Theorem \ref{thm:gen_overlap_arbnois}. 

\paragraph{Dependence on $\dimn$.} 
We see in Fig. \ref{exp:f1_f2_f3_plots}, that for $\widetilde{C} \approx 5.6$ the probability of successful identification (noiseless case)
undergoes a phase transition and becomes close to $1$, for different values of $\dimn$, as in the non-overlapping case. This validates 
the statement of Theorem \ref{thm:gen_overlap}. As in the non-overlapping case, in the presence of noise, the total number of queries is 
roughly $10^2$ times that in the noiseless setting, however the scaling with $\dimn$ is similar to that for the noiseless setting. 
\begin{figure}[!ht]
\begin{center}
	\includegraphics[width=0.32\textwidth]{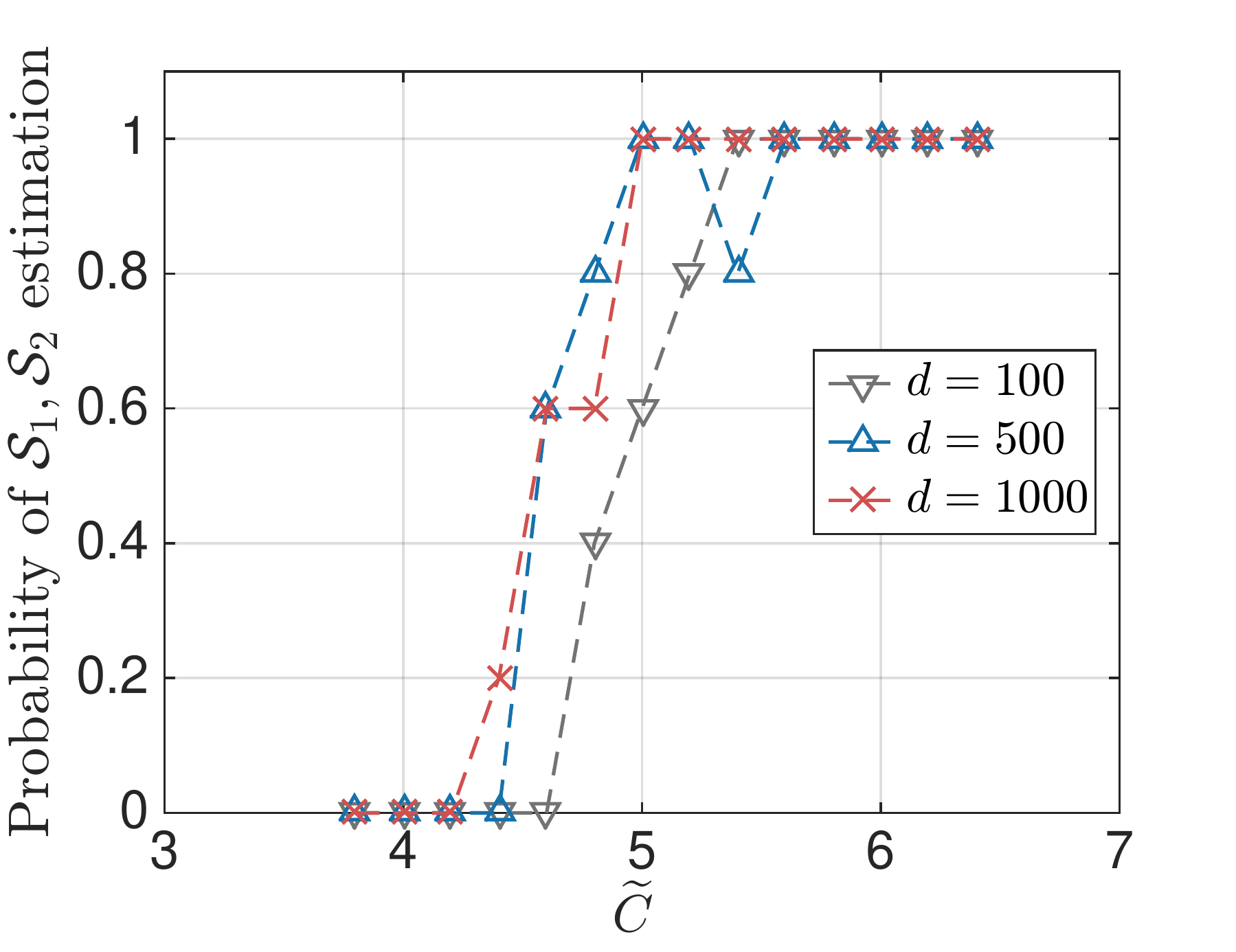} \includegraphics[width=0.32\textwidth]{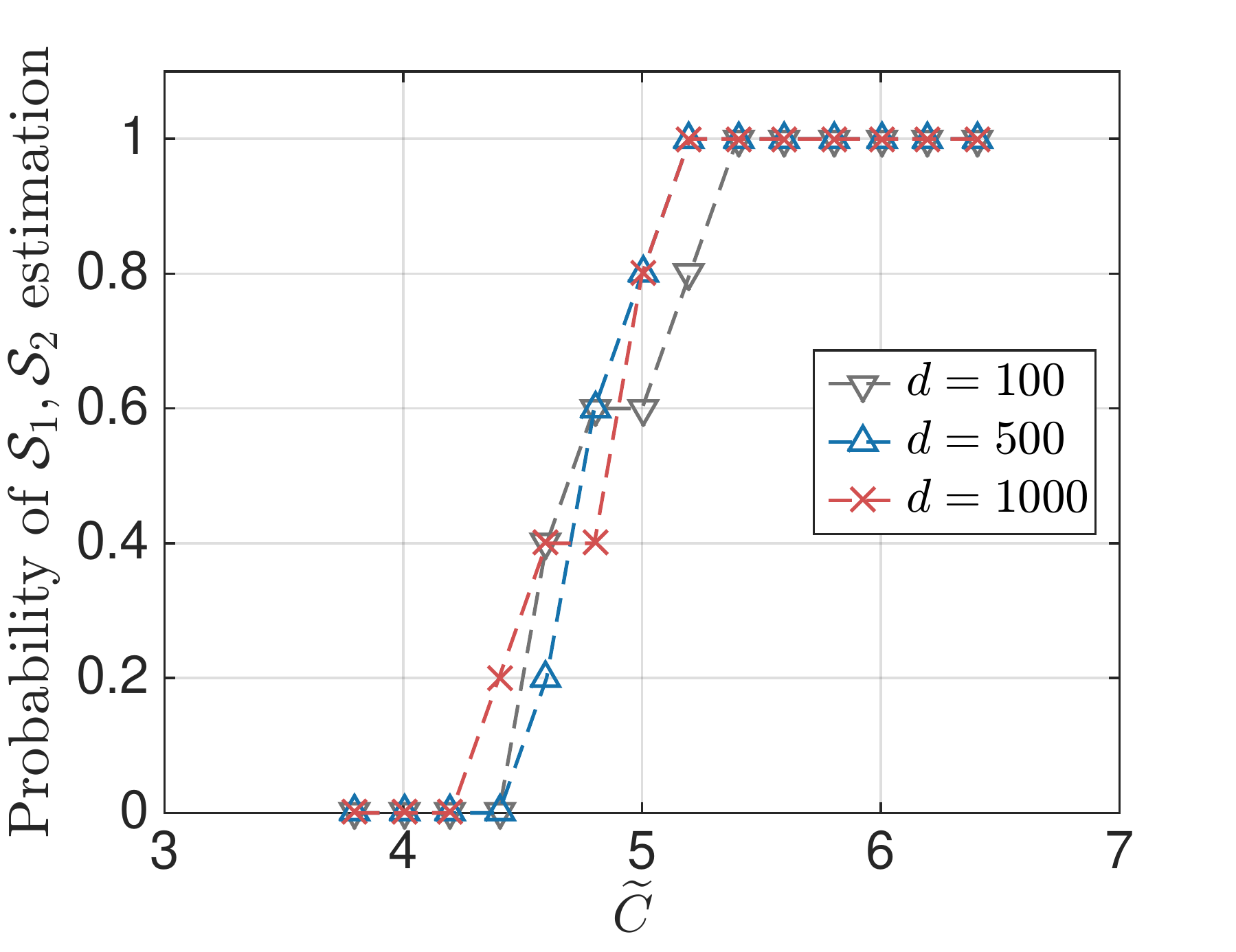} \includegraphics[width=0.32\textwidth]{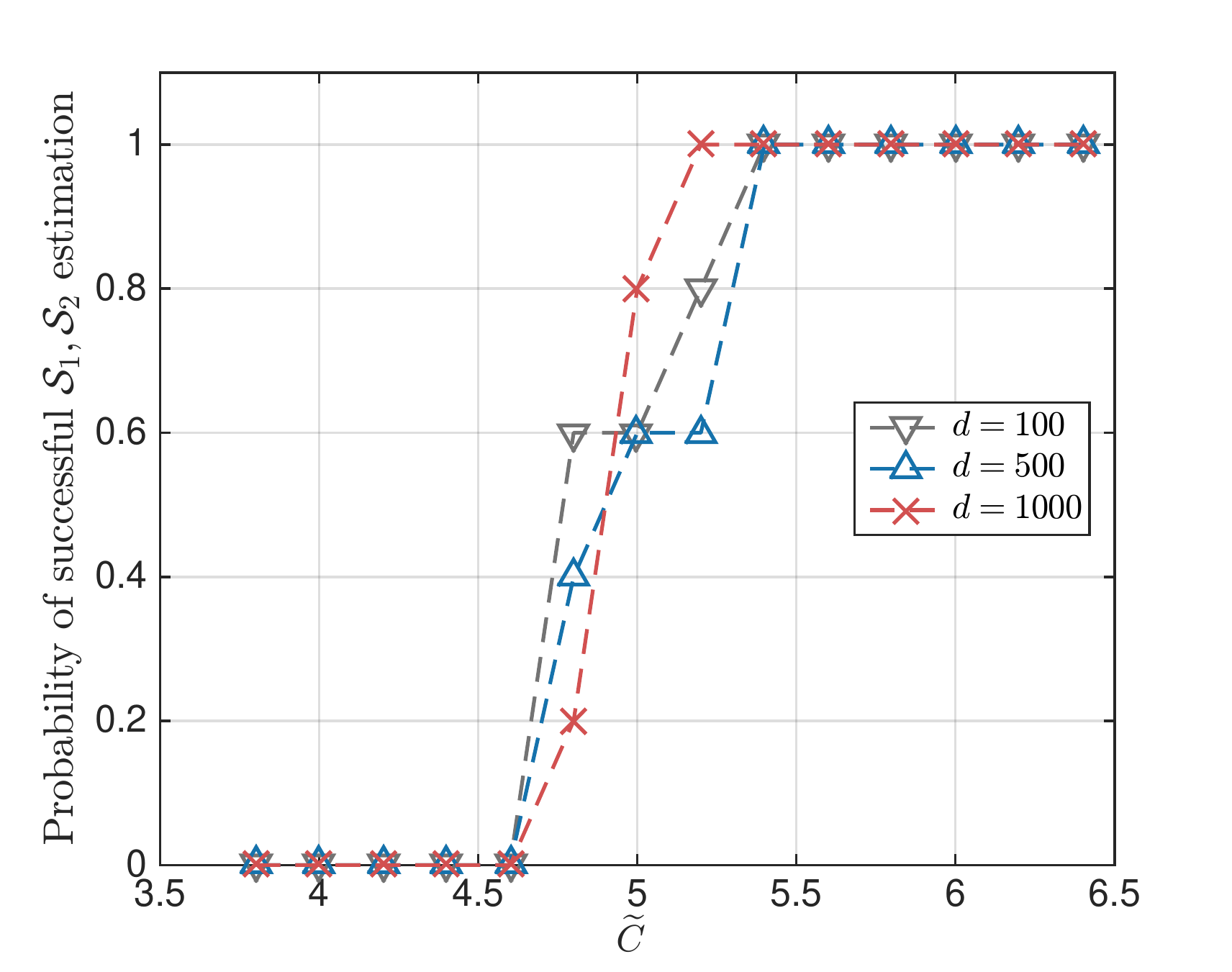}  \\ 
		\includegraphics[width=0.32\textwidth]{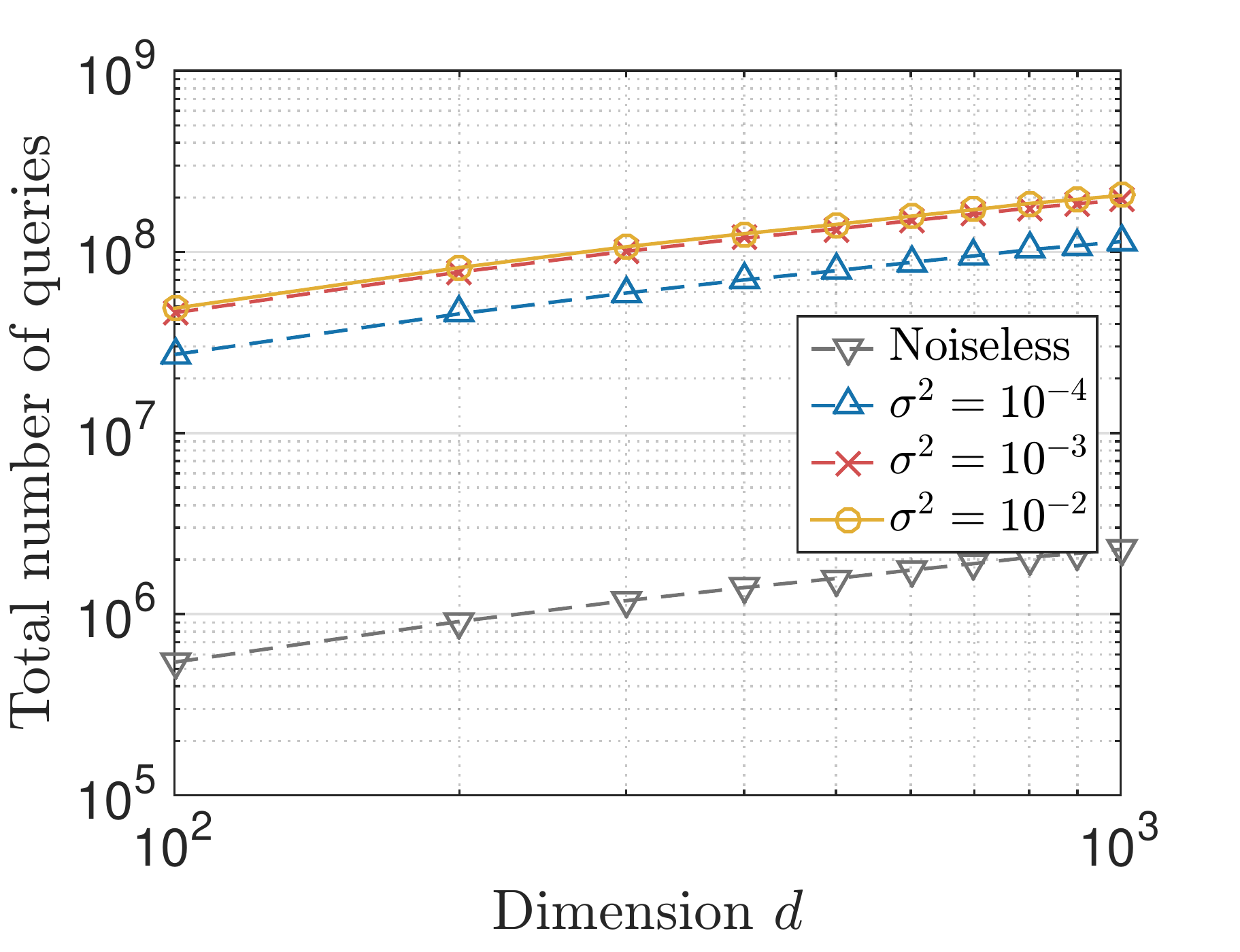} \includegraphics[width=0.32\textwidth]{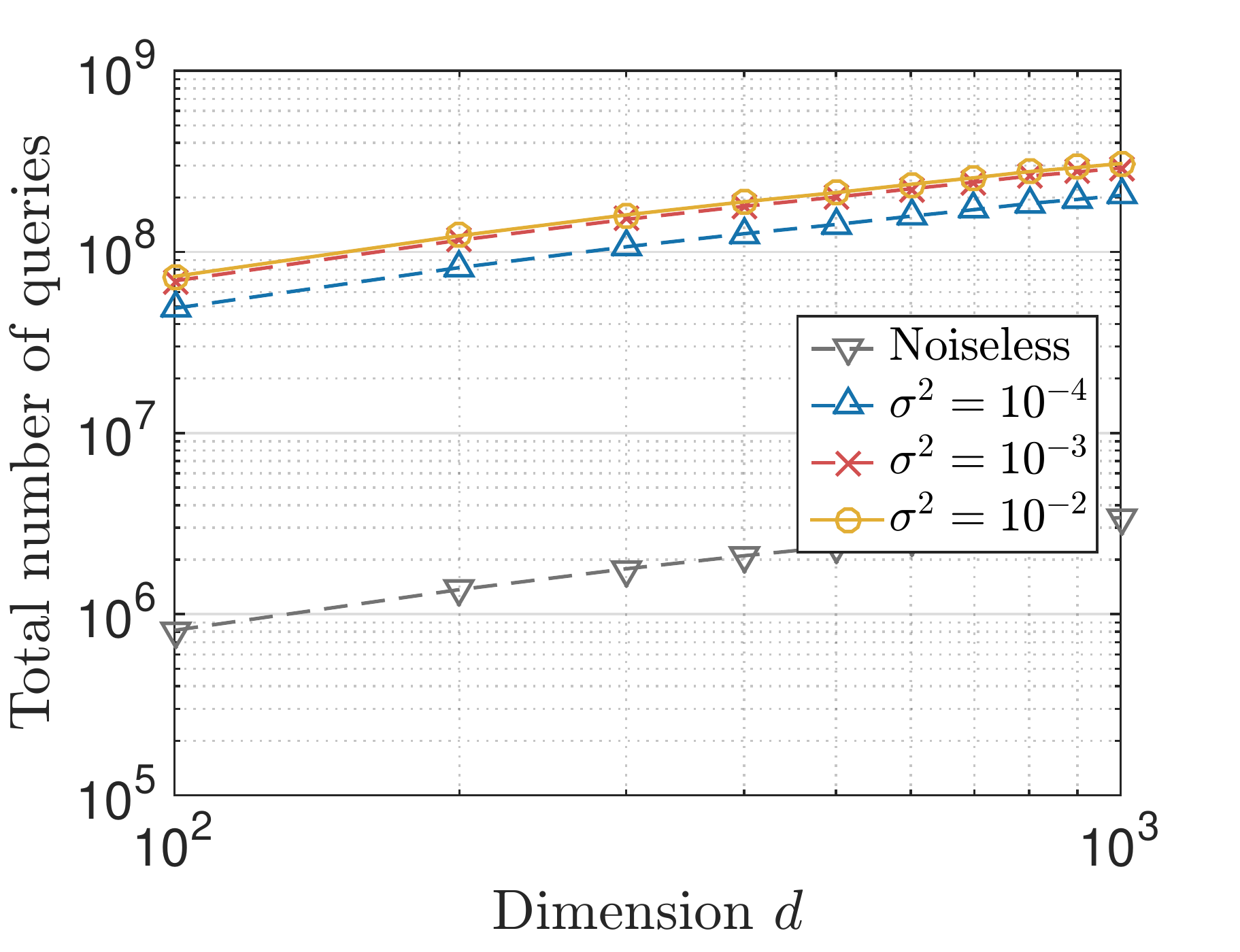} 
		\includegraphics[width=0.32\textwidth]{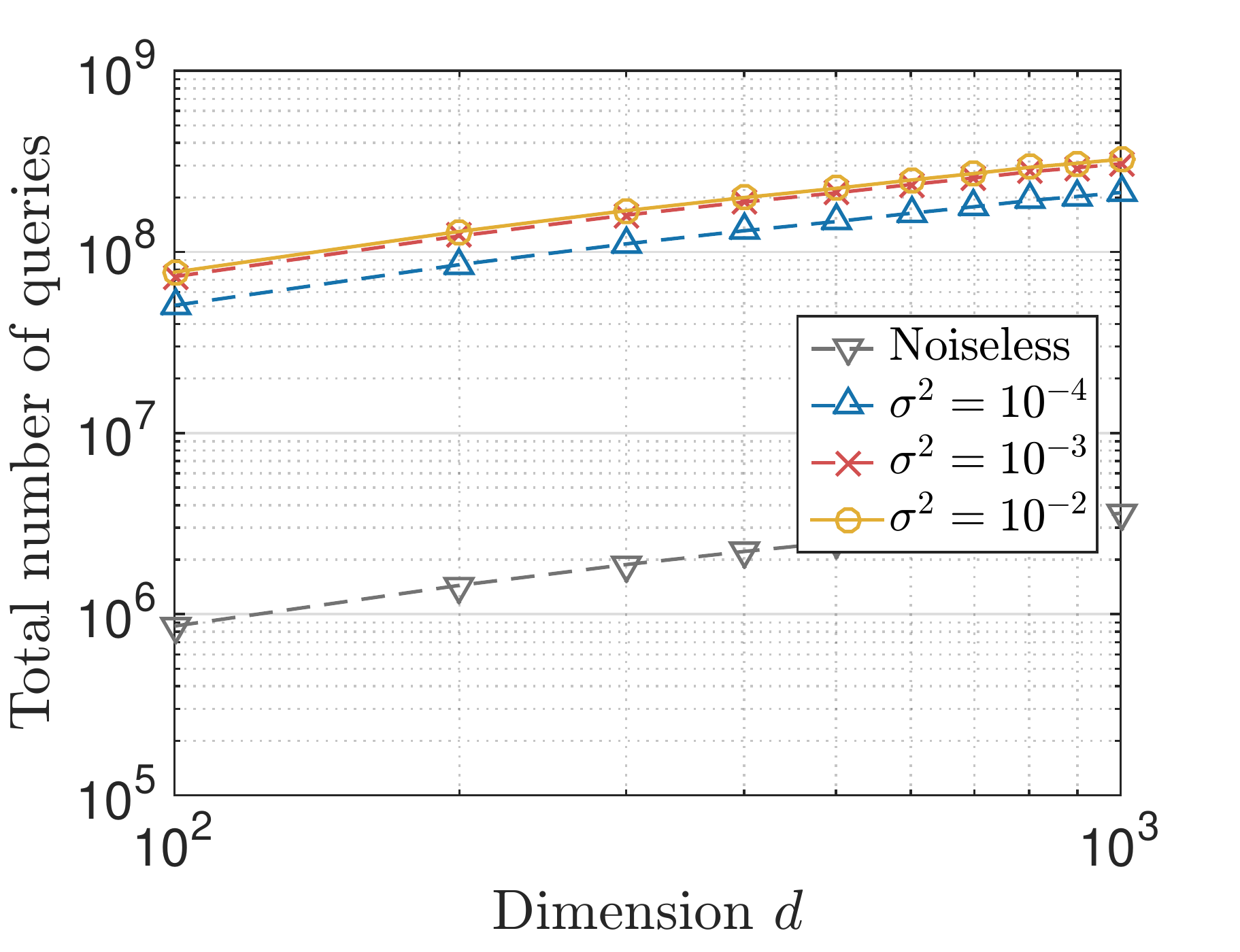}
\end{center} 
\caption{\small First (resp. second and third) column is for $f_1$ (resp. $f_2$ and $f_3$). Top row depicts
the success probability of identifying \empty{exactly} $\univsupp,\bivsupp$, in the noiseless case. 
$x$-axis represent the constant $\widetilde{C}$. The bottom panel depicts 
total queries vs. $\dimn$ for exact recovery, with $\widetilde{C} = 5.6$ and various noise settings.
All results are over $5$ independent Monte Carlo trials.} \label{exp:f1_f2_f3_plots}
\end{figure}

\paragraph{Dependence on $\totsparsity$.} We now demonstrate the scaling of the total number of queries versus 
the sparsity $\totsparsity$ for identification of $\univsupp,\bivsupp$. Consider the model  
\begin{small}
\begin{align}
f(\vecx) &= \sum_{i = 1}^T \Big(\alpha_1 \vecx_{(i-1)5 + 1} - \alpha_2\vecx_{(i-1)5 + 2}  ^2 
+ \alpha_3 \vecx_{(i-1)5 + 3} \vecx_{(i-1)5 + 4} - \alpha_4 \vecx_{(i-1)5 + 4} \vecx_{(i-1)5 + 5}\Big) \label{eq:f_diff_k_exp}
\end{align} 
\end{small} 
where $\vecx \in \matR^{\dimn}$ for $d  = 500$. Here, $\alpha_i \in [2, 5], \forall i$; \textit{i.e.}, 
we randomly selected $\alpha_i$'s within range and kept the values fixed for all $5$ Monte Carlo iterations. 
Note that $\maxdegree = 2$ and the sparsity $\totsparsity = 5T$; we consider $T \in \left\{1, 2, \dots, 10\right\}$. 
We set $\critintmeas_1 = 0.3$, $\critintmeas_2 = 1$, $\idenconst_1 = 2$, $\idenconst_2 = 3$, $\smconst_3 = 6$ and $\widetilde{C} = 5.6$. 
For the noisy cases, we consider $\sigma^2$ as before, and choose the same values for $(N_1,N_2)$ as for $f_1$. 
In Figure \ref{fig:exp_k}(Left panel), we again see that the number of queries scales as $\sim \totsparsity \log(\dimn/\totsparsity)$, and 
is roughly $10^2$ more in the noisy case as compared to the noiseless setting. 
\begin{figure}[!ht]
	\begin{center}
		\includegraphics[width=0.41\textwidth]{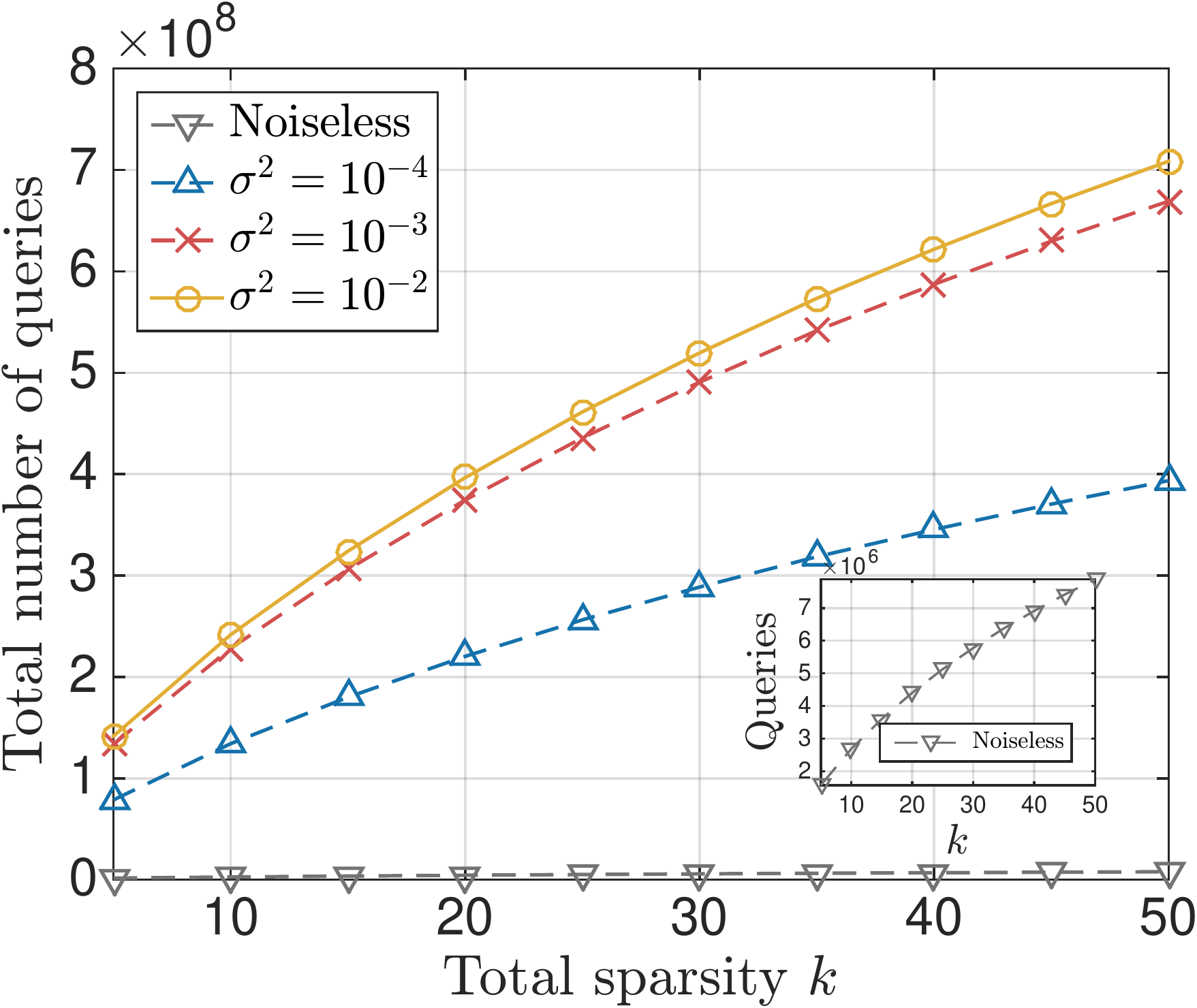} \hspace{0.5cm} \includegraphics[width=0.42\textwidth]{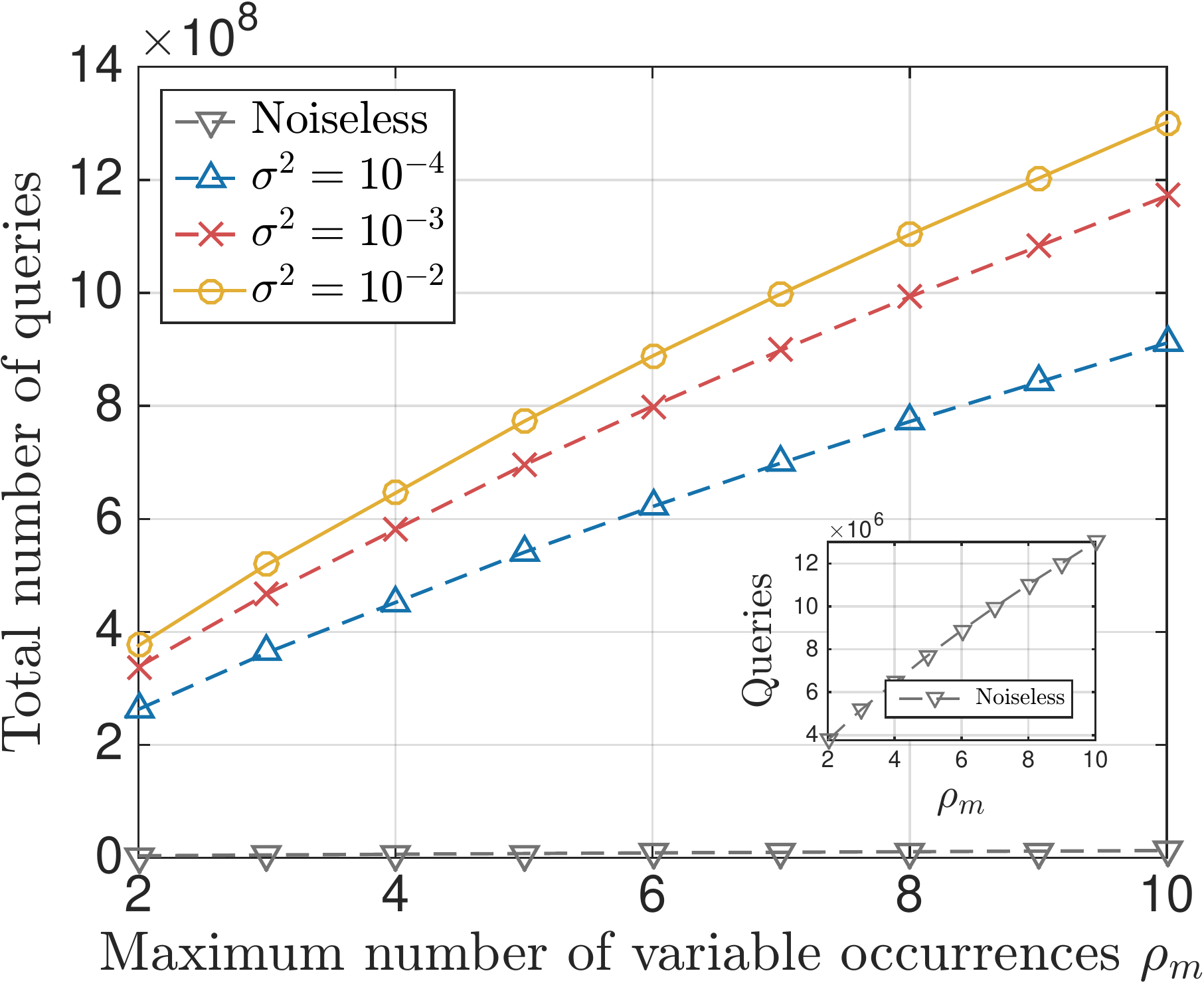} 
	\end{center} 
	\caption{\small Left panel: Total number of queries versus different sparsity values $k$, for \eqref{eq:f_diff_k_exp}. 
	Right panel: Total number of queries versus $\maxdegree$ 
	for \eqref{eq:f_diff_rho_exp}. This is for both noiseless and noisy cases (i.i.d Gaussian) with variances 
	$\sigma^2 \in \left\{10^{-4}, 10^{-3}, 10^{-2}\right\}$.} \label{fig:exp_k}
\end{figure}
\paragraph{Dependence on $\maxdegree$.}
We now demonstrate the scaling of the total queries versus 
the maximum degree $\maxdegree$ for identification of $\univsupp,\bivsupp$. Consider the model $f(\vecx) = $
\begin{small}
\begin{align}
\alpha_1 \vecx_{1} - \alpha_2\vecx_{2}^2 + \sum_{i = 1}^{T} \left(\alpha_{3,i} \vecx_{3} \vecx_{i+3}\right) 
+ \sum_{i=1}^{5}\left(\alpha_{4,i}\vecx_{2+2i}\vecx_{3+2i} \right). \label{eq:f_diff_rho_exp}
\end{align} 
\end{small}
We choose $d = 500$, $\widetilde{C} = 6$, $\alpha_i \in [2,\dots, 5], \forall i$ (as earlier) and set  
$\critintmeas_1 = 0.3$, $\critintmeas_2 = 1$, $\idenconst_1 = 2$, $\idenconst_2 = 3$, $\smconst_3 = 6$.
For $T \geq 2$, we have $\maxdegree = T$; we choose $T \in \set{2,3,\dots,10}$. 
Also note that $\totsparsity = 13$ throughout. For the noisy cases, we consider $\sigma^2$ as before, and choose 
$(N_1,N_2) \in \set{(70,40), (90,50), (100,70)}$. In Figure \ref{fig:exp_k}(Right panel), 
we see that the number of queries scales as $\sim \maxdegree \log(\dimn/\maxdegree)$, and 
is roughly $10^2$ more in the noisy case as compared to the noiseless setting.
%
\section{Discussion} \label{sec:discuss}
We now provide a more detailed discussion with respect to related work, \hemant{starting with results for learning 
SPAMs.} 
\paragraph{Learning SPAMs.}
{\hemantt Ravikumar et al. \cite{Ravi2009}, Meier et al. \cite{Meier2009} proposed methods based on least
squares loss regularized with sparsity and smoothness constraints. While Ravikumar et al. 
show their method to be sparsistent for second order Sobolev smooth $f$, 
one can obtain a rough estimate of how the number of samples $n$ behaves 
with respect to $\totsparsity, d$. Indeed, from Corollary 1 of Theorem 2 in \cite{Ravi2009}, we see that the 
probability of incorrect identification of $\totsupp$ approximately 
scales\footnote{Here, we set the term $\rho^{*}_n$ capturing the minimum magnitude of the 
univariate components (as defined in \cite[Theorem 2]{Ravi2009}) to O(1).} as: 
$\frac{\log d}{(\log n)^2} + \frac{\totsparsity}{\log n} + \frac{\log d\sqrt{\totsparsity}}{n^{1/6}}$.
This means that $n$ roughly scales as $\max\{\totsparsity^3 (\log d)^6, e^{\totsparsity}, e^{\sqrt{\log d}}\}$,
for a constant probability of error. In contrast, our $O(k^2 (\log d)^2)$ bound (recall Theorem \ref{thm:gen_overlap_gaussnois}) 
has a clearly better scaling. 
 
Meier et al. \cite{Meier2009} derive error rates of $O(\totsparsity (\log d / n)^{2/5})$ for estimating $C^2$ smooth $f$ 
in the empirical $L_2(\prob_n)$ norm. They also show conditions under which their method is guaranteed to
recover $\est{\totsupp} \subset \totsupp$.

Huang et al. \cite{Huang2010} proposed a method based on the adaptive group Lasso,
and show that it is sparsistent. In contrast to \cite{Ravi2009}, it is unclear here how exactly $n$ 
scales with $\totsparsity, d$. They also derive $L_2$ error rates for estimating the individual components of the SPAM. 

Wahl \cite{Wahl15} consider the variable selection problem for SPAMs. They propose an estimator 
that essentially involves looking at all subsets of $\{1,\dots,d\}$ of size $\totsparsity$, and hence is practically 
infeasible. They show that for the periodic Sobolev class of functions (with smoothness parameter $\alpha > 1/2$), 
their estimator recovers $\totsupp$ w.h.p with $O(\totsparsity^{\frac{2\alpha+1}{2\alpha}} (\log d)^4)$ samples \cite[Corollary 3]{Wahl15}.  
Consequently, they are also able to estimate each individual component of the model in the $L_2(\prob)$ norm.
We observe that the dependency of their bound on $d$ is worse than ours by a factor of $(\log d)^2$, however the 
scaling with $k$ is better for all $\alpha > 1/2$.
}

\paragraph{Learning generalized SPAMs.} 
Radchenko et al. \cite{Rad2010} proposed the VANISH algorithm -- a least squares method with sparsity constraints. 
\hemant{Assuming $f$ to be second order Sobolev smooth, they show their method to be sparsistent. 
They also show a consistency result for estimating $f$, similar to \cite{Ravi2009}. 
One can obtain a rough estimate of how their sampling bounds scale with $d,\abs{\univsupp},\abs{\bivsupp}$ for 
exact identification of $\univsupp,\bivsupp$. 
Denoting $m = \abs{\univsupp} + \abs{\bivsupp}$, and $n$ to be the number of samples, we see from Corollary 1 
of \cite[Theorem 2]{Rad2010} that the probability of failure, i.e., incorrect identification of $\univsupp,\bivsupp$, 
approximately scales\footnote{Here, we set the term $b$ capturing the minimum magnitude of the 
univariate and bivariate components (as defined in \cite[Section 3.2]{Rad2010}) to O(1). } 
as $\frac{\sqrt{m}}{\log n} + \frac{(\log d)^3}{n^{3/5}}$. This implies that $n$ roughly scales 
as $\max\{e^{m}, (\log d)^5\}$ for a constant probability of error. 
In contrast, as seen from Theorems \ref{thm:gen_overlap_gaussnois},\ref{thm:gen_overlap_gaussnois_alt}, 
our bounds are polynomial in $m$, and have a better scaling with dimension $d$.}

Dalalyan et al. \cite{Dala2014} studied a generalization of \eqref{eq:intro_gspam_form} that allows for the presence of 
a sparse number ($m$) of $s$-wise interaction terms for some additional sparsity parameter $s$. Specifically, they studied this in the 
Gaussian white noise model\footnote{This is known to be asymptotically equivalent to the nonparametric regression model as the number of samples 
$n \rightarrow \infty.$}. \hemant{Assuming $f$ to lie in a Sobolev space with smoothness parameters $\beta,L > 0$, 
and some $\epsilon \in (0,1)$\footnote{$\epsilon$ corresponds to $\sigma^2/\sqrt{n}$ in regression, where $\sigma^2$ denotes variance of noise.}, 
they derive a non-asymptotic $L_2$ error rate (in expectation) of: 
$\max\{m L^{\frac{s}{2\beta+s}}\epsilon^{\frac{4\beta}{2\beta+s}}, ms\epsilon^2 \log(d/(sm^{1/s}))\}$}, 
which is also shown to be minimax optimal. However, they do not guarantee unique identification of 
the interaction terms for any value of $s$. Furthermore, the computational complexity of their estimator is exponential in
$d,s,m$, although they discuss possible ways to reduce this complexity. 

\hemant{The above model was also recently studied by 
Yang et al. \cite{Yang2015}; they consider a Bayesian estimation of $f$ in the Gaussian process (GP) setting wherein 
a GP prior is placed on $f$, and inference on $f$ is carried out by summarizing the resulting posterior probability given the data.
They derived minimax estimation rates for H\"older smooth $f$ in the $L_2$ norm, 
along with a method that nearly achieves the optimal estimation rate (modulo some log factors) 
in the empirical $L_2(\prob_n)$ norm}. However they do 
not guarantee unique identification of the interaction terms. Suzuki \cite{Suzuki12} studied a special case where $[d]$ 
is pre-divided into $m$ disjoint subsets, with an additive component\footnote{Thus for $m = d$, we obtain a Sparse additive model (SPAM).} 
defined on each subset. Assuming a sparse number of 
components, they derived PAC Bayesian bounds for estimation of $f$ in the $L_2(\prob_n)$ norm. 

A special case of \eqref{eq:intro_gspam_form} -- where $\phi_p$'s are linear and each $\phi_{\lpair}$ is of the form $x_l x_{\lp}$ --
has been studied considerably. Within this setting, there exist algorithms that recover $\univsupp,\bivsupp$, along with 
convergence rates for estimating $f$, in the limit of large $n$ \cite{Choi2010,Rad2010,Bien2013}. 
\hemant{Kekatos et al. \cite{Kekatos11} show that exact recovery is possible (w.h.p) via $\ell_1$ minimization 
with $O((\abs{\univsupp} + \abs{\bivsupp}) (\log d)^4)$ noiseless point queries. This is based on the Restricted 
Isometry Property (RIP) for structured random matrices as developed in \cite{RauhutStruct2010}.}
Nazer et al. \cite{Nazer2010} generalized this to the setting of sparse multilinear systems -- albeit in the noiseless setting -- 
and derived non-asymptotic sampling bounds for identifying the interaction terms, via $\ell_1$ minimization. 
Upon translating Theorem $1.1$ from their paper into our setting, with general overlap (so $\maxdegree \geq 1$), 
we obtain a sample complexity\footnote{This sample complexity implies exact recovery of $\univsupp,\bivsupp$ w.h.p} of 
$O((\abs{\univsupp} + \abs{\bivsupp})^2 \log(d/(\abs{\univsupp} + \abs{\bivsupp}))) = O(\totsparsity^2\maxdegree^2 \log(d/(\totsparsity\maxdegree)))$. 
On the other hand, for the case of no overlap, their sample complexity turns out to be 
$O((\abs{\univsupp} + \abs{\bivsupp}) \log(d/(\abs{\univsupp} + \abs{\bivsupp}))) = O(\totsparsity\log(d/\totsparsity))$ 
for recovering $\univsupp,\bivsupp$ w.h.p. However finite sample bounds for the non-linear model \eqref{eq:intro_gspam_form} are not known in general.

\hemant{We also note that it is common in the statistics literature to impose a heredity constraint on the interactions, wherein an interaction term is present only if the corresponding main effect terms (i.e. those in $\univsupp$) are present (cf., \cite{Choi2010,Rad2010,Bien2013}). This is typically done to make the model interpretable, as interaction terms are difficult to interpret compared to main effect terms.} 

\paragraph{Other low-dimensional function models.} 
We now provide a comparison with existing work related to other low dimensional models from the literature, 
starting with the approximation theoretic setting.  
Devore et al. \cite{Devore2011} consider functions depending on a small subset $\totsupp$ 
of the variables. The functions do not necessarily possess an additive structure, 
thus the setting is more general than \eqref{eq:intro_gspam_form}. 
They provide algorithms that recover $\totsupp$ exactly w.h.p, with $O(c^{\totsparsity} \totsparsity \log d)$ noiseless queries of $f$, 
for some constant $c > 1$. Their methods essentially make use of a $(\dimn,\totsparsity)$-hash family: $\khashfam$ (cf. Definition \ref{def:thash_fam}). 
for constructing their sampling sets, and while these methods could be used for identifying $\totsupp$, the sample complexity 
would be exponential in $\totsparsity$.

Schnass et al. \cite{karin2011} consider the same model for $f$ in the noiseless setting, 
and derive a simple algorithm that recovers $\totsupp$ w.h.p, 
with $O(\frac{C_1^4}{\alpha^4} \totsparsity (\log \dimn)^2)$ noiseless queries. 
Here, $C_1 = \max_{i \in \totsupp} \norm{\partial_i f}_{\infty}$ and 
$\alpha = \min_{i \in \totsupp} \norm{\partial_i f}_1$ with $\norm{\cdot}_1$ denoting the $L_1$ norm. 
While the $C_1$ term is a constant depending on the smoothness of $f$, 
one can construct examples of $f$ for which $\alpha = c^{-\totsparsity}$, for some constant $c > 1$. This implies  
that the sample bounds could be exponential in $\totsparsity$ for general $\totsparsity$ variate functions (as one would expect). 
This method could be applied to \eqref{eq:intro_gspam_form}, 
to learn the set of active variables. In particular, for the general overlap case ($\maxdegree \geq 1$), their algorithm will identify 
the support $\totsupp$ w.h.p, with $O(\frac{{C_1^{\prime}}^4 \maxdegree^4}{\alpha^4} \totsparsity (\log \dimn)^2)$ noiseless 
queries where now: $C_1 = \max_{i \in \totsupp} \norm{\partial_i f}_{\infty} \leq C_1^{\prime}\maxdegree$, with $C_1^{\prime}$ a constant depending on the smoothness of $f$. 
For the general overlap case, we see that their bounds in the noiseless setting are worse by a $\maxdegree^3$ factor 
compared to those for Algorithms \ref{algo:gen_overlap}, \ref{algo:gen_overlap_alt}, however better by a $\log d$ factor compared 
to Algorithm \ref{algo:gen_overlap}. 
Moreover, it is not clear how the $\alpha$ term scales with respect to $\maxdegree$ here. For the non-overlap case, the 
scaling of their sampling bounds with respect to $\totsparsity, d$ matches ours for the noiseless setting, up to an additional $\frac{C_1^4}{\alpha^4}$ 
term. While $\alpha$ does not depend on $\totsparsity$ now, their sampling bound increases for small values of $\alpha$ or large values of $C_1$. 
The dependence of the sampling bound on the parameters $C_1,\alpha$ is 
not necessary in the noiseless setting, as seen from our sampling bounds that (in the noiseless case) 
depend on the \emph{measure} of the region where $\partial_i f$ ($i \in \totsupp$) and/or $\partial_l \partial_{\lp} f$ ($\lpair \in \bivsupp$) 
are large. 

This model was considered by Comminges et al. \cite{Comming2011, Comming2012} in the regression setting. 
Assuming $f$ to be differentiable, and the joint density of the covariates to be known, 
they propose an estimator that identifies the 
unknown subset $\totsupp$ w.h.p, with sample complexity $O(c^{\totsparsity} \totsparsity \log d)$. 
This bound is shown to be tight although the estimator that achieves it is impractical -- in the worst case 
it looks at all subsets of $\set{1,\dots,d}$ of size $\totsparsity$. 

Fornasier et al. \cite{Fornasier2012}, Tyagi et al. \cite{Tyagi2012_nips} generalized this model class to functions $f$ of the form 
$f(\vecx) = g(\matA\vecx)$, for unknown $\matA \in \matR^{k \times d}$. They derive algorithms that approximately 
recover the row-span of $\matA$, with sample complexities\footnote{These were derived predominantly in the noiseless 
setting, with some discussion in \cite{Tyagi2012_nips} about handling Gaussian noise via resampling and averaging.} 
typically \emph{polynomial} in $\totsparsity, \dimn$. 
Specifically, \cite{Fornasier2012} considers the setting where the rows of $\matA$ are sparse. They propose a method 
that essentially estimates the gradient of $f$ -- via $\ell_1$ minimization -- at suitably (typically polynomially in $d$) 
many points on the unit sphere $\mathbb{S}^{d-1}$. 
\cite{Tyagi2012_nips} generalized this result to the setting where $\matA$ is not necessarily sparse, by making use of low rank matrix recovery techniques. 

\paragraph{Estimation of sparse Hessian matrices.} There exists related work for estimating sparse Hessian matrices in the optimization 
literature. Powell et al. \cite{Powell79} and Coleman et al. \cite{Cole84} consider the setting where the sparsity structure
of $\hess f(\vecx)$ \emph{is known}, and aim to estimate $\hess f(\vecx)$ via gradient differences. Their aim is to minimize 
the number of gradient evaluations, needed for this purpose. In particular, Coleman et al. \cite{Cole84} approach the problem from a 
graph theoretic point of view and provide a graph coloring interpretation. Bandeira et al. \cite{Bandeira12} consider derivative free 
optimization (DFO) problems, wherein they approximate the underlying objective function $f$, by a quadratic polynomial interpolation model. 
Specifically, they build such a model by assuming $\hess f$ to be sparse, but do not assume the sparsity pattern to be known.  
Their approach is to minimize the $\ell_1$ norm of the entries of the model Hessian, subject to interpolation conditions. 
As they do not assume $\grad f$ to be sparse, they arrive at a sampling bound of 
$O(\dimn (\log \dimn)^4)$ \cite[Corollary $4.1$]{Bandeira12}, for recovering $\grad f(\vecx)$, $\hess f(\vecx)$, with high probability. 
In case $\grad f$ were also sparse, one can verify that their bound changes to 
$O((\abs{\totsupp} + 2\abs{\bivsupp}) (\log(\abs{\totsupp} + 2\abs{\bivsupp}))^2 (\log d)^2)$ = 
$O(\totsparsity\maxdegree (\log(\totsparsity\maxdegree))^2 (\log d)^2)$. \hemant{They essentially make use of 
the Restricted Isometry Property (RIP) for structured random matrices as outlined in Theorem 4.4 of \cite{RauhutStruct2010}.}

\paragraph{Bounded orthonormal systems.} One of the reviewers pointed out another interesting approach that could be used 
for identifying $\univsupp,\bivsupp$, that we now discuss. Note that this is only a rough sketch and verifying the details is 
left for future work. Let $\psi_k(\vecx)$ be a bounded orthonormal 
system\footnote{$\psi_0 \equiv 1$, i.e., it is the constant function.} 
in $L_2([-1,1]^d)$, for $k = 0,1,\dots,N$, consisting of univariate and bivariate functions. This could 
for example be constructed using a subset of the real trigonometric basis functions (see \cite[Section 1.2]{Dala2014}), 
with the $\psi_k$'s satisfying the zero (marginal) mean conditions. In our model, there are a total of $d$ univariate and ${d \choose 2}$ 
bivariate functions. Say we take $N_1$ basis functions per coordinate, and $N_2$ basis functions per coordinate-tuple, so that 
$N = d N_1 + {d \choose 2}N_2$. 

Now, $f(\vecx) = \sum_{k=0}^{N} \alpha_k \psi_k(\vecx) + r(\vecx)$ where $r$ denotes the remainder term. 
Since $f$ is $C^3$ smooth, we can uniformly approximate each univariate and bivariate $\phi$ with error rates: 
$N_1^{-p_1}$ (for some $p_1 > 0$) and $N_2^{-p_2}$ (for some $p_2 > 0$) respectively. 
Using triangle inequality, we then obtain for any $\vecx \in [-1,1]^d$ the bound:  
\begin{eqnarray}
\abs{r(\vecx)} \lesssim \abs{\univsupp} N_1^{-p_1} + \abs{\bivsupp} N_2^{-p_2}. \label{eq:bound_rem}
\end{eqnarray}
So for bounding $\abs{r(\vecx)}$ by a sufficiently small constant, we require $N_1 \sim \abs{\univsupp}^{\frac{1}{p_1}}$ and 
$N_2 \sim \abs{\bivsupp}^{\frac{1}{p_1}}$. 
By querying $f$ at $\vecx_1,\dots,\vecx_m$ (sampled uniformly at random), we get $y_l = f(\vecx_l) + z_l$; $l=1,\dots,m$, which 
in matrix form can be written as $y = \matA \alpha + \vece$. Here, $e_l = z_l + r_l(\vecx)$ and, $\alpha \in \matR^{N}$ is 
$N_1 \abs{\univsupp} + N_2 \abs{\bivsupp}$ sparse. Since the rows of $\matA$ correspond to a bounded orthonormal system (BOS), 
one can recover $\alpha$ via $\ell_1$ minimization\footnote{Consequently, we would be able to identify $\univsupp,\bivsupp$ by 
thresholding.}; using the RIP result for BOS \cite[Theorem 4.4]{RauhutStruct2010}, we 
obtain the bound:
\begin{align}
m &\geq C_1 (\abs{\univsupp}N_1 + \abs{\bivsupp}N_2) \log^2(\abs{\univsupp}N_1 + \abs{\bivsupp}N_2) \log^2(d N_1 + {d \choose 2} N_2) \label{eq:not_blk_sparse_1} \\
&\gtrsim  (\abs{\univsupp}^{\frac{1}{p_1} + 1} + \abs{\bivsupp}^{\frac{1}{p_2} + 1}) \log^2(\abs{\univsupp}^{\frac{1}{p_1} + 1} + \abs{\bivsupp}^{\frac{1}{p_2} + 1}) \log^2(d). \label{eq:not_blk_sparse_2} 
\end{align}
Note that the above bound is super-linear in the sparsity: $\abs{\univsupp} + \abs{\bivsupp}$ 
and this would be the case even when the samples are noiseless. 
In contrast, our bounds for Algorithms \ref{algo:est_act}-\ref{algo:gen_overlap_alt} are linear in sparsity, for the noiseless 
and bounded noise case. Also, observe that $\alpha$ is actually \emph{block sparse}: 
it has ${d \choose 2}$ ``blocks'', each of length $N_2$, out of which exactly 
$\abs{\bivsupp}$ blocks are non-zero. Moreover, there are $d$ blocks, each of length 
$N_1$, out of which $\abs{\univsupp}$ blocks are non-zero. While we are not aware of a 
RIP result for BOS with block sparsity\footnote{The existing ones seem to be only for matrices with i.i.d sub-Gaussian entries.}, 
we would nevertheless still require $m \gtrsim (\abs{\univsupp}N_1 + \abs{\bivsupp}N_2) \sim 
\abs{\univsupp}^{\frac{1}{p_1} + 1} + \abs{\bivsupp}^{\frac{1}{p_2} + 1}$, which 
is super-linear in sparsity. For the setting of Gaussian noise however, it is possible that the above approach
might give a better scaling with $\totsparsity,\maxdegree$ compared to our results.

\section{Concluding remarks} \label{sec:concl_rems}
In this paper, we considered a generalization of Sparse Additive Models, of the form \eqref{eq:intro_gspam_form}, 
now also allowing for the presence of a small number of bivariate components. We started with the special case where each variable 
interacts with at most one other variable, and then moved on to the general setting where variables can possibly be part of more than interaction term. 
For each of these settings, we derived algorithms with sample complexity bounds -- both in the noiseless as well as the noisy query settings. 
For the general overlap case, the identification of the interaction set $\bivsupp$ essentially involved the estimation of the $d \times d$ Hessian of $f$ at carefully chosen points. 
In fact, these points were simply part of a collection of canonical two dimensional uniform grids, within $[-1,1]^d$. Upon identifying $\bivsupp$, 
the estimation of $\univsupp$ was subsequently performed by employing the sampling scheme of Tyagi et al. \cite{Tyagi14_nips} on the reduced set of variables. 
Furthermore, once $\univsupp, \bivsupp$ are identified, we showed how one can recover uniform approximations to the individual components of the model, by additionally 
querying $f$ along the one/two dimensional subspaces corresponding to $\univsupp, \bivsupp$.

For the setting of noiseless queries, we observed that the sample complexity of Algorithm \ref{algo:gen_overlap_alt} is close to optimal. 
However for the noisy setting -- in particular the setting of Gaussian noise -- we saw that the sample complexity of Algorithm \ref{algo:gen_overlap_alt} 
has a worse dependency in terms of $\totsparsity, \maxdegree$ compared to Algorithm \ref{algo:gen_overlap}. 
In general, the sample complexity bounds of our algorithms, in the presence of Gaussian noise, 
have a sub optimal dependence on $\totsparsity,\maxdegree$. This is mainly due to the localized nature of our sampling schemes -- 
the external noise gets scaled by the step size parameter leading to the noise variance scaling up. Hence the number of samples required 
to reduce the noise variance (by resampling and averaging) increases, leading to an increase in the total sample complexity.
An interesting direction for future work would be to consider alternate -- possibly non localized sampling schemes -- 
with improved  non-asymptotic sampling bounds for identifying $\univsupp,\bivsupp$ in the setting of 
Gaussian noise.

Another limitation of our analysis is that it is restricted to $C^3$ smooth functions. It would be interesting to extend the results to 
more general $C^{r}$ smooth functions $r \geq 1$ and also to other smoothness classes such 
as H\"older/Lipschitz continuous functions. Lastly, we only consider pairwise interactions between the variables; a natural 
generalization would be to consider a model that can include components which are at most $m$-variate. The goal would then be to 
query $f$, in order to identify all interaction terms.

\paragraph{Acknowledgments.} This research was supported in part by SNSF grant CRSII$2$\_$147633$ 
and by The Alan Turing Institute under the EPSRC grant EP/N$510129/1$. This work was mostly done while H.T was 
affiliated to the Department of Computer Science, ETH Z\"urich. 
H.T would like to thank: Yuxin Chen for helpful discussions related to the recovery of sparse symmetric matrices 
in Section \ref{subsec:noiseless_overlap_set_est_alt}; Jan Vybiral for helpful discussions related to bounded orthonormal systems 
in Section \ref{sec:discuss}. The authors would like to thank the anonymous reviewers for 
helpful comments and suggestions that greatly helped to improve a preliminary version of the manuscript.
\bibliographystyle{plain}
\bibliography{SPAM}
\appendix
%
%
\section{Model uniqueness} \label{sec:anova_uniq_rep}

We show here that the model representation \eqref{eq:unique_mod_rep} is a unique representation for $f$ of the form \eqref{eq:gspam_form}.
We first note that any measurable $f:[-1,1]^{\dimn} \rightarrow \matR$ admits a unique ANOVA decomposition \cite{Gu02,wabha03} of the form:
\begin{equation}
f(x_1,\dots,x_{\dimn}) = c + \sum_{\alpha} f_{\alpha}(x_{\alpha}) + \sum_{\alpha < \beta} f_{\alpha\beta} + \sum_{\alpha < \beta < \gamma} f_{\alpha\beta\gamma} + \cdots 
\end{equation}
Indeed, for any probability measure $\mu_{\alpha}$ on $[-1,1]$, let $\calE_{\alpha}$ denote the averaging operator, defined as
\begin{equation}
\calE_{\alpha}(f)(\vecx) := \int_{[-1,1]} f(x_1,\dots,x_{\dimn}) d\mu_{\alpha}. 
\end{equation}
Then the components of the model can be written as: $\hemant{c = (\prod_{\alpha} \calE_{\alpha}) f}$, $f_{\alpha} = (I-\calE_{\alpha})\prod_{\beta \neq \alpha} \calE_{\beta} f$, 
$f_{\alpha\beta} = ((I-\calE_{\alpha})(I-\calE_{\beta})\prod_{\gamma \neq \alpha,\beta}\calE_{\gamma}) f$, and so on.
For our purpose, $\mu_{\alpha}$ is taken to be the \hemant{uniform probability} measure on $[-1,1]$. Given this, we now find the ANOVA decomposition of 
$f$ defined in \eqref{eq:gspam_form}. As a sanity check, let us verify that $f_{\alpha\beta\gamma} \equiv 0$ for all $\alpha < \beta < \gamma$. 
Indeed if $p \in \univsupp$, then at least two of $\alpha < \beta < \gamma$ will not be equal to $p$. Similarly for any $\lpair \in \bivsupp$, 
at least one of $\alpha,\beta,\gamma$ will not be equal to $l$ and $\lp$. This implies $f_{\alpha\beta\gamma} \equiv 0$. The same reasoning trivially applies for 
high order components of the ANOVA decomposition. 

That $c = \expec[f] = \sum_{p \in \univsupp} \expec_p[\phi_p] + \sum_{\lpair \in \bivsupp} \expec_{\lpair}[\phi_{\lpair}]$ is readily seen. 
Next, we have that 
\begin{equation} \label{eq:first_ord_univ}
(I-\calE_{\alpha})\prod_{\beta \neq \alpha} \calE_{\beta} \phi_p = \left\{
\begin{array}{rl}
0 \quad ; & \alpha \neq p, \\
\phi_p - \expec_p[\phi_p] \quad ; & \alpha = p
\end{array} \right\}; \quad p \in \univsupp.
\end{equation}
\begin{equation} \label{eq:first_ord_biv}
(I-\calE_{\alpha})\prod_{\beta \neq \alpha} \calE_{\beta} \phi_{\lpair} = \left\{
\begin{array}{rl}
\expec_{\lp}[\phi_{\lpair}] - \expec_{\lpair}[\phi_{\lpair}] \quad ; & \alpha = l, \\
\expec_{l}[\phi_{\lpair}] - \expec_{\lpair}[\phi_{\lpair}] \quad ; & \alpha = \lp, \\
0 \quad ; & \alpha \neq l,\lp,
\end{array} \right\}; \quad \lpair \in \bivsupp.
\end{equation}
\eqref{eq:first_ord_univ}, \eqref{eq:first_ord_biv} give us the first order components of $\phi_p, \phi_{\lpair}$ respectively.
\hemant{One can next verify using the same arguments as earlier that for any $\alpha < \beta$}: 
\begin{equation} \label{eq:sec_ord_univ}
(I-\calE_{\alpha})(I-\calE_{\beta})\prod_{\gamma \neq \alpha,\beta} \calE_{\gamma} \phi_p = 0; \quad \forall p \in \univsupp. 
\end{equation}
Lastly, we have for any $\alpha < \beta$ that the corresponding second order component of $\phi_{\lpair}$ is given by:
\begin{equation} \label{eq:sec_ord_biv}
(I-\calE_{\alpha})(I-\calE_{\beta})\prod_{\gamma \neq \alpha,\beta} \calE_{\gamma} \phi_{\lpair} = \left\{
\begin{array}{rl}
\phi_{\lpair} - \expec_l[\phi_{\lpair}] \\ - \expec_{\lp}[\phi_{\lpair}] + \expec_{\lpair}[\phi_{\lpair}] \quad ; & \alpha = l, \beta = \lp, \\
0 \quad ; & \text{otherwise} 
\end{array} \right\}; \quad \lpair \in \bivsupp.
\end{equation}
We now make the following observations regarding the variables in $\univsupp \cap \bivsuppvar$.
\begin{enumerate}
\item For each $l \in \univsupp \cap \bivsuppvar$ such that: $\degree(l) = 1$, and $\lpair \in \bivsupp$, we can simply merge $\phi_l$ 
with $\phi_{\lpair}$. Thus $l$ is no longer in $\univsupp$. 

\item For each $l \in \univsupp \cap \bivsuppvar$ such that: $\degree(l) > 1$, we can add the first order component for $\phi_l$ 
with the total first order component corresponding to all $\phi_{\lpair}$'s and $\phi_{\lpairi}$'s. 
Hence again, $l$ will no longer be in $\univsupp$. 
\end{enumerate}
Therefore all $q \in \univsupp \cap \bivsuppvar$ can essentially be merged with $\bivsupp$. Keeping this re-arrangement 
in mind, we can to begin with, assume in \eqref{eq:gspam_form} that $\univsupp \cap \bivsuppvar = \emptyset$.
Then with the help of \eqref{eq:first_ord_univ}, \eqref{eq:first_ord_biv}, \eqref{eq:sec_ord_univ}, \eqref{eq:sec_ord_biv}, 
we have that any $f$ of the form \eqref{eq:gspam_form} (with $\univsupp \cap \bivsuppvar = \emptyset$), can be uniquely written as:
\begin{equation}
f(x_1,\dots,x_d) = c + \sum_{p \in \univsupp}\phitil_{p} (x_p) + \sum_{\lpair \in \bivsupp} \phitil_{\lpair} \xlpair + 
\sum_{q \in \bivsuppvar: \degree(q) > 1} \phitil_{q} (x_q); \quad \hemant{\univsupp \cap \bivsuppvar} = \emptyset,
\end{equation}
where
\begin{align}
c &= \sum_{p \in \univsupp} \expec_p[\phi_p] + \sum_{\lpair \in \bivsupp} \expec_{\lpair}[\phi_{\lpair}], \label{eq:mod_mean} \\
\phitil_{p} &= \phi_p - \expec_p[\phi_p]; \quad \forall p \in \univsupp, \label{eq:mod_s1}
\end{align}
\begin{equation} \label{eq:mod_s2_biv}
\phitil_{\lpair} = \left\{
\begin{array}{rl}
\phi_{\lpair} - \expec_{\lpair}[\phi_{\lpair}] ; & \degree(l), \degree(l^{\prime}) = 1, \\
\phi_{\lpair} - \expec_{l}[\phi_{\lpair}] ; &  \degree(l) = 1, \degree(l^{\prime}) > 1, \\
\phi_{\lpair} - \expec_{l^{\prime}}[\phi_{\lpair}] ; & \degree(l) > 1, \degree(l^{\prime}) = 1, \\
\phi_{\lpair} - \expec_{l}[\phi_{\lpair}] - \expec_{l^{\prime}}[\phi_{\lpair}] + \expec_{\lpair}[\phi_{\lpair}] ; & \degree(l) > 1, \degree(l^{\prime}) > 1,
\end{array} \right. 
\end{equation}
\begin{align} 
\text{and} \quad \phitil_{q} &= \sum_{q^{\prime}: (q,q^{\prime}) \in \bivsupp} (\expec_{q^{\prime}}[\phi_{(q,q^{\prime})}] - 
\expec_{q,q^{\prime}}[\phi_{(q,q^{\prime})}]) \nonumber \\
&+ \sum_{q^{\prime}: (q^{\prime},q) \in \bivsupp} (\expec_{q^{\prime}}[\phi_{(q^{\prime},q)}] - 
\expec_{q^{\prime},q}[\phi_{(q^{\prime},q)}]); \quad \forall q \in \bivsuppvar: \degree(q) > 1. \label{eq:mod_s2_uni}
\end{align}
%


\section{Real roots of a cubic equation in trigonometric form} \label{sec:real_roots_cub}
{\hemantt
Before proceeding with the proofs, we briefly recall the conditions under which a cubic equation 
possesses real roots, along with expressions for the same. The material in this section is taken from 
 \cite[Chapter $18$ (Secs. 191,192)]{HigherMath64}. To begin with, given any cubic equation: 

\begin{equation} \label{eq:orig_cub_eq}
y^3 + a_1y^2 +a_2y + a_3 = 0, 
\end{equation}

one can make the substitution $x = y - (a_1/3)$ to change \eqref{eq:orig_cub_eq} to the form: 

\begin{equation} \label{eq:trans_cub_eq}
x^3 + px + q = 0 \quad \text{where} \quad p = a_2 - \frac{a_1^2}{3} \ \text{and} \ q = \frac{2a_1^3}{27} - \frac{a_1a_2}{3} + a_3. 
\end{equation}

If $p,q$ are real (which is the case if $a_1,a_2,a_3$ are real), then \eqref{eq:trans_cub_eq} 
has three real and distinct roots if its discriminant: $(q^2/4) + (p^3/27) < 0$. Denoting 

\begin{equation}
r = \sqrt{-\frac{p^3}{27}}, \quad \cos \phi = -\frac{q}{2r}, 
\end{equation} 

we then have that the real roots of \eqref{eq:trans_cub_eq} are given by 

\begin{equation} \label{eq:cub_trans_roots}
x = 2 \sqrt[3]{r} \cos \frac{\phi + 2j\pi}{3} = 2 \sqrt{-\frac{p}{3}} \cos \frac{\phi + 2j\pi}{3}; \quad j = 0,1,2. 
\end{equation}

Consequently, the roots of \eqref{eq:orig_cub_eq} are then given by: 

\begin{equation} \label{eq:cub_orig_roots}
y = 2 \sqrt{-\frac{p}{3}} \cos \frac{\phi + 2j\pi}{3} - \frac{a_1}{3}; \quad j = 0,1,2. 
\end{equation}
}

\section{Proofs for Section \ref{sec:algo_nonoverlap}} \label{sec:proofs_no_overlap}
\subsection{Proof of Lemma \ref{lem:rec_act_set}} \label{subsec:proof_nonover_estact}
Recall that for $\vecx \in \baseset$, we recover a stable approximation $\est{\grad} f(\vecx)$ to $\grad f(\vecx)$ via 
$\ell_1$ minimization \cite{Candes2006,Donoho2006}:
\begin{equation}
\est{\grad} f(\vecx) = \triangle(\vecy) := \argmin{\vecy = \matV\vecz} \norm{\vecz}_1. \label{eq:l1_min_dec}
\end{equation}
Applying Theorem \ref{thm:sparse_recon_bound} to our setting yields the following Corollary.
\begin{corollary} \label{corr:rec_grad_vecs}
There exist constants $c_3^{\prime} \geq 1$ and $C,c_1^{\prime} > 0$ such that for $\numdirec$ 
satisfying $c_3^{\prime} \totsparsity \log(\dimn/\totsparsity)  < \numdirec < \dimn/(\log 6)^2$ we have with probability at least
$1-e^{-c_1^{\prime}\numdirec} - e^{-\sqrt{\numdirec \dimn}}$ that $\est{\grad} f(\vecx)$ 
satisfies for all $\vecx \in \baseset$:
\begin{equation} \label{eq:grad_est_fin_bd}
\norm{\est{\grad} f(\vecx) - \grad f(\vecx)}_2 \leq \frac{2 C \gradstep^2 \smconst_3 \totsparsity}{3\numdirec},
\end{equation}
where $\smconst_3 > 0$ is the constant defined in Assumption \ref{assum:smooth}. 
\end{corollary}
%
\begin{proof}
Since $\grad f(\vecx)$ is at most $\totsparsity$-sparse for any $\vecx \in \matR^{\dimn}$ we immediately have from \eqref{eq:sparse_recon_err}
that 
\begin{equation} \label{eq:grad_est_bd}
\norm{\est{\grad} f(\vecx) - \grad f(\vecx)}_2 \leq C \max\set{\norm{\taynoisvec}_2, \sqrt{\log \dimn}\norm{\taynoisvec}_{\infty}}; \quad 
\forall \vecx \in \baseset. 
\end{equation}
It remains to bound $\norm{\taynoisvec}_2, \norm{\taynoisvec}_{\infty}$. To this end, recall that $\taynoisvec = [\taynoissca_1 \dots \taynoissca_{\numdirec}]$
where $\taynoissca_j = \frac{\thirdtayrem_3(\zeta_{j}) - \thirdtayrem_3(\zeta^{\prime}_{j})}{2\gradstep}$, for some $\zeta_j,\zeta^{\prime}_j \in \matR^{\dimn}$.
Here $\thirdtayrem_3(\zeta)$ denotes the third order Taylor remainder term. By taking the structure of $f$ into account, we can uniformly bound 
$\abs{\thirdtayrem_3(\zeta_j)}$ as follows (so the same bound holds for $\abs{\thirdtayrem_3(\zeta^{\prime}_j)}$).
\begin{align}
\abs{\thirdtayrem_3(\zeta_j)} &= \frac{\gradstep^3}{6} |\sum_{p \in \univsupp} \partial_p^3 \phi_p(\zeta_{j,p}) v_p^3  + 
\sum_{\lpair \in \bivsupp} ( \partial_l^3 \phi_{\lpair}(\zeta_{j,l}, \zeta_{j,{l^{\prime}}}) v_l^3 + \partial_{l^{\prime}}^3 \phi_{\lpair}(\zeta_{j,l}, \zeta_{j,{l^{\prime}}}) v_{l^{\prime}}^3) \\
&+ \sum_{\lpair \in \bivsupp} (3\partial_l \partial_{l^{\prime}}^2 \phi_{\lpair}(\zeta_{j,l}, \zeta_{j,{l^{\prime}}}) v_l v_{l^{\prime}}^2 + 
3\partial_l^2 \partial_{l^{\prime}} \phi_{\lpair}(\zeta_{j,l}, \zeta_{j,{l^{\prime}}}) v_l^2 v_{l^{\prime}})|, \nonumber \\
&\leq \frac{\gradstep^3}{6} \left[\left(\frac{1}{\sqrt{\numdirec}}\right)^3 \univsparsity \smconst_3 + \left(\frac{1}{\sqrt{\numdirec}}\right)^3 \bivsparsity (2\smconst_3)
+ \left(\frac{1}{\sqrt{\numdirec}}\right)^3 \bivsparsity (6\smconst_3)\right], \\
&= \frac{\gradstep^3\smconst_3 (\univsparsity + 8\bivsparsity)}{6\numdirec^{3/2}}.
\end{align}
Using the fact that $\univsparsity + 8\bivsparsity \leq 4(\univsparsity + 2\bivsparsity) = 4\totsparsity$, we consequently obtain 
\begin{align}
 \norm{\taynoisvec}_{\infty} &= \max_{j} \abs{\taynoissca_j} \leq \frac{\gradstep^2\smconst_3 (\univsparsity + 8\bivsparsity)}{6\numdirec^{3/2}} \leq 
\frac{2 \gradstep^2\smconst_3\totsparsity}{3\numdirec^{3/2}}, \label{eq:nois_inf_bd}\\
\text{and} \quad \norm{\taynoisvec}_2 &\leq \sqrt{\numdirec} \norm{\taynoisvec}_{\infty} \leq \frac{2 \gradstep^2\smconst_3\totsparsity}{3 \numdirec}. \label{eq:nois_l2_bd}
\end{align}
Using \eqref{eq:nois_inf_bd},\eqref{eq:nois_l2_bd} in \eqref{eq:grad_est_bd}, we finally obtain for the stated choice of $\numdirec$ (cf. Remark \ref{rem:l1min_samp_bd}), 
the bound in \eqref{eq:grad_est_fin_bd}.
\end{proof}

Let us denote $\derivsamperr = \frac{2 C \gradstep^2 \smconst_3 \totsparsity}{3\numdirec}$. In order to prove the lemma, 
we first observe that \eqref{eq:grad_est_fin_bd} trivially implies that
\begin{equation} \label{eq:intbd_est_parderiv}
 \est{\partial_q} f(\vecx) \in [\partial_q f(\vecx) - \derivsamperr,  \partial_q f(\vecx) + \derivsamperr]; \quad q=1,\dots,\dimn.
\end{equation}
Now, in case $q \notin \univsupp \cup \bivsuppvar$, then $\partial_q f(\vecx) = 0$ $\forall \vecx \in \matR^{\dimn}$, meaning that 
$\est{\partial_q} f(\vecx) \in [-\derivsamperr,\derivsamperr]$. If $\numcen \geq \critintmeas_1^{-1}$ then 
for every $q \in \univsupp \cup \bivsuppvar$, $\exists \hashfn \in \twohashfam$ and at least one $\vecx \in \baseset(\hashfn)$,
so that $\abs{\partial_q f(\vecx)} > \idenconst_1$. Indeed, this follows from the definition of $\twohashfam$, and by construction of 
$\baseset(\hashfn)$ for $\hashfn \in \twohashfam$. Furthermore, for such $\vecx$, we have from \eqref{eq:intbd_est_parderiv} 
that $\abs{\est{\partial_q} f(\vecx)} \geq \idenconst_1 - \derivsamperr$. Therefore if $\derivsamperr < \frac{\idenconst_1}{2}$ holds, 
then clearly we would have $\abs{\est{\partial_q} f(\vecx)} > \frac{\idenconst_1}{2} > \derivsamperr$, meaning that we will be able to identify $q$.  

Lastly, we observe that the condition $\derivsamperr < \frac{\idenconst_1}{2}$ translates to an equivalent condition on the step size $\gradstep$ as follows.
\begin{equation}
 \derivsamperr < \frac{\idenconst_1}{2} \Leftrightarrow \frac{2 C \gradstep^2 \smconst_3 \totsparsity}{3\numdirec} < \frac{\idenconst_1}{2} \Leftrightarrow 
\gradstep < \left( \frac{3\idenconst_1 \numdirec}{4 C \smconst_3 \totsparsity} \right)^{1/2}
\end{equation}
%

\subsection{Proof of Lemma \ref{lem:est_ind_sets}} \label{subsec:proof_nonover_est_ind}
We proceed by first bounding the error term that arises in the estimation of $\partial_i g(\vecx)$. 
As $g$ is $\calC^3$ smooth, consider the Taylor's expansion of $g$ at $\vecx$, along $\canvec_1(i),-\canvec_1(i) \in \matR^{\totsparsity}$, with step size
$\pardevstep > 0$. For some $\zeta = \vecx + \theta \canvec_1(i)$, $\zeta^{\prime} = \vecx - \theta^{\prime} \canvec_1(i)$ with $\theta,\theta^{\prime} \in (0,\pardevstep)$,
we obtain the identities:
\begin{align}
g(\vecx + \pardevstep \canvec_1(i)) &= g(\vecx) + \pardevstep \dotprod{\canvec_1(i)}{\grad g(\vecx)} + \frac{\pardevstep^2}{2}\canvec_1(i)^T \hess g(\vecx) \canvec_1(i)
+ \thirdtayrem_3(\zeta), \\
g(\vecx - \pardevstep \canvec_1(i)) &= g(\vecx) - \pardevstep \dotprod{\canvec_1(i)}{\grad g(\vecx)} + \frac{\pardevstep^2}{2}\canvec_1(i)^T \hess g(\vecx) \canvec_1(i)
+ \thirdtayrem_3(\zeta^{\prime}),
\end{align}
with $\thirdtayrem_3(\zeta),\thirdtayrem_3(\zeta^{\prime}) = O(\pardevstep^3)$ being the third order remainder terms. Subtracting the above leads to the following identity.
\begin{equation} \label{eq:nonover_lin_eq}
\underbrace{\frac{g(\vecx + \pardevstep \canvec_1(i)) - g(\vecx - \pardevstep \canvec_1(i))}{2\pardevstep}}_{\est{\partial_i} g(\vecx)} = 
\underbrace{\dotprod{\canvec_1(i)}{\grad g(\vecx)}}_{\partial_i g(\vecx)} + 
\underbrace{\frac{\thirdtayrem_3(\zeta) - \thirdtayrem_3(\zeta^{\prime})}{2\pardevstep}}_{\pardevestnois_i(\vecx,\pardevstep) = O(\pardevstep^2)}
\end{equation}
We now uniformly bound $\abs{\thirdtayrem_3(\zeta)}$, so the same bound holds for $\abs{\thirdtayrem_3(\zeta^{\prime})}$. Due to the structure of $g$, we have that
\begin{align}
\abs{\thirdtayrem_3(\zeta)} &= \frac{\pardevstep^3}{6}|\sum_{p \in \univsupp} \partial_p^3 \phi_p(\zeta_p) (\canvec_1(i))^3_p + 
\sum_{\lpair \in \bivsupp}(\partial_l^3 \phi_{\lpair}(\zeta_l,\zeta_{l^{\prime}}) (\canvec_1(i))^3_l  + \partial_{l^{\prime}}^3 \phi_{\lpair}(\zeta_l,\zeta_{l^{\prime}}) (\canvec_1(i))^3_{l^{\prime}}) \\
&+ \sum_{\lpair \in \bivsupp}(3 \partial_l^2 \partial_{l^{\prime}} \phi_{\lpair}(\zeta_l,\zeta_{l^{\prime}}) (\canvec_1(i))^2_l (\canvec_1(i))_{l^{\prime}} 
+ 3\partial_{l^{\prime}}^2 \partial_l \phi_{\lpair}(\zeta_l,\zeta_{l^{\prime}}) (\canvec_1(i))^2_{l^{\prime}}(\canvec_1(i))_l)| \\
&= \left\{
\begin{array}{rl}
\frac{\pardevstep^3}{6} \abs{\partial_i^3 \phi_i(\zeta_i)} ; \quad & i \in \univsupp, \\
\frac{\pardevstep^3}{6} \abs{\partial_i^3 \phi_{i,j} (\zeta_i,\zeta_j)} ; \quad & i \in \bivsuppvar, (i,j) \in \bivsupp, \\
\frac{\pardevstep^3}{6} \abs{\partial_i^3 \phi_{j,i} (\zeta_j,\zeta_i)} ; \quad & i \in \bivsuppvar, (j,i) \in \bivsupp
\end{array} \right . \\
&\leq \frac{\pardevstep^3 \smconst_3}{6}.
\end{align}
The above consequently implies that $\abs{\pardevestnois_i(\vecx,\pardevstep)} \leq \frac{\pardevstep^2 \smconst_3}{6}$. This in turn means, for 
any $\vecv \in \matR^{\totsparsity}, \hessstep > 0$, that
\begin{equation} \label{eq:bd_totnois_cendiff}
\abs{\frac{\pardevestnois_i(\vecx + \hessstep\vecv,\pardevstep) - \pardevestnois_i(\vecx,\pardevstep)}{\hessstep}} \leq \frac{\pardevstep^2 \smconst_3}{3\hessstep}.
\end{equation}
Thus we have a uniform bound on the magnitude of one of the contributors of the error term in \eqref{eq:pardev_lin_eq}. We can bound the magnitude of the other term as
follows. For $\vecv \in \matR^{\totsparsity}$ and $\zeta = \vecx + \theta\vecv$; $\theta \in (0,\hessstep)$, we have
\begin{align}
\vecv^T \hess \partial_i g(\zeta) \vecv =  
\left\{
\begin{array}{rl}
v_i^2 \partial_i^3 \phi_i(\zeta_i); \quad & i \in \univsupp, \\
v_i^2 \partial_i^3 \phi_{i,i^{\prime}} (\zeta_i, \zeta_{i^{\prime}}) + 
v_{i^{\prime}}^2 \partial_{i^{\prime}}^2 \partial_i \phi_{i,i^{\prime}} (\zeta_i, \zeta_{i^{\prime}}) 
+ 2 v_i v_{i^{\prime}} \partial_{i^{\prime}} \partial_i^2 \phi_{i,i^{\prime}} (\zeta_i, \zeta_{i^{\prime}}); \quad & i \in \bivsuppvar, (i,i^{\prime}) \in \bivsupp
\end{array} \right .
\end{align}
Since in our scheme we employ only $\vecv \in \set{0,1}^{\totsparsity}$, this leads to the following uniform bound. 
\begin{equation} \label{eq:nonover_sec_ord_bd}
\abs{\frac{\hessstep}{2} \vecv^T \hess \partial_i g(\zeta) \vecv} \leq 4 \smconst_3 \frac{\hessstep}{2} = 2 \hessstep \smconst_3; \quad \forall i \in \univsupp \cup \bivsuppvar.
\end{equation}
Denoting by $\hesssamperr$, the upper bound on the magnitude of the error term in \eqref{eq:pardev_lin_eq}, we thus obtain:
\begin{equation}
\hesssamperr = 2 \hessstep \smconst_3 + \frac{\pardevstep^2 \smconst_3}{3\hessstep}.
\end{equation}
Now in case $i \in \univsupp$, we have $\dotprod{\grad \partial_i g(\vecx)}{\vecv_0(i)} = 0$, $\forall \vecx \in \matR^{\totsparsity}$. 
This in turn implies that 
\begin{equation}
\abs{\frac{\est{\partial}_i g(\vecx + \hessstep \vecv_0(i)) - \est{\partial}_i g(\vecx)}{\hessstep}} \leq \hesssamperr; \quad \forall \vecx \in \matR^{\totsparsity}.
\end{equation}
If $i \in \bivsuppvar$ with $(i,i^{\prime}) \in \bivsupp$, then 
\begin{equation}
\dotprod{\grad \partial_i g(\vecx)}{\vecv_0(i)} = \partial_{i} \partial_{i^{\prime}} g(\vecx) = \partial_{i} \partial_{i^{\prime}} \phi_{(i,i^{\prime})}(x_i, x_{i^{\prime}}).                  
\end{equation}
For the choice $\numcenpair > \critintmeas_2^{-1}$, $\exists \vecx^{*} \in \baseset_i$ such that
$\abs{\partial_{i} \partial_{i^{\prime}} \phi_{(i,i^{\prime})}(x^{*}_i, x^{*}_{i^{\prime}})} > \idenconst_2$. 
This is clear from the construction of $\baseset_i$, and on account of Assumption \ref{assum:actvar_iden}. If we guarantee that
$\hesssamperr < \idenconst_2/2$ holds, then consequently
\begin{equation}
\abs{\frac{\est{\partial}_i g(\vecx^{*} + \hessstep \vecv_0(i)) - \est{\partial}_i g(\vecx^{*})}{\hessstep}} > \idenconst_2 - \hesssamperr > \hesssamperr
\end{equation}
meaning that the pair $(i,i^{\prime})$ can be identified. Lastly, it is easily verifiable, that the requirement $\hesssamperr < \idenconst_2/2$, equivalently
translates to the stated conditions on $\pardevstep, \gradstep$.

\subsection{Proof of Theorem \ref{thm:gen_no_over_arbnois}} \label{subsec:proof_nonover_arbnoise}
We begin by first establishing the conditions that guarantee $\est{\totsupp} = \totsupp$, and then 
derive conditions that guarantee exact recovery of $\univsupp,\bivsupp$.

\paragraph{Estimation of $\totsupp$.}
We first note that \eqref{eq:cs_form} now changes to $\vecy = \matV\grad f(\vecx) + \taynoisvec +  
\exnoisevec$ where $\exnoise_{j} = (\exnoisep_{j,1} - \exnoisep_{j,2})/(2\gradstep)$ represents the external noise
component, for $j=1,\dots,\numdirec$. Since $\norm{\exnoisevec}_{\infty} \leq \exnoisemag/\gradstep$, therefore 
using the bounds on $\norm{\taynoisvec}_{\infty}$ from Section \ref{subsec:proof_nonover_estact} one can verify that 
\eqref{eq:grad_est_fin_bd} in Corollary \ref{corr:rec_grad_vecs} changes to 

\begin{equation}
\norm{\est{\grad} f(\vecx) - \grad f(\vecx)}_2 \leq C\left(\frac{2 \gradstep^2 \smconst_3 \totsparsity}{3\numdirec} + 
\frac{\exnoisemag\sqrt{\numdirec}}{\gradstep} \right).
\end{equation}

Following the same arguments mentioned in Section \ref{subsec:proof_nonover_estact}, we 
observe that if $\derivsamperr < \idenconst_1/2$ holds, then it implies that $\est{\totsupp} = \totsupp$.  
Now, $\derivsamperr < \idenconst_1/2$ is equivalent to

\begin{align} \label{eq:cub_ineq_nonover}
\underbrace{\frac{2 \gradstep^2 \smconst_3 \totsparsity}{3\numdirec}}_{a\gradstep^2} + 
\underbrace{\frac{\exnoisemag\sqrt{\numdirec}}{\gradstep}}_{\frac{b\exnoisemag}{\gradstep}}  
< \frac{\idenconst_1}{2C} \quad \Leftrightarrow \quad
\gradstep^3 - \hemant{\frac{\idenconst_1}{2aC}\gradstep} + \frac{b\exnoisemag}{a} < 0.
\end{align}

\eqref{eq:cub_ineq_nonover} is a cubic inequality. \hemant{Recall from Section \ref{sec:real_roots_cub} that a cubic equation of the form: $y^3 + py + q = 0$, has $3$ distinct real roots} 
if its discriminant $\frac{p^3}{27} + \frac{q^2}{4} < 0$. Note that for this to be possible, $p$ must be negative, 
which is the case in \eqref{eq:cub_ineq_nonover}. Applying the discriminant condition on \eqref{eq:cub_ineq_nonover} 
leads to 

\begin{align}
-\frac{\idenconst_1^3}{27 \cdot 8 a^3 C^3} + \frac{b^2}{4 a^2}\exnoisemag^2 < 0 \quad 
\Leftrightarrow \quad \exnoisemag < \frac{\idenconst_1^{3/2}}{3 b C\sqrt{6a C}}. 
\end{align}

\hemant{Also, recall from \eqref{eq:cub_trans_roots} that} the $3$ distinct real roots of the cubic equation are then given by:
\begin{equation} 
y_1 = 2\sqrt{-p/3}\cos(\theta/3), \ y_2 = -2\sqrt{-p/3}\cos(\theta/3 + \pi/3), \ y_3 = -2\sqrt{-p/3}\cos(\theta/3 - \pi/3) 
\end{equation}
where $\theta = \cos^{-1}\left(\frac{-q/2}{\sqrt{-p^3/27}}\right)$. 
In particular, if $q > 0$, then one can verify that $y^3 + py + q < 0$ holds if $y \in (y_2,y_1)$. 
Applying this to the cubic equation corresponding to \eqref{eq:cub_ineq_nonover}, 
we consequently obtain: 
\begin{equation}
\gradstep \in \left(2\sqrt{\frac{\idenconst_1}{6aC}}\cos(\theta_1/3 - 2\pi/3), 
2\sqrt{\frac{\idenconst_1}{6aC}}\cos(\theta_1/3)\right). 
\end{equation}
where $\theta_1 = \cos^{-1}(-\exnoisemag/\exnoisemag_1)$.

\paragraph{Estimation of $\univsupp, \bivsupp$.}
On account of noise, we first note that \eqref{eq:nonover_lin_eq} changes to 

\begin{equation}
\underbrace{\frac{g(\vecx + \pardevstep \canvec_1(i)) - g(\vecx - \pardevstep \canvec_1(i))}{2\pardevstep}}_{\est{\partial_i} g(\vecx)} = 
\underbrace{\dotprod{\canvec_1(i)}{\grad g(\vecx)}}_{\partial_i g(\vecx)} + 
\underbrace{\frac{\thirdtayrem_3(\zeta) - \thirdtayrem_3(\zeta^{\prime})}{2\pardevstep}}_{\pardevestnois_i(\vecx,\pardevstep) = O(\pardevstep^2)} 
+ \underbrace{\frac{\exnoisep_{i,1} - \exnoisep_{i,2}}{2 \pardevstep}}_{\exnoise_i(\vecx,\pardevstep)}. 
\end{equation}

This in turn results in \eqref{eq:pardev_lin_eq} changing to 

\begin{equation} \label{eq:pardev_lin_eq_noise}
\frac{\est{\partial_i} g(\vecx + \hessstep\vecv) - \est{\partial_i} g(\vecx)}{\hessstep} = \dotprod{\grad \partial_i g(\vecx)}{\vecv} + 
\underbrace{\frac{\hessstep}{2} \vecv^T \hess \partial_i g(\zeta_i) \vecv + 
\frac{\pardevestnois_i(\vecx + \hessstep\vecv,\pardevstep) - \pardevestnois_i(\vecx,\pardevstep)}{\hessstep} 
+ \frac{\exnoise_i(\vecx+\hessstep\vecv,\pardevstep) - \exnoise_i(\vecx,\pardevstep)}{\hessstep}}_{\text{Error term}}.
\end{equation}

Using \eqref{eq:bd_totnois_cendiff}, \eqref{eq:nonover_sec_ord_bd} and noting that 
$\abs{(\exnoise_i(\vecx+\hessstep\vecv,\pardevstep) - \exnoise_i(\vecx,\pardevstep))/\hessstep} \leq 
2\exnoisemag/(\pardevstep \hessstep)$, then by denoting $\hesssamperr$ to be an upper bound on the magnitude 
of the error term in \eqref{eq:pardev_lin_eq_noise}, we have that 
$\hesssamperr = 2 \hessstep \smconst_3 + \frac{\pardevstep^2 \smconst_3}{3\hessstep} + 
\frac{2\exnoisemag}{\pardevstep\hessstep}$. Following the same argument as in 
Section \ref{subsec:proof_nonover_est_ind}, we have that $\hesssamperr < \idenconst_2/2$ implies 
$\est{\univsupp} = \univsupp$ and $\est{\bivsupp} = \bivsupp$. The condition $\hesssamperr < \idenconst_2/2$ 
is equivalent to 

\begin{align}
2 \hessstep \smconst_3 + \frac{\pardevstep^2 \smconst_3}{3\hessstep} + 
\frac{2\exnoisemag}{\pardevstep\hessstep} &< \frac{\idenconst_2}{2} \\ 
\Leftrightarrow 6\smconst_3 \pardevstep \hessstep^2 - \frac{3\idenconst_2}{2} \pardevstep\hessstep 
+ (\pardevstep^3\smconst_3 + 6\exnoisemag) &< 0. \label{eq:nonover_quad_ineq}
\end{align}

Solving \eqref{eq:nonover_quad_ineq} in terms of $\hessstep$ leads to 

\begin{align}
 \hessstep &\in \left(\frac{\frac{3\idenconst_2}{2}\pardevstep - \sqrt{\frac{9\idenconst_2^2}{4}\pardevstep^2 - 24\pardevstep\smconst_3(\pardevstep^3\smconst_3 + 6\exnoisemag)}}{12\smconst_3\pardevstep}, 
\frac{\frac{3\idenconst_2}{2}\pardevstep + \sqrt{\frac{9\idenconst_2^2}{4}\pardevstep^2 - 24\pardevstep\smconst_3(\pardevstep^3\smconst_3 + 6\exnoisemag)}}{12\smconst_3\pardevstep}\right) \\
 \Leftrightarrow \hessstep &\in \left(\frac{\idenconst_2 - \sqrt{\idenconst_2^2 - \frac{32}{3\pardevstep}\smconst_3(\pardevstep^3\smconst_3 + 6\exnoisemag)}}{8\smconst_3}, 
\frac{\idenconst_2 + \sqrt{\idenconst_2^2 - \frac{32}{3\pardevstep}\smconst_3(\pardevstep^3\smconst_3 + 6\exnoisemag)}}{8\smconst_3}\right) 
\end{align}

Now in order for the above condition on $\hessstep$ to be meaningful, we require 

\begin{equation} \label{eq:nonover_ind_est_cub}
 \idenconst_2^2 - \frac{32}{3\pardevstep}\smconst_3(\pardevstep^3\smconst_3 + 6\exnoisemag) > 0 \quad 
\Leftrightarrow \quad \pardevstep^3 - \frac{3\idenconst_2^2}{32\smconst_3^2} \pardevstep + \frac{6\exnoisemag}{\smconst_3} < 0.
\end{equation}

Since \eqref{eq:nonover_ind_est_cub} is a cubic inequality, therefore by following the steps described earlier 
(for identification of $\totsupp$), one readily obtains the stated conditions on $\hessstep, \exnoisemag$ and $\pardevstep$.

\subsection{Proof of Theorem \ref{thm:gen_no_over_gauss}} \label{subsec:proof_thm_nonover_stoch}

We first derive conditions for estimating $\totsupp$, and then for estimating $\univsupp, \bivsupp$.
\paragraph{Estimating $\totsupp$.} Upon resampling $N_1$ times and averaging, we have for the
noise vector $\vecz \in \matR^{\numdirec}$ that 
\begin{equation}
\vecz = \left[\frac{(\exnoisep_{1,1} - \exnoisep_{1,2})}{2\gradstep} \cdots \frac{(\exnoisep_{\numdirec,1} - \exnoisep_{\numdirec,2})}{2\gradstep} \right], 
\end{equation}
where $\exnoisep_{j,1}, \exnoisep_{j,2} \sim \calN(0,\sigma^2/N_1)$ are i.i.d. Our aim is to guarantee that 
$\abs{\exnoisep_{j,1} - \exnoisep_{j,2}} < 2\exnoisemag$ holds $\forall j=1,\dots,\numdirec$, and across all points where 
$\grad f$ is estimated. Indeed, we then obtain a bounded noise model and can simply use the analysis for the setting 
of arbitrary bounded noise. 

To this end, note that $\exnoisep_{j,1} - \exnoisep_{j,2} \sim \calN(0,\frac{2\sigma^2}{N_1})$. 
It can be shown that for any $X \sim \calN(0,1)$ we have:
\begin{equation}
\prob(\abs{X} > t) \leq {\hemantt 2 e^{-t^2/2}}, \quad \forall t > 0.
\end{equation}
Since $\exnoisep_{j,1} - \exnoisep_{j,2} = \sigma\sqrt{\frac{2}{N_1}} X$ therefore for any $\exnoisemag > 0$ we have that:
\begin{align}
\prob(\abs{\exnoisep_{j,1} - \exnoisep_{j,2}} > 2\exnoisemag) &= \prob\left(\abs{X} > \frac{2\exnoisemag}{\sigma}\sqrt{\frac{N_1}{2}}\right) \\
&\leq {\hemantt 2} \exp\left(-\frac{\exnoisemag^2 N_1}{\sigma^2}\right).
\end{align}
Now to estimate $\grad f(\vecx)$ we have $\numdirec$ many ``difference'' terms: $\exnoisep_{j,1} - \exnoisep_{j,2}$. 
As this is done for each $\vecx \in \baseset$, therefore we have a total of 
$\numdirec(2\numcen+1)^2\abs{\twohashfam}$ many difference terms. 
Taking a union bound over all of them, we have for any $p_1 \in (0,1)$ that the choice 
{\hemantt $N_1 > \frac{\sigma^2}{\exnoisemag^2} \log (\frac{2}{p_1}\numdirec(2\numcen+1)^2\abs{\twohashfam})$} 
implies that the magnitudes of all difference terms are bounded by $2\exnoisemag$, with probability at least 
$1-p_1$. 

\paragraph{Estimating $\univsupp, \bivsupp$.} In this case, we resample each query $N_2$ times and average -- 
therefore the variance of the noise terms gets scaled by $N_2$. 
Note that for each $i \in \totsupp$ and $\vecx \in \baseset_i$, we have two difference terms 
corresponding to external noise -- one corresponding to $\est{\partial_i} g(\vecx)$ and the other corresponding 
to $\est{\partial_i} g(\vecx + \hessstep \vecv)$. This means that in total we have at most 
$\totsparsity(2 {\numcenpair}^{2} + \ceil{\log \totsparsity})$ many difference terms arising.

Therefore, taking a union bound over all of them, we have for any $p_2 \in (0,1)$ that the choice 
$N_2 > \frac{\sigma^2}{{\exnoisemagp}^2} {\hemantt \log(\frac{2 \totsparsity(2 {\numcenpair}^{2} + \ceil{\log \totsparsity})}{p_2})}$ 
implies that the magnitudes of all difference terms are bounded by $2\exnoisemagp$, with probability at least 
$1-p_2$. 

\section{Proofs for Section \ref{sec:algo_gen_overlap}} \label{sec:proofs_gen_overlap}
\subsection{Proof of Theorem \ref{thm:gen_overlap}} \label{subsec:proof_thm_genover}
The proof is divided into the following steps.
\paragraph{Bounding the $\hessestnoisb$ term.}
The proof of this step is similar to that of Corollary \ref{corr:rec_grad_vecs}. Since $\grad f(\vecx)$ is at most $\totsparsity$ sparse, therefore
for any $\vecx \in \matR^{\dimn}$ we immediately have from Theorem \ref{thm:sparse_recon_bound}, \eqref{eq:sparse_recon_err}, the following. 
$\exists C_1, c_4^{\prime}> 0, c_1^{\prime} \geq 1$ such that for 
$c_1^{\prime} \totsparsity \log(\frac{\dimn}{\totsparsity}) < \numdirec < \frac{\dimn}{(\log 6)^2}$ we have 
with probability at least $1 - e^{-c_4^{\prime}\numdirec} - e^{-\sqrt{\numdirec\dimn}}$ that 
\begin{equation} \label{eq:grad_est_bd_gen}
\norm{\est{\grad} f(\vecx) - \grad f(\vecx)}_2 \leq C_1 \max\set{\norm{\taynoisvec}_2, \sqrt{\log \dimn}\norm{\taynoisvec}_{\infty}}. 
\end{equation}
Recall from \eqref{eq:taylor_exp_f_1} that $\taynoisvec = [\taynoissca_1 \dots \taynoissca_{\numdirec}]$
where $\taynoissca_j = \frac{\thirdtayrem_3(\zeta_{j}) - \thirdtayrem_3(\zeta^{\prime}_{j})}{2\gradstep}$, 
for some $\zeta_j,\zeta^{\prime}_j \in \matR^{\dimn}$.
Here $\thirdtayrem_3(\zeta)$ denotes the third order Taylor remainder terms of $f$. By taking the structure of $f$ into account, 
we can uniformly bound 
$\abs{\thirdtayrem_3(\zeta_j)}$ as follows (so the same bound holds for $\abs{\thirdtayrem_3(\zeta^{\prime}_j)}$).
Let us define $\numdegree := \abs{\set{q \in \bivsuppvar: \degree(q) > 1}}$, 
to be the number of variables in $\bivsuppvar$, with degree greater than one. 
\begin{align}
\abs{\thirdtayrem_3(\zeta_j)} &= \frac{\gradstep^3}{6} |\sum_{p \in \univsupp} \partial_p^3 \phi_p(\zeta_{j,p}) v_p^3  + 
\sum_{\lpair \in \bivsupp} ( \partial_l^3 \phi_{\lpair}(\zeta_{j,l}, \zeta_{j,{l^{\prime}}}) v_l^3 + \partial_{l^{\prime}}^3 
\phi_{\lpair}(\zeta_{j,l}, \zeta_{j,{l^{\prime}}}) v_{l^{\prime}}^3) \nonumber \\
&+ \sum_{\lpair \in \bivsupp} (3\partial_l \partial_{l^{\prime}}^2 \phi_{\lpair}(\zeta_{j,l}, \zeta_{j,{l^{\prime}}}) v_l v_{l^{\prime}}^2 + 
3\partial_l^2 \partial_{l^{\prime}} \phi_{\lpair}(\zeta_{j,l}, \zeta_{j,{l^{\prime}}}) v_l^2 v_{l^{\prime}})
+ \sum_{q \in \bivsuppvar: \degree(q)>1} \partial_q^3 \phi_q(\zeta_{j,q}) v_q^3| \\
&\leq \frac{\gradstep^3}{6} \left(\frac{\univsparsity \smconst_3}{\numdirec^{3/2}} + 
\frac{2\bivsparsity\smconst_3}{\numdirec^{3/2}} + \frac{\alpha\smconst_3}{\numdirec^{3/2}} + \frac{6\bivsparsity\smconst_3}{\numdirec^{3/2}} \right) \\
&= \frac{\gradstep^3}{6}\frac{(\univsparsity + \alpha + 8\bivsparsity)\smconst_3}{\numdirec^{3/2}}.  \label{eq:temp_bd_1}
\end{align}
Using the fact $2\bivsparsity = \sum_{l \in \bivsuppvar: \degree(l) > 1} \degree(l) + (\abs{\bivsuppvar} - \numdegree)$, we can observe that 
$2\bivsparsity \leq \maxdegree\numdegree + (\abs{\bivsuppvar} - \numdegree) = \abs{\bivsuppvar} + (\maxdegree-1)\numdegree$. Plugging this in 
\eqref{eq:temp_bd_1}, and using the fact $\alpha \leq \totsparsity$ (since we do not assume $\alpha$ to be known), we obtain
\begin{align}
\abs{\thirdtayrem_3(\zeta_j)} \leq \frac{\gradstep^3}{6}\frac{(\univsparsity + \numdegree + 4\abs{\bivsuppvar} + 4(\maxdegree-1)\numdegree)\smconst_3}{\numdirec^{3/2}} 
\leq \frac{\gradstep^3(4\totsparsity + (4\maxdegree-3)\numdegree)\smconst_3}{6\numdirec^{3/2}} 
\leq \frac{\gradstep^3((4\maxdegree+1)\totsparsity)\smconst_3}{6\numdirec^{3/2}}. 
\end{align}
This in turn implies that $\norm{\taynoisvec}_{\infty} \leq \frac{\gradstep^2((4\maxdegree+1)\totsparsity)\smconst_3}{6\numdirec^{3/2}}$. 
Using the fact $\norm{\taynoisvec}_{2} \leq \sqrt{\numdirec}\norm{\taynoisvec}_{\infty}$, we thus obtain for the stated choice of $\numdirec$ (cf. Remark \ref{rem:l1min_samp_bd}) that 
\begin{equation} \label{eq:grad_est_over_bd}
\norm{\est{\grad} f(\vecx) - \grad f(\vecx)}_2 \leq \frac{C_1\gradstep^2((4\maxdegree+1)\totsparsity)\smconst_3}{6\numdirec}, 
\quad \forall \vecx \in [-(1+r),1+r]^{\dimn}.
\end{equation}
Recall that $[-(1+r),1+r]^{\dimn}, r > 0$, denotes the enlargement around $[-1,1]^{\dimn}$, in which the smoothness 
properties of $\phi_p, \phi_{\lpair}$ are defined in Section \ref{sec:problem} (as Assumption \ref{assum:spamin_samp_reg}). 
Also recall $\vecw(\vecx) \in \matR^{\dimn}, \hessestnoisb \in \matR^{\numdirecp}$ from \eqref{eq:hessrow_est_linfin}.
Since $\norm{\vecw(\vecx)}_{\infty} \leq \norm{\est{\grad} f(\vecx) - \grad f(\vecx)}_2$, this then implies that 
$\norm{\hessestnoisb}_{\infty} \leq \frac{C_1\gradstep^2((4\maxdegree+1)\totsparsity)\smconst_3}{3\numdirec\hessstep}$.
\paragraph{Bounding the $\hessestnoisa$ term.}
We will bound $\norm{\hessestnoisa}_{\infty}$. To this end, we see from \eqref{eq:hessrow_est_linfin} that it 
suffices to uniformly bound $\abs{{\vecvp}^{T} \hess \partial_q f(\zeta) \vecvp}$, 
over all: $q \in \univsupp \cup \bivsuppvar$, $\vecvp \in \calVp$, $\zeta \in [-(1+r),(1+r)]^{\dimn}$. Note that
\begin{equation}
{\vecvp}^T \hess \partial_q f(\zeta) \vecvp = \sum_{l=1}^{\dimn} {v_l^{\prime}}^2 (\hess \partial_q f)_{l,l} + 
\sum_{i \neq j = 1}^{\dimn} v_i^{\prime}v_j^{\prime} (\hess \partial_q f)_{i,j}.
\end{equation}
We have the following three cases, depending on the type of $q$.
\begin{enumerate}
\item $\mathbf{q \in \univsupp.}$
\begin{equation}
{\vecvp}^T \hess \partial_q f(\zeta) \vecvp = {v^{\prime}_{q}}^2 \partial_q^3 \phi_q(\zeta_q) \Rightarrow 
\abs{{\vecvp}^T \hess \partial_q f(\zeta) \vecvp} \leq \frac{\smconst_3}{\numdirecp}.
\end{equation}

\item $\mathbf{\qpair \in \bivsupp}$, $\mathbf{\degree(q) = 1.}$
\begin{align}
{\vecvp}^T \hess \partial_q f(\zeta) \vecvp &= {v^{\prime}_{q}}^2 \partial_q^3 \phi_{\qpair}(\zeta_q,\zeta_{q^{\prime}}) + 
{v^{\prime}_{q^{\prime}}}^2 \partial_{q^{\prime}}^2 \partial_q \phi_{\qpair}(\zeta_q,\zeta_{q^{\prime}}) + 
2 v^{\prime}_{q} v^{\prime}_{q^{\prime}} \partial_{q^{\prime}} \partial_q^2 \phi_{\qpair}(\zeta_q,\zeta_{q^{\prime}}), \\
\Rightarrow \abs{{\vecvp}^T \hess \partial_q f(\zeta) \vecvp} &\leq \frac{4\smconst_3}{\numdirecp}.
\end{align}

\item $\mathbf{q \in \bivsuppvar}$, $\mathbf{\degree(q) > 1.}$
\begin{align}
{\vecvp}^T \hess \partial_q f(\zeta) \vecvp &= {v^{\prime}_{q}}^2(\partial_q^3 \phi_q(\zeta_q) + 
\sum_{\qpair \in \bivsupp} \partial_q^3 \phi_{\qpair}(\zeta_{q},\zeta_{q^{\prime}}) + 
\sum_{\qpairi \in \bivsupp} \partial_q^3 \phi_{\qpairi}(\zeta_{q^{\prime}},\zeta_{q})) \nonumber \\ 
&+ \sum_{\qpair \in \bivsupp} {v^{\prime}_{q^{\prime}}}^2 \partial_{q^{\prime}}^2 \partial_q \phi_{\qpair}(\zeta_{q},\zeta_{q^{\prime}}) + 
\sum_{\qpairi \in \bivsupp} {v^{\prime}_{q^{\prime}}}^2 \partial_{q^{\prime}}^2 \partial_q \phi_{\qpairi}(\zeta_{q^{\prime}},\zeta_{q}) \nonumber \\
&+ 2\sum_{\qpair \in \bivsupp} v^{\prime}_{q} v^{\prime}_{q^{\prime}} \partial_{q^{\prime}} \partial_q^2 \phi_{\qpair}(\zeta_{q},\zeta_{q^{\prime}}) + 
2\sum_{\qpairi \in \bivsupp} v^{\prime}_{q} v^{\prime}_{q^{\prime}} \partial_{q^{\prime}} \partial_q^2 \phi_{\qpairi}(\zeta_{q^{\prime}},\zeta_{q}), \\
\Rightarrow \abs{{\vecvp}^T \hess \partial_q f(\zeta) \vecvp} &\leq \frac{1}{\numdirecp}((\maxdegree+1)\smconst_3 + \maxdegree\smconst_3 + 2\maxdegree\smconst_3)
= \frac{(4\maxdegree+1)\smconst_3}{\numdirecp}.
\end{align}
\end{enumerate}
We can now uniformly bound $\norm{\hessestnoisa}_{\infty}$ as follows.
\begin{equation}
\norm{\hessestnoisa}_{\infty} := \max_{j=1,\dots,\numdirecp} \frac{\hessstep}{2} \abs{{{\vecvp_j}^T \hess \partial_q f(\zeta_j) \vecvp_j}} 
\leq \frac{\hessstep(4\maxdegree+1)\smconst_3}{2\numdirecp}. 
\end{equation}
\paragraph{Estimating $\bivsupp$.}
We now proceed towards estimating $\bivsupp$. To this end, we estimate $\grad \partial_q f(\vecx)$ for each $q=1,\dots,\dimn$ and $\vecx \in \baseset$.
Since $\grad \partial_q f(\vecx)$ is at most $(\maxdegree+1)$-sparse, therefore Theorem \ref{thm:sparse_recon_bound}, \eqref{eq:sparse_recon_err}, 
immediately yield the following.
$\exists C_2, c_5^{\prime}> 0, c_2^{\prime} \geq 1$ such that for 
$c_2^{\prime} \maxdegree \log(\frac{\dimn}{\maxdegree}) < \numdirecp < \frac{\dimn}{(\log 6)^2}$ we have 
with probability at least $1 - e^{-c_5^{\prime}\numdirecp} - e^{-\sqrt{\numdirecp\dimn}}$ that 
\begin{equation} \label{eq:hessrow_est_bd_gen}
\norm{\est{\grad} \partial_q f(\vecx) - \grad \partial_q f(\vecx)}_2 \leq C_2 \max\set{\norm{\hessestnoisa+\hessestnoisb}_2, \sqrt{\log \dimn}\norm{\hessestnoisa+\hessestnoisb}_{\infty}}. 
\end{equation} 
Since $\norm{\hessestnoisa+\hessestnoisb}_{\infty} \leq \norm{\hessestnoisa}_{\infty} + \norm{\hessestnoisb}_{\infty}$, 
therefore using the bounds on $\norm{\hessestnoisa}_{\infty},\norm{\hessestnoisb}_{\infty}$ and noting that 
$\norm{\hessestnoisa+\hessestnoisb}_2 \leq \sqrt{\numdirecp}\norm{\hessestnoisa+\hessestnoisb}_{\infty}$, we obtain for the stated choice of $\numdirecp$ (cf. Remark \ref{rem:l1min_samp_bd}) that
\begin{equation} \label{eq:hessrow_est_bd_gen_1}
\norm{\est{\grad} \partial_q f(\vecx) - \grad \partial_q f(\vecx)}_2 \leq \underbrace{C_2\left(\frac{\hessstep(4\maxdegree+1)\smconst_3}{2\sqrt{\numdirecp}} 
+ \frac{C_1\sqrt{\numdirecp}\gradstep^2((4\maxdegree+1)\totsparsity)\smconst_3}{3\numdirec\hessstep} \right)}_{\hesssamperr}; \quad q=1,\dots,\dimn, \ \forall \vecx \in [-1,1]^{\dimn}.
\end{equation}
We next note that \eqref{eq:hessrow_est_bd_gen_1} trivially leads to the bound
\begin{equation}
\est{\partial_q\partial_{q^{\prime}}} f(\vecx) \in [\partial_q\partial_{q^{\prime}} f(\vecx) -\hesssamperr, \partial_q\partial_{q^{\prime}} f(\vecx) + \hesssamperr]; \quad
q,q^{\prime} = 1,\dots,\dimn. 
\end{equation}
Now if $q \notin \bivsuppvar$ then clearly $\est{\partial_q\partial_{q^{\prime}}} f(\vecx) \in [-\hesssamperr, \hesssamperr]$; $\forall \vecx \in [-1,1]^{\dimn}, q \neq q^{\prime}$.
On the other hand, if $\qpair \in \bivsupp$ then 
\begin{equation}
\est{\partial_q\partial_{q^{\prime}}} f(\vecx) \in [\partial_q\partial_{q^{\prime}} \phi_{\qpair}(x_q,x_{q^{\prime}}) -\hesssamperr, \partial_q\partial_{q^{\prime}} \phi_{\qpair}(x_q,x_{q^{\prime}}) + \hesssamperr]. 
\end{equation}
If furthermore $\numcen \geq \critintmeas_2^{-1}$, then due to the construction of $\baseset$, $\exists \vecx \in \baseset$ so that 
$\abs{\est{\partial_q\partial_{q^{\prime}}} f(\vecx)} \geq \idenconst_2 - \hesssamperr$. Hence if $\hesssamperr < \idenconst_2/2$ holds, 
the we would have $\abs{\est{\partial_q\partial_{q^{\prime}}} f(\vecx)} > \idenconst_2/2$, leading to the identification of $\qpair$.
Since this is true for each $\qpair \in \bivsupp$, hence it follows that $\est{\bivsupp} = \bivsupp$. 
Now, $\hesssamperr < \idenconst_2/2$ is equivalent to 
\begin{align}
\underbrace{\frac{(4\maxdegree+1)\smconst_3}{2\sqrt{\numdirecp}}}_{a} \hessstep
+ \underbrace{\left(\frac{C_1\sqrt{\numdirecp}((4\maxdegree+1)\totsparsity)\smconst_3}{3\numdirec}\right)}_{b}\frac{\gradstep^2}{\hessstep} < \frac{\idenconst_2}{2C_2} 
\Leftrightarrow a \hessstep^2 - \frac{\idenconst_2}{C_2}\hessstep + b\gradstep^2 < 0 \\
\Leftrightarrow \hessstep \in \left((\idenconst_2/(4aC_2)) - \sqrt{(\idenconst_2/(4aC_2))^2 - (b\gradstep^2/a)} , (\idenconst_2/(4aC_2)) + \sqrt{(\idenconst_2/(4aC_2))^2 - (b\gradstep^2/a)} \right). 
\label{eq:hessstep_bd_gen_overlap}
\end{align}
Lastly, we see that the bounds in \eqref{eq:hessstep_bd_gen_overlap} are valid if:
\begin{equation} \label{eq:gradstep_bd_gen_over}
\gradstep^2 < \frac{\idenconst_2^2}{16 a b C_2^2} = \frac{3 \idenconst_2^2 \numdirec}{8C_1 C_2^2 \smconst_3^2(4\maxdegree+1)((4\maxdegree+1)\totsparsity)}.
\end{equation}
%

\paragraph{Estimating $\univsupp$.} 
With $\calP := [\dimn] \setminus \est{\bivsuppvar}$, we have via Taylor's expansion of $f$:
\begin{equation} \label{eq:taylor_exp_f_2}
\frac{f((\vecx + \gradstepp\vecvpp_j)_{\calP}) - f((\vecx - \gradstepp\vecvpp_j)_{\calP})}{2\gradstepp} = \dotprod{(\vecvpp_j)_{\calP}}{(\grad f((\vecx)_{\calP}))_{\calP}} 
+ \underbrace{\frac{\thirdtayrem_3((\zeta_j)_{\calP}) - \thirdtayrem_3((\zeta_j^{\prime})_{\calP})}{2\gradstepp}}_{\taynoissca_j}; \quad j=1,\dots,\numdirecpp.
\end{equation}
\eqref{eq:taylor_exp_f_2} corresponds to linear measurements of the $(\totsparsity-\abs{\est{\bivsuppvar}})$ sparse vector: $(\grad f((\vecx)_{\calP}))_{\calP}$.
We now proceed similar to the proof of Corollary \ref{corr:rec_grad_vecs}. 
Note that we effectively perform $\ell_1$ minimization over $\matR^{\abs{\calP}}$. 
Therefore for any $\vecx \in \matR^{\dimn}$ we immediately have from Theorem \ref{thm:sparse_recon_bound}, \eqref{eq:sparse_recon_err}, the following. 
$\exists C_3, c_6^{\prime}> 0, c_3^{\prime} \geq 1$ such that for 
$c_3^{\prime} (\totsparsity-\abs{\est{\bivsuppvar}}) \log(\frac{\abs{\calP}}{\totsparsity-\abs{\est{\bivsuppvar}}}) < \numdirecpp < \frac{\abs{\calP}}{(\log 6)^2}$, we have with probability at least $1 - e^{-c_6^{\prime}\numdirecpp} - e^{-\sqrt{\numdirecpp\abs{\calP}}}$ that 
\begin{equation} \label{eq:grad_est_bd_gen_1}
\norm{(\est{\grad} f((\vecx)_{\calP}))_{\calP} - (\grad f((\vecx)_{\calP}))_{\calP}}_2 \leq C_3 \max\set{\norm{\taynoisvec}_2, \sqrt{\log \abs{\calP}} \norm{\taynoisvec}_{\infty}}. 
\end{equation}
We now uniformly bound $\thirdtayrem_3((\zeta_j)_{\calP})$ for all $j=1,\dots,\numdirecpp$ and $\zeta_j \in [-(1+r),1+r]^{\dimn}$ as follows.
\begin{align}
\thirdtayrem_3((\zeta_j)_{\calP}) = \frac{{\gradstepp}^3}{6}\sum_{p \in \univsupp \cap \calP} \partial_p^3 \phi_p(\zeta_{j,p}){\vpp_{j,p}}^3 \quad 
\Rightarrow \abs{\thirdtayrem_3((\zeta_j)_{\calP})} \leq \frac{(\totsparsity-\abs{\est{\bivsuppvar}}) {\gradstepp}^3\smconst_3}{6 \numdirecpp^{3/2}}.
\end{align}
This in turn implies that $\norm{\taynoisvec}_{\infty} \leq \frac{(\totsparsity-\abs{\est{\bivsuppvar}}) {\gradstepp}^2\smconst_3}{6 \numdirecpp^{3/2}}$ and 
$\norm{\taynoisvec}_2 \leq \sqrt{\numdirecpp}\norm{\taynoisvec}_{\infty} \leq \frac{(\totsparsity-\abs{\est{\bivsuppvar}}) {\gradstepp}^2\smconst_3}{6 \numdirecpp}$. 
Plugging these bounds in \eqref{eq:grad_est_bd_gen_1}, we obtain for the stated choice of $\numdirecpp$ (cf. Remark \ref{rem:l1min_samp_bd}) that
\begin{equation} \label{eq:grad_est_bd_gen_2}
\norm{(\est{\grad} f((\vecx)_{\calP}))_{\calP} - (\grad f((\vecx)_{\calP}))_{\calP}}_2 \leq 
\underbrace{\frac{C_3 (\totsparsity-\abs{\est{\bivsuppvar}}) {\gradstepp}^2\smconst_3}{6 \numdirecpp}}_{\derivsamperrpp} ; \quad \vecx \in [-1,1]^{\dimn}. 
\end{equation}
Finally, using the same arguments as before, we have that $\derivsamperrpp < \idenconst_1/2$ or equivalently 
${\gradstepp}^2 < \frac{3\numdirecpp \idenconst_1}{C_3 (\totsparsity-\abs{\est{\bivsuppvar}}) \smconst_3}$ is sufficient to recover $\univsupp$. This 
completes the proof.

\subsection{Proof of Theorem \ref{thm:gen_overlap_arbnois}} \label{subsec:proof_thm_genover_arbnoise}
We begin by establishing the conditions pertaining to the estimation of $\bivsupp$. Then  
we prove the conditions for estimation of $\univsupp$.
\paragraph{Estimation of $\bivsupp$.} We first note that the linear system \eqref{eq:cs_form} 
now has the form: $\vecy = \matV\grad f(\vecx) + \taynoisvec + \exnoisevec$ where 
$\exnoise_{j} = (\exnoisep_{j,1} - \exnoisep_{j,2})/(2\gradstep)$ represents the external noise
component, for $j=1,\dots,\numdirec$. Observe that $\norm{\exnoisevec}_{\infty} \leq \exnoisemag/\gradstep$.
Using the bounds on $\norm{\taynoisvec}_{\infty}, \norm{\taynoisvec}_{2}$ from Section \ref{subsec:proof_thm_genover}, 
we then observe that \eqref{eq:grad_est_over_bd} changes to:
\begin{equation} \label{eq:grad_est_over_bd_arbnois}
\norm{\est{\grad} f(\vecx) - \grad f(\vecx)}_2 \leq C_1\left(\frac{\gradstep^2((4\maxdegree+1)\totsparsity)\smconst_3}{6\numdirec} + \frac{\exnoisemag\sqrt{\numdirec}}{\gradstep}\right), 
\quad \forall \vecx \in [-(1+r),1+r]^{\dimn}.
\end{equation}
As a result, we then have that 
\begin{equation}
\norm{\hessestnoisb}_{\infty} \leq C_1\left(\frac{\gradstep^2((4\maxdegree+1)\totsparsity)\smconst_3}{3\numdirec\hessstep} + \frac{2\exnoisemag\sqrt{\numdirec}}{\gradstep\hessstep}\right).
\end{equation}
Now note that the bound on $\norm{\hessestnoisa}_{\infty}$ is unchanged from Section \ref{subsec:proof_thm_genover}, \textit{i.e.}, 
$\norm{\hessestnoisa}_{\infty} \leq \frac{\hessstep(4\maxdegree+1)\smconst_3}{2\numdirecp}$. As a consequence, we see that 
\eqref{eq:hessrow_est_bd_gen_1} changes to:
\begin{equation} \label{eq:hessrow_est_bd_gen_arbnois}
\norm{\est{\grad} \partial_q f(\vecx) - \grad \partial_q f(\vecx)}_2 \leq \underbrace{C_2\left(\frac{\hessstep(4\maxdegree+1)\smconst_3}{2\sqrt{\numdirecp}} 
+ C_1\frac{\sqrt{\numdirecp}\gradstep^2((4\maxdegree+1)\totsparsity)\smconst_3}{3\numdirec\hessstep} + \frac{2C_1\exnoisemag\sqrt{\numdirec\numdirecp}}{\gradstep\hessstep}\right)}_{\hesssamperr}.
\end{equation}
With $a$ and $b$ as stated in the Theorem,  we then see that $\hesssamperr < \idenconst_2/2$ is equivalent to
\begin{equation}
a\hessstep^2 - \frac{\idenconst_2}{2C_2}\hessstep + \left(b\gradstep^2 + \frac{2C_1\exnoisemag\sqrt{\numdirec\numdirecp}}{\gradstep}\right) < 0.
\end{equation}
which in turn is equivalent to 
\begin{equation}
\hessstep \in \left(\frac{\idenconst_2}{4aC_2} - \sqrt{\left(\frac{\idenconst_2}{4aC_2}\right)^2 - 
\left(\frac{b\gradstep^3 + 2C_1\exnoisemag\sqrt{\numdirec\numdirecp}}{a\gradstep}\right)} , 
\frac{\idenconst_2}{4aC_2} + \sqrt{\left(\frac{\idenconst_2}{4aC_2}\right)^2 - \left(\frac{b\gradstep^3 + 2C_1\exnoisemag\sqrt{\numdirec\numdirecp}}{a\gradstep}\right)}\right).
\end{equation}
For the above bound to be valid, we require 
\begin{align}
\frac{b\gradstep^2}{a} + \frac{2C_1\exnoisemag\sqrt{\numdirec\numdirecp}}{a\gradstep} &< \frac{\idenconst_2^2}{16a^2C_2^2} \\
\Leftrightarrow \gradstep^3 - \frac{\idenconst_2^2}{16abC_2^2}\gradstep + \frac{2C_1\exnoisemag\sqrt{\numdirec\numdirecp}}{b} &< 0 \label{eq:cub_gradstep}
\end{align}
to hold. \eqref{eq:cub_gradstep} is a cubic inequality. \hemant{Recall from Section \ref{sec:real_roots_cub} that} a cubic equation of the form: $y^3 + py + q = 0$, has $3$ distinct real roots 
if its discriminant $\frac{p^3}{27} + \frac{q^2}{4} < 0$. Note that for this to be possible, $p$ must be negative, which is the case in \eqref{eq:cub_gradstep}. 
Applying this to \eqref{eq:cub_gradstep} leads to the condition:
$\exnoisemag < \frac{\idenconst_2^{3}}{192\sqrt{3} C_1 C_2^3 \sqrt{a^3 b \numdirecp \numdirec}} = \exnoisemag_1$. 
\hemant{Furthermore, as stated in \eqref{eq:cub_trans_roots},} the $3$ distinct real roots are given by:
\begin{equation} 
y_1 = 2\sqrt{-p/3}\cos(\theta/3), \ y_2 = -2\sqrt{-p/3}\cos(\theta/3 + \pi/3), \ y_3 = -2\sqrt{-p/3}\cos(\theta/3 - \pi/3) \label{eq:cub_roots}
\end{equation}
where $\theta = \cos^{-1}\left(\frac{-q/2}{\sqrt{-p^3/27}}\right)$. Applying this to \eqref{eq:cub_gradstep} then leads to $\theta_1 = \cos^{-1}(-\exnoisemag/\exnoisemag_1)$.
For $0 < \exnoisemag < \exnoisemag_1$ we have $\pi/2 < \theta_1 < \pi$ which implies $0 < y_2 < y_1$ and $y_3 < 0$. In particular 
if $q > 0$, then one can verify that $y^3 + py + q < 0$ holds if $y \in (y_2,y_1)$. 
Applying this to \eqref{eq:cub_gradstep}, we consequently obtain: 
\begin{equation}
\gradstep \in \left(\sqrt{\frac{\idenconst_2^2}{12 a b C_2^2}}\cos(\theta_1/3 - 2\pi/3) , \sqrt{\frac{\idenconst_2^2}{12 a b C_2^2}}\cos(\theta_1/3)\right). 
\end{equation}
%

\paragraph{Estimation of $\univsupp$.} We now prove the conditions for estimation of $\univsupp$. 
First note that \eqref{eq:taylor_exp_f_2} now changes to:
\begin{equation} \label{eq:taylor_exp_f_3}
\frac{f((\vecx + \gradstepp\vecvpp_j)_{\calP}) - f((\vecx - \gradstepp\vecvpp_j)_{\calP})}{2\gradstepp} = \dotprod{(\vecvpp_j)_{\calP}}{(\grad f((\vecx)_{\calP}))_{\calP}} 
+ \underbrace{\frac{\thirdtayrem_3((\zeta_j)_{\calP}) - \thirdtayrem_3((\zeta_j^{\prime})_{\calP})}{2\gradstepp}}_{\taynoissca_j} + \underbrace{\frac{\exnoisep_{j,1} - \exnoisep_{j,2}}{2\gradstepp}}_{\exnoise_j}, 
\end{equation}
for $j=1,\dots,\numdirecpp$. Denoting $\exnoisevec = [\exnoise_1 \cdots \exnoise_{\numdirecpp}]$, we have $\norm{\exnoisevec}_{\infty} \leq \exnoisemag/\gradstepp$. 
As the bounds on $\norm{\taynoisvec}_{2}, \norm{\taynoisvec}_{\infty}$ are unchanged, therefore \eqref{eq:grad_est_bd_gen_3} now changes to:
\begin{equation} \label{eq:grad_est_bd_gen_3}
\norm{(\est{\grad} f((\vecx)_{\calP}))_{\calP} - (\grad f((\vecx)_{\calP}))_{\calP}}_2 \leq 
\underbrace{C_3 \left(\frac{(\totsparsity-\abs{\est{\bivsuppvar}}) {\gradstepp}^2\smconst_3}{6 \numdirecpp} + \frac{\exnoisemag\sqrt{\numdirecpp}}{\gradstepp}\right)}_{\derivsamperrpp}  ; \quad \vecx \in [-1,1]^{\dimn}. 
\end{equation}
Denoting $a_1 = \frac{(\totsparsity-\abs{\est{\bivsuppvar}}) \smconst_3}{6\numdirecpp}$, $b_1 = \sqrt{\numdirecpp}$, we then see from 
\eqref{eq:grad_est_bd_gen_3} that the condition $\derivsamperrpp < \idenconst_1/2$ is equivalent to 
\begin{equation} \label{eq:s1_noise_cub}
{\gradstepp}^3 - \frac{\idenconst_1}{2 a_1 C_3}\gradstepp + \frac{b_1 \exnoisemag}{a_1} < 0.
\end{equation}
As discussed earlier for estimation of $\bivsupp$, the cubic equation corresponding to \eqref{eq:s1_noise_cub} has 
$3$ distinct real roots if its discriminant is negative. This then 
leads to the condition $\exnoisemag < \frac{\idenconst_1^{3/2}}{3\sqrt{6 a_1 C_3^3 b_1^2}} = \exnoisemag_2$. 
Then by using the expressions for the roots of the cubic from \eqref{eq:cub_roots}, one can verify 
that \eqref{eq:s1_noise_cub} holds if 
\begin{equation}
\gradstepp \in (2\sqrt{\idenconst_1/(6 a_1 C_3)} \cos(\theta_2/3 - 2\pi/3), 2\sqrt{\idenconst_1/(6 a_1 C_3)} \cos(\theta_2/3))
\end{equation}
with $\theta_2 = \cos^{-1}(\exnoisemag/\exnoisemag_2)$. This completes the proof. 

\subsection{Proof of Theorem \ref{thm:gen_overlap_gaussnois}} \label{subsec:proof_thm_genover_gauss}
We first derive conditions for estimating $\bivsupp$, and then for $\univsupp$. 
The outline is essentially the same as the proof of Theorem \ref{thm:gen_no_over_gauss} in 
Section \ref{subsec:proof_thm_nonover_stoch}, so we omit the details.

\paragraph{Estimating $\bivsupp$.} Upon resampling $N_1$ times and averaging, we have for the
noise vector $\vecz \in \matR^{\numdirec}$ that 
\begin{equation}
\vecz = \left[\frac{(\exnoisep_{1,1} - \exnoisep_{1,2})}{2\gradstep} \cdots \frac{(\exnoisep_{\numdirec,1} - \exnoisep_{\numdirec,2})}{2\gradstep} \right], 
\end{equation}
where $\exnoisep_{j,1}, \exnoisep_{j,2} \sim \calN(0,\sigma^2/N_1)$ are i.i.d. Our aim is to guarantee that 
$\abs{\exnoisep_{j,1} - \exnoisep_{j,2}} < 2\exnoisemag$ holds $\forall j=1,\dots,\numdirec$, and across all points where 
$\grad f$ is estimated. Indeed, we then obtain a bounded noise model and can simply use the analysis for the setting 
of arbitrary bounded noise. 

Now to estimate $\grad f(\vecx)$ we have $\numdirec$ many ``difference'' terms: $\exnoisep_{j,1} - \exnoisep_{j,2}$. We additionally estimate 
$\numdirecp$ many gradients at each $\vecx$ implying a total of $\numdirec(\numdirecp + 1)$ difference terms. As this is done for each 
$\vecx \in \baseset$, therefore we have a total of $\numdirec(\numdirecp + 1)(2\numcen+1)^2\abs{\twohashfam}$ many difference terms. 
Taking a union bound over all of them, we have for any 
$p_1 \in (0,1)$ that the choice 
$N_1 > \frac{\sigma^2}{\exnoisemag^2} {\hemantt \log (\frac{2}{p_1}\numdirec(\numdirecp+1)(2\numcen+1)^2\abs{\twohashfam})}$ 
implies that the magnitudes of all difference terms are bounded by $2\exnoisemag$, with probability at least 
$1-p_1$. 

\paragraph{Estimating $\univsupp$.} In this case, we resample each query $N_2$ times and average -- 
therefore the variance of the noise 
terms gets scaled by $N_2$. We now have $\abs{\baseset_{\text{diag}}} \numdirecpp = (2\numcenpair+1) \numdirecpp$ 
many ``difference'' terms corresponding to Gaussian noise. 
Therefore, taking a union bound over all of them, we have for any $p_2 \in (0,1)$ that the choice 
$N_2 > \frac{\sigma^2}{{\exnoisemagp}^2} {\hemantt \log(\frac{2 (2\numcenpair+1)\numdirecpp}{p_2})}$ 
implies that the magnitudes of all difference terms are bounded by $2\exnoisemagp$, with probability at least 
$1-p_2$. 

\section{Proofs for Section \ref{sec:algo_gen_overlap_alt}} \label{sec:proofs_gen_overlap_alt}
\subsection{Proof of Theorem \ref{thm:gen_overlap_alt}} \label{subsec:proof_thm_genover_alt}
We only prove the part concerning the identification of $\bivsupp$, as the proof for identifying $\univsupp$ is identical to that 
of Theorem \ref{thm:gen_overlap} (see Section \ref{subsec:proof_thm_genover}). 
Consider the linear system defined in \eqref{eq:lin_sys_alt_hess_samp} at some $\vecx \in [-1,1]^d$. 
We begin by uniformly bounding the magnitude of the remainder terms: $\abs{\thirdtayrem_3(\zeta_j)}$, $\abs{\thirdtayrem_3(\zeta_j^{\prime})}$ 
where $\zeta_j, \zeta_j^{\prime} \in [-(1+r), 1+r]^d$ for some $r > 0$; $j=1,\dots,\numdirec$. Let us define 
$\numdegree := \abs{\set{q \in \bivsuppvar: \degree(q) > 1}}$, to be the number of variables in $\bivsuppvar$, with degree greater than one. 
By taking the structure of $f$ into account, we can uniformly bound $\abs{\thirdtayrem_3(\zeta_j)}$ as follows.
\begin{align}
\abs{\thirdtayrem_3(\zeta_j)} &= \frac{\gradstep^3}{6} |\sum_{p \in \univsupp} \partial_p^3 \phi_p(\zeta_{j,p}) (2v_p)^3  + 
\sum_{\lpair \in \bivsupp} ( \partial_l^3 \phi_{\lpair}(\zeta_{j,l}, \zeta_{j,{l^{\prime}}}) (2v_l)^3 + \partial_{l^{\prime}}^3 
\phi_{\lpair}(\zeta_{j,l}, \zeta_{j,{l^{\prime}}}) (2v_{l^{\prime}})^3) \nonumber \\
&+ \sum_{\lpair \in \bivsupp} (3\partial_l \partial_{l^{\prime}}^2 \phi_{\lpair}(\zeta_{j,l}, \zeta_{j,{l^{\prime}}}) (2v_l) (2v_{l^{\prime}})^2 + 
3\partial_l^2 \partial_{l^{\prime}} \phi_{\lpair}(\zeta_{j,l}, \zeta_{j,{l^{\prime}}}) (2v_l)^2 (2v_{l^{\prime}}))
+ \sum_{q \in \bivsuppvar: \degree(q)>1} \partial_q^3 \phi_q(\zeta_{j,q}) (2v_q)^3| \\
&\leq \frac{\gradstep^3}{6} \left(\frac{8\univsparsity \smconst_3\hemant{(\sqrt{3})^3}}{\numdirec^{3/2}} + 
\frac{16\bivsparsity\smconst_3\hemant{(\sqrt{3})^3}}{\numdirec^{3/2}} + \frac{8\alpha\smconst_3\hemant{(\sqrt{3})^3}}{\numdirec^{3/2}} + \frac{48\bivsparsity\smconst_3\hemant{(\sqrt{3})^3}}{\numdirec^{3/2}} \right) \\
&= \hemant{\frac{4\sqrt{3}\gradstep^3\smconst_3}{\numdirec^{3/2}}} (\univsparsity + \alpha + 8\bivsparsity).  \label{eq:alt_temp_bd_1}
\end{align}
By observing $2\bivsparsity = \sum_{l \in \bivsuppvar: \degree(l) > 1} \degree(l) + (\abs{\bivsuppvar} - \numdegree)$, we obtain  
$2\bivsparsity \leq \maxdegree\numdegree + (\abs{\bivsuppvar} - \numdegree) = \abs{\bivsuppvar} + (\maxdegree-1)\numdegree$. Plugging this in 
\eqref{eq:alt_temp_bd_1}, and using the fact $\alpha \leq \totsparsity$, we obtain
\begin{align}
\abs{\thirdtayrem_3(\zeta_j)} \leq \hemant{\frac{4\sqrt{3}\gradstep^3\smconst_3}{\numdirec^{3/2}}}(4\maxdegree+1)\totsparsity. 
\end{align}
Since the same bound holds also for $\abs{\thirdtayrem_3(\zeta_j^{\prime})}$, we thus obtain:
\begin{align}
\norm{\taynoisvec}_{\infty} &\leq \frac{1}{2\gradstep^2} \left(\frac{4\sqrt{3}\gradstep^3\smconst_3}{\numdirec^{3/2}}(4\maxdegree+1)\totsparsity\right) 
= \frac{2\sqrt{3}\gradstep\smconst_3}{\numdirec^{3/2}}(4\maxdegree+1)\totsparsity, \\
\Rightarrow \norm{\taynoisvec}_{1} &\leq \numdirec \frac{2\sqrt{3}\gradstep\smconst_3}{\numdirec^{3/2}}(4\maxdegree+1)\totsparsity 
= \frac{2\sqrt{3}\gradstep\smconst_3}{\numdirec^{1/2}}(4\maxdegree+1)\totsparsity. \label{eq:taynoisevec_l1_bd_alt}
\end{align}
Therefore by setting $\hessestnoisebd = \frac{2\sqrt{3}\gradstep\smconst_3}{\numdirec^{1/2}}(4\maxdegree+1)\totsparsity$, and for the stated choice of 
$\numdirec$, we obtain via Theorem \ref{thm:chen_sparse_symm_rec} that
\begin{align} \label{eq:hess_est_bd_gen_alt_1}
\norm{\est{\hess} f(\vecx) - \hess f(\vecx)}_F \leq C_1 \hessestnoisebd = \underbrace{C_1 \frac{2\sqrt{3}\gradstep\smconst_3}{\numdirec^{1/2}}(4\maxdegree+1)\totsparsity}_{\derivsamperr}; \quad \forall \vecx \in [-1,1]^d.
\end{align}
We next note that \eqref{eq:hess_est_bd_gen_alt_1} leads to
\begin{equation}
\est{\partial_q\partial_{q^{\prime}}} f(\vecx) \in [\partial_q\partial_{q^{\prime}} f(\vecx) -\frac{\derivsamperr}{\sqrt{2}}, \partial_q\partial_{q^{\prime}} f(\vecx) + \frac{\derivsamperr}{\sqrt{2}}]; \quad \qpair \in {[d] \choose 2}. 
\end{equation}
Now if $\qpair \notin \bivsupp$ then clearly $\est{\partial_q\partial_{q^{\prime}}} f(\vecx) \in [-\frac{\derivsamperr}{\sqrt{2}}, \frac{\derivsamperr}{\sqrt{2}}]$; $\forall \vecx \in [-1,1]^{\dimn}$. On the other hand, if $\qpair \in \bivsupp$ then 
\begin{equation}
\est{\partial_q\partial_{q^{\prime}}} f(\vecx) \in [\partial_q\partial_{q^{\prime}} \phi_{\qpair}(x_q,x_{q^{\prime}}) - \frac{\derivsamperr}{\sqrt{2}}, \partial_q\partial_{q^{\prime}} \phi_{\qpair}(x_q,x_{q^{\prime}}) + \frac{\derivsamperr}{\sqrt{2}}]. 
\end{equation}
If furthermore $\numcen \geq \critintmeas_2^{-1}$, then due to the construction of $\baseset$, $\exists \vecx \in \baseset$ so that 
$\abs{\est{\partial_q\partial_{q^{\prime}}} f(\vecx)} \geq \idenconst_2 - \frac{\derivsamperr}{\sqrt{2}}$. Hence if $\frac{\derivsamperr}{\sqrt{2}} < \idenconst_2/2$ holds, the we would have $\abs{\est{\partial_q\partial_{q^{\prime}}} f(\vecx)} > \idenconst_2/2$, leading to the identification of $\qpair$.
Since this is true for each $\qpair \in \bivsupp$, hence it follows that $\est{\bivsupp} = \bivsupp$. 
Lastly, we easily see that $\frac{\derivsamperr}{\sqrt{2}} < \idenconst_2/2$ is equivalent to the stated condition on $\gradstep$.

\subsection{Proof of Theorem \ref{thm:gen_overlap_arbnois_alt}} \label{subsec:proof_thm_genover_arbnoise_alt}
We only prove the part concerning the identification of $\bivsupp$, as the proof for identifying $\univsupp$ is identical to that 
of Theorem \ref{thm:gen_overlap_arbnois} (see Section \ref{subsec:proof_thm_genover_arbnoise}). To this end, note that 
\eqref{eq:lin_sys_alt_hess_samp} now changes to the linear system $\vecy = \linoptmat(\hess f(\vecx)) + \taynoisvec + \exnoisevec$, where
$\exnoise_{j} = (\exnoisep_{j,1} + \exnoisep_{j,2} - 2\exnoisep_{3})/(4\gradstep^2)$ for $j = 1,\dots,\numdirec$. 
Since $\norm{\exnoisevec}_{\infty} \leq \frac{\exnoisemag}{\gradstep^2}$, therefore using the bound on 
$\norm{\taynoisvec}_{1}$ in \eqref{eq:taynoisevec_l1_bd_alt}, we readily obtain 
\begin{equation}
 \norm{\taynoisvec + \exnoisevec}_1 \leq \underbrace{\hemant{\frac{2\sqrt{3}\gradstep\smconst_3}{\numdirec^{1/2}}(4\maxdegree+1)\totsparsity} + \frac{\exnoisemag\numdirec}{\gradstep^2}}_{\hessestnoisebd}, 
\end{equation}
which in conjunction with Theorem \ref{thm:chen_sparse_symm_rec} readily implies that
\begin{align} \label{eq:hess_est_bd_gen_alt_2}
\norm{\est{\hess} f(\vecx) - \hess f(\vecx)}_F \leq C_1 \hessestnoisebd = \underbrace{C_1 \left(\hemant{\frac{2\sqrt{3}\gradstep\smconst_3}{\numdirec^{1/2}}(4\maxdegree+1)\totsparsity} 
+ \frac{\exnoisemag\numdirec}{\gradstep^2}\right)}_{\derivsamperr}; \quad \forall \vecx \in [-1,1]^d.
\end{align}
As shown in Section \ref{subsec:proof_thm_genover_alt}, it is sufficient to guarantee $\derivsamperr/\sqrt{2} < \idenconst_2/2$ for 
exact identification of $\bivsupp$. This is equivalent to saying that 
\begin{align}
\hemant{\underbrace{\frac{\sqrt{6}\gradstep\smconst_3}{\numdirec^{1/2}}(4\maxdegree+1)\totsparsity}_{a\gradstep}} + \underbrace{\frac{\exnoisemag\numdirec}{\gradstep^2\sqrt{2}}}_{b/\gradstep^2} < \frac{\idenconst_2}{2C_1} 
\Leftrightarrow \gradstep^3 - \frac{\idenconst_2}{2aC_1}\gradstep^2 + \frac{b}{a} < 0. \label{eq:cub_ineq_arb_alt}
\end{align}
\eqref{eq:cub_ineq_arb_alt} is a cubic inequality. \hemant{Recall from Section \ref{sec:real_roots_cub} that a cubic equation of the form: $x^3 + Ax^2 + C = 0$, has $3$ distinct real roots 
if its discriminant \hemant{$\frac{p^3}{27} + \frac{q^2}{4} < 0$ where $p = -\frac{A^2}{3}$ and $q = \frac{27C + 2A^3}{27}$}}. Assuming the discriminant to be 
negative (which means $p < 0$), and denoting \hemant{$\theta_1 = \cos^{-1}(-\frac{q/2}{\sqrt{-p^3/27}})$, the three roots are given as in \eqref{eq:cub_orig_roots}}:
\begin{align} 
 x_1 = \hemant{2\sqrt{-\frac{p}{3}}} \cos\left(\frac{\theta_1}{3}\right) - \frac{A}{3} = -\frac{2A}{3}\cos\left(\frac{\theta_1}{3}\right) - \frac{A}{3}, \label{eq:cub_roots_11}\\ 
 x_2 = \hemant{2\sqrt{-\frac{p}{3}}} \cos\left(\frac{\theta_1}{3} + \frac{2\pi}{3}\right) - \frac{A}{3} = \frac{2A}{3}\cos\left(\frac{\theta_1}{3} - \frac{\pi}{3}\right) - \frac{A}{3},\label{eq:cub_roots_12} \\
 x_3 = \hemant{2\sqrt{-\frac{p}{3}}} \cos\left(\frac{\theta_1}{3} + \frac{4\pi}{3}\right) - \frac{A}{3} = \frac{2A}{3}\cos\left(\frac{\theta_1}{3} + \frac{\pi}{3}\right) - \frac{A}{3}. \label{eq:cub_roots_13}
\end{align}
For $0 < \theta_1 < \pi$ one can verify that $x_2 < 0$ and $0 < x_3 < x_1$. Moreover, since $A < 0$ and $C > 0$, 
it is not hard to verify that $x^3 + Ax^2 + C < 0$ for $x \in (x_3,x_1)$.  

Translated to our setting, we have $A = -\idenconst_2/(2aC_1)$, $C = b/a$ which gives us 
\hemant{$p = -\frac{\idenconst_2^2}{12a^2C_1^2}$ and $q = \left(\frac{27b}{a} - \frac{\idenconst_2^3}{4a^3C_1^3}\right)/27$}. 
The cubic equation corresponding to \eqref{eq:cub_ineq_arb_alt} has three distinct real roots if
\begin{align}
\hemant{\abs{q}/2 < (-p^3/27)^{1/2}} &= \frac{\idenconst_2^3}{216a^3C_1^3}, \\
\Leftrightarrow \frac{27b}{a} - \frac{\idenconst_2^3}{4a^3C_1^3} &< \frac{\idenconst_2^3}{4a^3C_1^3}, \\
\Leftrightarrow \exnoisemag < \frac{\sqrt{2}\idenconst_2^3}{54a^2C_1^3\numdirec} 
&= \hemant{\frac{\idenconst_2^3}{162\sqrt{2}C_1^3\smconst_3^2} \frac{1}{(4\maxdegree+1)^2\totsparsity^2}} = \exnoisemag_1.
\end{align}
Furthermore, we have: 
\begin{align}
\theta_1 = \hemant{\cos^{-1}\left(\frac{-q/2}{\sqrt{-p^3/27}}\right)} = \cos^{-1}\left(\frac{\frac{-27b}{a} 
+ \frac{\idenconst_2^3}{4a^3C_1^3}}{\frac{\idenconst_2^3}{4a^3C_1^3}}\right) = \cos^{-1}\left(1-\frac{2\exnoisemag}{\exnoisemag_1}\right).
\end{align}
Lastly, \eqref{eq:cub_ineq_arb_alt} is satisfied for $\gradstep \in (x_3,x_1)$. Substituting the expression for $A$ 
in \eqref{eq:cub_roots_11},\eqref{eq:cub_roots_13}, we arrive at the stated condition on $\gradstep$. This completes the proof.

\subsection{Proof of Theorem \ref{thm:gen_overlap_gaussnois_alt}} \label{subsec:proof_thm_genover_gaussnoise_alt}
Let the external noise vector be denoted by $\exnoisevec \in \matR^{\numdirec}$ where 
$\exnoise_{j} = (\exnoisep_{j,1} + \exnoisep_{j,2} - 2\exnoisep_{3})/(4\gradstep^2)$. 
Upon resampling $N_1$ times and averaging, we have 
$\exnoisep_{j,1}$, $\exnoisep_{j,2}$, $\exnoisep_{3}$ $\sim \calN(0,\sigma^2/N_1)$, which in turn 
implies $\exnoisep_{j,1} + \exnoisep_{j,2} - 2\exnoisep_{3} \sim \calN(0,6\sigma^2/N_1)$.  
Our aim is to guarantee that 
$\abs{\exnoisep_{j,1} + \exnoisep_{j,2} - 2\exnoisep_{3}} < 4\exnoisemag$ holds $\forall j=1,\dots,\numdirec$, 
and across all points where $\hess f$ is estimated. Indeed, we then obtain a bounded noise model and can simply 
use the analysis for the setting of arbitrary bounded noise. 

To this end, we proceed as in the proof of Theorem \ref{thm:gen_no_over_gauss} in Section \ref{subsec:proof_thm_nonover_stoch}. 
Denoting $X \sim \calN(0,1)$, we first have 
{\hemantt $\exnoisep_{j,1} + \exnoisep_{j,2} - 2\exnoisep_{3}$  $= \sigma\sqrt{\frac{6}{N_1}} X$}. 
Using the tail bound for standard Gaussian random variables, we then obtain

$$\prob(\abs{\exnoisep_{j,1} + \exnoisep_{j,2} - 2\exnoisep_{j,3}} > 4\exnoisemag) \leq 
{\hemantt 2} \exp\left(-\frac{4\exnoisemag^2 N_1}{3\sigma^2}\right).$$  

At each $\vecx \in \baseset$, we have $\numdirec$ many terms of the form: $\exnoisep_{j,1} + \exnoisep_{j,2} - 2\exnoisep_{3}$, 
meaning that we have a total of $\numdirec (2\numcen+1)^2 \abs{\twohashfam}$ such terms.  
Taking a union bound over all of them, we have for any $p_1 \in (0,1)$ that the choice 
$N_1 > \frac{3\sigma^2}{4\exnoisemag^2} {\hemantt \log (\frac{2}{p_1}\numdirec(2\numcen+1)^2\abs{\twohashfam})}$ 
implies that the magnitudes of all such terms are bounded by $4\exnoisemag$, with probability at least 
$1-p_1$.

\section{Proofs for Section \ref{sec:est_comp}} \label{sec:proofs_comp_est}
\subsection{Proof of Proposition \ref{prop:no_nois_est_comp}} \label{subsec:proof_prop_noiseless_comp_est}
\begin{enumerate}
\item $\mathbf{p \in \univsupp}$.

We have for $\phitil_p$ that $\norm{\phitil_p - (\phi_p + C)}_{\Linfnorm[-1,1]} = O(n^{-3})$. 
Denoting $\phitil_p(x_p) - (\phi_p(x_p) + C) = z_p(x_p)$, this means $\abs{z_p(x_p)} = O(n^{-3})$, $\forall x_p \in [-1,1]$.
Now $\abs{\expec_p[\phitil_p - (\phi_p + C)]} = \abs{\expec_p[\phitil_p] - C} = \abs{\expec_p[z_p]} \leq \expec_p[\abs{z_p}] = O(n^{-3})$. 

Lastly, we have that:
\begin{align}
\norm{\est{\phi}_p - \phi_p}_{\Linfnorm[-1,1]} &= \norm{\phitil_p - \expec_p[\phitil_p] - \phi_p}_{\Linfnorm[-1,1]} \\ 
&= \norm{\phitil_p - (\phi_p + C) - (\expec_p[\phitil_p] - C)}_{\Linfnorm[-1,1]} \\ &= O(n^{-3}).
\end{align}

\item $\mathbf{\lpair \in \bivsupp}$.

We only consider the case where $\degree(l), \degree(\lp) > 1$ as proofs for the other cases are similar. 
Now for $\phitil_{\lpair}$ we have that $\norm{\phitil_{\lpair} - (g_{\lpair} + C)}_{\Linfnorm[-1,1]^2} = O(n^{-3/2})$. 
Denoting $\phitil_{\lpair}\xlpair - (g_{\lpair}\xlpair + C) = z_{\lpair} \xlpair$, this means 
$\abs{z_{\lpair}\xlpair} = O(n^{-3/2})$, $\forall \xlpair \in [-1,1]^2$. 
Consequently, one can easily verify that:

\begin{align}
 \norm{\expec_l[\phitil_{\lpair}] - (\expec_l[g_{\lpair}] + C)}_{\Linfnorm[-1,1]} = O(n^{-3/2}), \label{eq:temp5}\\
 \norm{\expec_{\lp}[\phitil_{\lpair}] - (\expec_{\lp}[g_{\lpair}] + C)}_{\Linfnorm[-1,1]} = O(n^{-3/2}), \label{eq:temp6}\\
 \norm{\expec_{\lpair}[\phitil_{\lpair}] - (\expec_{\lpair}[g_{\lpair}] + C)}_{\Linfnorm} = O(n^{-3/2}). \label{eq:temp7}
\end{align}

Now note that using the form for $g_{\lpair}$ from \eqref{eq:biv_part_fn_exp}, we have that 

\begin{align}
\expec_{l}[g_{\lpair}] &= \sum_{\substack{l_1:(l,l_1) \in \bivsupp \\ l_1 \neq \lp}} \expec_{l}[\phi_{(l,l_1)} (x_l,0)] + \sum_{\substack{l_1:(l_1,l) \in \bivsupp \\ l_1 \neq \lp}} 
\expec_l[\phi_{(l_1,l)} (0,x_l)] + \sum_{\substack{\lp_1:(\lp,\lp_1) \in \bivsupp \\ \lp_1 \neq l}} \phi_{(\lp,\lp_1)} (x_{\lp},0) \nonumber \\ 
&+ \sum_{\substack{\lp_1:(\lp_1,\lp) \in \bivsupp \\ \lp_1 \neq l}} \phi_{(\lp_1,\lp)} (0,x_{\lp}) + 
\phi_{\lp}(x_{\lp}) + C, \quad \text{and} \label{eq:temp1} 
\end{align}
\begin{align}
\expec_{\lp}[g_{\lpair}] &= \sum_{\substack{l_1:(l,l_1) \in \bivsupp \\ l_1 \neq \lp}} \phi_{(l,l_1)} (x_l,0) + \sum_{\substack{l_1:(l_1,l) \in \bivsupp \\ l_1 \neq \lp}} 
\phi_{(l_1,l)} (0,x_l) 
+ \sum_{\substack{\lp_1:(\lp,\lp_1) \in \bivsupp \\ \lp_1 \neq l}} \expec_{\lp} [\phi_{(\lp,\lp_1)} (x_{\lp},0)] \nonumber \\ 
&+ \sum_{\substack{\lp_1:(\lp_1,\lp) \in \bivsupp \\ \lp_1 \neq l}} \expec_{\lp}\phi_{(\lp_1,\lp)} (0,x_{\lp}) + 
\phi_l(x_l) + C, \quad \text{and} \label{eq:temp2} \\
\expec_{\lpair}[g_{\lpair}] &= \sum_{\substack{l_1:(l,l_1) \in \bivsupp \\ l_1 \neq \lp}} \expec_{l}[\phi_{(l,l_1)} (x_l,0)] + \sum_{\substack{l_1:(l_1,l) \in \bivsupp \\ l_1 \neq \lp}} 
\expec_l[\phi_{(l_1,l)} (0,x_l)] \nonumber \\ 
&+ \sum_{\substack{\lp_1:(\lp,\lp_1) \in \bivsupp \\ \lp_1 \neq l}} \expec_{\lp} [\phi_{(\lp,\lp_1)} (x_{\lp},0)]  
+ \sum_{\substack{\lp_1:(\lp_1,\lp) \in \bivsupp \\ \lp_1 \neq l}} \expec_{\lp}\phi_{(\lp_1,\lp)} (0,x_{\lp}) +  C. \label{eq:temp3}
\end{align}
We then have from \eqref{eq:biv_part_fn_exp}, \eqref{eq:temp1}, \eqref{eq:temp2}, \eqref{eq:temp3} that 
\begin{equation}
g_{\lpair} - \expec_{l}[g_{\lpair}] - \expec_{\lp}[g_{\lpair}] + \expec_{\lpair}[g_{\lpair}] = \phi_{\lpair}. \label{eq:temp4}
\end{equation}
Using \eqref{eq:temp5}, \eqref{eq:temp6}, \eqref{eq:temp7}, \eqref{eq:temp4}, and \eqref{eq:biv_est} it then follows that:
\begin{align}
\norm{\est{\phi}_{\lpair} - \phi_{\lpair}}_{\Linfnorm[-1,1]^2} = O(n^{-3/2}).
\end{align}

\item $\mathbf{l \in \bivsuppvar: \degree(l) > 1}$.

In this case, for $\phitil_l : [-1,1]^2 \rightarrow \matR$, we have that 
$\norm{\phitil_l -(g_l + C)}_{\Linfnorm[-1,1]^2} = O(n^{-3/2})$, with 
\begin{align}
g_{l}(x_l,x) =  \phi_l(x_l) &+ \sum_{\degree(\lp) > 1, \lp \neq l} \phi_{\lp}(x) + \sum_{\lp:\lpair \in \bivsupp} \phi_{\lpair}(x_l,x) \nonumber \\
&+ \sum_{\lp:\lpairi \in \bivsupp} \phi_{\lpairi}(x,x_l) + \sum_{\qpair \in \bivsupp : q,\qp \neq l} \phi_{\qpair}(x,x). \label{eq:temp11}
\end{align}
From \eqref{eq:temp11}, we see that: 
\begin{align}
\expec_{x}[g_l(x_l,x)] &= \phi_l(x_l) + \sum_{\qpair \in \bivsupp : q,\qp \neq l} \expec_x[\phi_{\qpair}(x,x)], \\
\text{and} \quad \expec_{(l,x)}[g_l(x_l,x)] &= \sum_{\qpair \in \bivsupp : q,\qp \neq l} \expec_{x}[\phi_{\qpair}(x,x)]. 
\end{align}
Hence clearly, $\expec_{x}[g_l(x_l,x)] - \expec_{(l,x)}[g_l(x_l,x)] = \phi_l(x_l)$. One can also easily verify that 
\begin{align}
\norm{\expec_{x}[\phitil_l] - (\expec_{x}[g_l] + C)}_{\Linfnorm[-1,1]} &= O(n^{-3/2}), \\ 
\norm{\expec_{(l,x)}[\phitil_l] - (\expec_{(l,x)}[g_l] + C)}_{\Linfnorm}  &= O(n^{-3/2}). 
\end{align}
Therefore it follows that 
\begin{align}
\norm{\est{\phi}_{l} - \phi_{l}}_{\Linfnorm[-1,1]} 
&= \norm{(\expec_{x}[\phitil_l] - \expec_{(l,x)}[\phitil_l]) - (\expec_{x}[g_l] - \expec_{(l,x)}[g_l])}_{\Linfnorm[-1,1]} \\
&\leq \norm{\expec_{x}[\phitil_l] - (\expec_{x}[g_l] + C)}_{\Linfnorm[-1,1]} + \norm{\expec_{(l,x)}[\phitil_l] - (\expec_{(l,x)}[g_l] + C)}_{\Linfnorm} \\
&= O(n^{-3/2}).
\end{align}
This completes the proof.
\end{enumerate}

\subsection{Proof of Proposition \ref{prop:gauss_nois_est_comp}} \label{subsec:proof_prop_gauss_comp_est}
Although the proof is again very similar to that of Proposition \ref{prop:no_nois_est_comp}, there 
are some technical differences. Hence we provide a brief sketch of the proof, avoiding details already 
highlighted in the proof of Proposition \ref{prop:no_nois_est_comp}.
\begin{enumerate}
\item $\mathbf{p \in \univsupp}$.

We have for $\phitil_p$ that $\expec_{z}[\norm{\phitil_p - (\phi_p + C)}_{\Linfnorm[-1,1]}] = O((n^{-1} \log n)^{\frac{3}{7}})$. 
Denoting $\phitil_p(x_p) - (\phi_p(x_p) + C) = b_p(x_p)$, this means $\expec_{z}[\abs{b_p(x_p)}] = O((n^{-1} \log n)^{\frac{3}{7}})$. 
Now, 

\begin{equation}
\expec_{z}[\abs{\expec_p[\phitil_p - (\phi_p + C)]}] = \expec_{z}[\abs{\expec_p[b_p]}] \leq \expec_{z}[\expec_p[\abs{b_p}]] = 
\expec_{p}[\expec_z[\abs{b_p(x_p)}]] = O((n^{-1} \log n)^{\frac{3}{7}}). 
\end{equation}

The penultimate equality above involves swapping the order of expectations, which is possible by Tonelli's 
theorem (since $\abs{b_p} > 0$). Then using triangle inequality, it follows that 
$\expec_z[\norm{\est{\phi}_p - \phi_p}_{\Linfnorm[-1,1]}] = O((n^{-1} \log n)^{\frac{3}{7}})$.

\item $\mathbf{\lpair \in \bivsupp}$.

We only consider the case where $\degree(l), \degree(\lp) > 1$ as proofs for the cases are similar. 
For $\phitil_{\lpair}$, we have that 
$\expec_z[\norm{\phitil_{\lpair} - (g_{\lpair} + C)}_{\Linfnorm[-1,1]^2}] = O((n^{-1} \log n)^{\frac{3}{8}})$. 
Denoting $\phitil_{\lpair}\xlpair - (g_{\lpair}\xlpair + C) = b_{\lpair} \xlpair$, this means 
$\expec_z[\abs{b_{\lpair}\xlpair}] = O((n^{-1} \log n)^{\frac{3}{8}})$, $\forall \xlpair \in [-1,1]^2$. 
Using Tonelli's theorem as earlier, one can next verify that:
\begin{align}
 \expec_z[\norm{\expec_l[\phitil_{\lpair}] - (\expec_l[g_{\lpair}] + C)}_{\Linfnorm[-1,1]}] = O((n^{-1} \log n)^{\frac{3}{8}}), \label{eq:gauss_temp5} \\
 \expec_z[\norm{\expec_{\lp}[\phitil_{\lpair}] - (\expec_{\lp}[g_{\lpair}] + C)}_{\Linfnorm[-1,1]}] = O((n^{-1} \log n)^{\frac{3}{8}}), \label{eq:gauss_temp6} \\
 \expec_z[\abs{\expec_{\lpair}[\phitil_{\lpair}] - (\expec_{\lpair}[g_{\lpair}] + C)}] = O((n^{-1} \log n)^{\frac{3}{8}}). \label{eq:gauss_temp7}
\end{align}
As in the proof of Proposition \ref{prop:no_nois_est_comp}, we obtain from \eqref{eq:gauss_temp5}, \eqref{eq:gauss_temp6}, 
\eqref{eq:gauss_temp7}, \eqref{eq:temp4}, \eqref{eq:biv_est} (via triangle inequality):
\begin{align}
\expec_z[\norm{\est{\phi}_{\lpair} - \phi_{\lpair}}_{\Linfnorm[-1,1]^2}] = O((n^{-1} \log n)^{\frac{3}{8}}).
\end{align}

\item $\mathbf{l \in \bivsuppvar: \degree(l) > 1}$.

In this case, for $\phitil_l : [-1,1]^2 \rightarrow \matR$, we have that 
$\expec_z[\norm{\phitil_l - (g_l + C)}_{\Linfnorm[-1,1]^2}] = O((n^{-1} \log n)^{\frac{3}{8}})$, with 
$g_l(x_l,x)$ as defined in \eqref{eq:temp11}. Using Tonelli's theorem as earlier, one can verify that  
\begin{align}
\expec_z[\norm{\expec_{x}[\phitil_l] - (\expec_{x}[g_l] + C)}_{\Linfnorm[-1,1]}] &= O((n^{-1} \log n)^{\frac{3}{8}}), \\ 
\expec_z[\abs{\expec_{(l,x)}[\phitil_l] - (\expec_{(l,x)}[g_l] + C)}] &= O((n^{-1} \log n)^{\frac{3}{8}}). 
\end{align}
Then using the fact $\expec_{x}[g_l(x_l,x)] - \expec_{(l,x)}[g_l(x_l,x)] = \phi_l(x_l)$, 
we obtain via triangle inequality the bound: 
$\expec_z[\norm{\est{\phi}_{l} - \phi_{l}}_{\Linfnorm[-1,1]}] = O((n^{-1} \log n)^{\frac{3}{8}})$. This completes the proof.
\end{enumerate}

\end{document}